\documentclass{article}

\PassOptionsToPackage{numbers, sort&compress}{natbib}

\usepackage[preprint]{arxiv}

\usepackage[utf8]{inputenc} 
\usepackage[T1]{fontenc}    
\usepackage{hyperref}       
\usepackage{url}            
\usepackage{booktabs}       
\usepackage{amsfonts}       
\usepackage{nicefrac}       
\usepackage{microtype}      
\usepackage{xcolor}         
\usepackage{booktabs}
\usepackage{siunitx}
\usepackage{multirow}

\usepackage{graphicx}
\usepackage{amsmath}
\usepackage{amsthm}
\usepackage{appendix}

\newtheorem{theorem}{Theorem}[section]

\newtheorem{lemma}[theorem]{Lemma}
\newtheorem{definition}[theorem]{Definition}

\title{Learning by solving differential equations}

\author{%
  Benoit Dherin$^*$ \\
  Google Research\\
  \texttt{dherin@google.com} \\
  \And
  Michael Munn$^*$ \\
  Google Research \\
  \texttt{munn@google.com} \\
  \AND
  Hanna Mazzawi \\
  Google Research \\
  \texttt{mazzawi@google.com} \\
  \And
  Michael Wunder \\
  Google Research \\
  \texttt{mwunder@google.com} \\
  \And
  Sourabh Medapati \\
  Google DeepMind \\
  \texttt{smedapati@google.com} \\
  \And
  Javier Gonzalvo \\
  Google Research \\
  \texttt{xavigonzalvo@google.com} \\
}

\begin{document}

\maketitle

\def\thefootnote{*}\footnotetext{These authors contributed equally to this work}
\def\thefootnote{\arabic{footnote}}

\begin{abstract}
Modern deep learning algorithms use variations of gradient descent as their main learning methods. Gradient descent can be understood as the simplest Ordinary Differential Equation (ODE) solver; namely, the Euler method applied to the gradient flow differential equation. Since Euler, many ODE solvers have been devised that follow the gradient flow equation more precisely and more stably. Runge-Kutta (RK) methods provide a family of very powerful explicit and implicit high-order ODE solvers. However, these higher-order solvers have not found wide application in deep learning so far. In this work, we evaluate the performance of higher-order RK solvers when applied in deep learning, study their limitations, and propose ways to overcome these drawbacks. In particular, we explore how to improve their performance by naturally incorporating key ingredients of modern neural network optimizers such as preconditioning, adaptive learning rates, and momentum.

\end{abstract}

\section{Introduction}

Optimization in deep learning is notoriously hard. One of the main difficulties is the rough nature of the optimization landscape which causes training instabilities or even convergence failures. Due to the high cost of training large-scale models, particularly modern LLMs, this has sparked a recent interest in optimization techniques which can mitigate these training instabilities (e.g., \cite{takase2025spike,liu2020admin,wortsman2024smallscale}). 

Motivated by the theoretical potential for improved stability offered by higher-order ODE solvers, this work investigates the practical application and adaptation of such Runge-Kutta methods to deep learning, focusing on their performance, limitations, and effective modifications.

The nature of the problem can be traced back to the behaviour of the gradient field for losses coming from neural network optimization. Namely, the loss gradient tends to have extreme variations both in magnitude and direction during the course of training \cite{mccandlish2018empirical,visualization,liu2020admin}.
This makes simple training strategies like Stochastic Gradient Descent (SGD) - where the raw gradient evaluated at a single point over a batch of data is used to update the weight - often inefficient and unstable. The best modern deep-learning optimizers are variations of SGD that rely on essentially two  important methods to modify the raw gradient in the SGD update in order to fix these unstable gradient issues. The first method consists of some form of averaging of the gradient over the training run yielding strategies like classical momentum \cite{Polyak1964,Rumelhart1986,Sutskever2013}. The averaging helps with canceling out the variations in gradient direction that are unhelpful. The second method, preconditioning, consists in multiplying the gradient by a matrix in order to lower gradient directions of high curvature to tame the gradient field. This has yielded methods like AdaGrad \cite{Duchi2011}, Adadelta \cite{Zeiler2012}, RMSProp \cite{Tieleman2012}. The most used generic method Adam \cite{kingma2015adam} (as well as its variants like Nadam \cite{dozat2016incorporating}, AdamW \cite{loshchilov2019decoupled}, or NadamW \cite{loshchilov2019decoupled}) use both schemes in conjunction, yielding optimizers that perform well on a wide range of settings. The most performant optimizers in the AlgoPerf leaderboard \cite{kasimbeg2025accelerating,dahl2023benchmarking} at the time of writing are methods using only a special form of preconditioning like Shampoo \cite{anil2020scalable,shi2023distributed}, or methods that leverage some special form of averaging, like the schedule-free AdamW from \cite{defazio2024road}. 

This work introduces a gradient update modification (beyond averaging and preconditioning) to address training instabilities. Raw gradient updates (e.g., $\theta' = \theta -h\nabla L(\theta)$) deviate from the continuous loss minimization path with an $\mathcal O(h^2)$ error, causing training oscillations. Drawing from numerical methods of ODEs \cite{hairer2006geometric}, we propose Runge-Kutta (RK) updates that track the gradient flow trajectory more precisely, leading to more stable optimization. In fact, many common deep learning optimizers can be viewed as first-order ODE solvers. For instance, gradient descent is Euler's method for the gradient flow ODE, and momentum methods can be derived from Hamiltonian systems \cite{ghosh2023implicit,kovachki2021continuous,franca2020,shi2019acceleration,muehlebach2019dynamical}. Our research explores higher-order RK methods. These higher-order approaches promise more accurate tracking of the true ODE solutions (i.e., the gradient flow ODE which continuously minimizes the loss) compared to their first-order counterparts, aiming for improved stability and performance.

Our main contributions are the following:
\begin{itemize}

    \item We conduct a benchmark of the classical ${4^\text{th}}$-order RK method and discuss the benefits and limitations of higher-order RK updates in deep learning; see Sec.~\ref{section:benchmarking rk}.
    \item We propose three modifications  which aim to address certain limitations of vanilla RK methods: preconditioning, adaptive learning rates, momentum; see Sec.~\ref{section:modifications to rk}.
    \item We demonstrate experimentally that these modifications do indeed confer benefits when training with RK optimizers; see Figs.~\ref{figure:bridging_the_gap_preconditioning}-\ref{figure:bridging_the_gap_momentum}.
\end{itemize}

\paragraph{Paper goal:}
The primary aim of this work is to explore the intersection of Runge-Kutta theory and deep learning optimization, demonstrating its relevance and motivating further research in the intersection of these two areas. We hope to show that while such established numerical solvers hold promise, their `out-of-the-box' application to neural network training is often challenging. Here we explore how to leverage theoretical insights to guide their effective adaptation. 

\subsection{Related Work}

\paragraph{RK methods in optimization.}There have been surprisingly few and limited studies on the application of higher-order RK methods \cite{hairer1993solving1,hairer_solving_ode_2,hairer2006geometric} to deep learning. To our knowledge there are only two bench-marking works, both comparing higher-order RK methods to SGD only. In \cite{ayadi2020runge_kutta}, the authors observe better performance of a second-order RK method, the Heun method (applied to the gradient flow ODE) on MNIST, Fashion-MNIST, and CIFAR-10 with a CNN architecture over SGD. 
In \cite{su2024improving}, the classical $4^{\text{th}}$-order RK method shows also improved test accuracy over SGD on the same workloads plus CIFAR-100, but with ResNet architectures. 
However, the paper provides limited specifics regarding hyperparameter configurations and the experiment setup, making it difficult to ascertain the precise extent of hyperparameter tuning applied in the comparisons.
A very interesting work \cite{franca2020} uses a second order RK method on a relativistic Hamiltonian system to obtain Relativistic Gradient Descent (RGD), which is shown to interpolate between classical momentum and Nesterov momentum depending on the values of some hyper-parameters. This RGD is not evaluated in the context of neural network optimization in that paper. Another noteworthy work in this context is \cite{qin2020training} where they used a gradient regularized RK method to train GANs with good success. Lastly, first-order implicit RK methods have been used with benefit in the case of Physics-Informed Neural Networks (PINN) \cite{li2023implicit,wang2021understanding}. Devising optimization schemes using ODE's has been also considered in \cite{brown1989some,fiscko2023towards}.
In a sort of converse way, a number of works have observed that central optimizers in deep learning can be realized as \emph{first-order} RK methods applied to specific differential equations: gradient descent, momentum methods \cite{ghosh2023implicit,kovachki2021continuous,franca2020,shi2019acceleration,muehlebach2019dynamical,Betancourt2018}, and accelerated gradient methods \cite{zhang2018direct,franca2020, Franca2021,franca2018admm,Su2016}. These findings point toward a potentially high impact of bridging RK theory with deep learning optimization, a connection we believe has been understudied. 

\paragraph{Preconditioning.} Most modern deep-learning optimizers involve a form of raw gradient modification by the application of a matrix to the gradient to decrease its variability \cite{fiscko2023towards,franca2018admm}. For instance, the Adam family uses at each step a diagonal preconditioning matrix formed by averaging the gradient squares over the trajectory \cite{Duchi2011,Zeiler2012,Tieleman2012,kingma2015adam,dozat2016incorporating,loshchilov2019decoupled,anil2020scalable,dahl2023benchmarking,franca2018admm}. The resulting matrix can be thought of as an approximation of the Hessian inverse. More recently, the Shampoo family of optimizers uses block diagonal approximations of the same Hessian inverse \cite{anil2020scalable,shi2023distributed}. Amari in \cite{amari1998natural}, defines the natural gradient as the preconditioning of the raw gradient with the inverse of a natural metric on the parameter space generated by the family of probability distributions. In Section \ref{section:preconditioning} we explain how to incorporate preconditioning into RK updates by performing local modifications of gradient flow, and we analyze a variation of the AdaGrad preconditioner \cite{Duchi2011}.

\paragraph{Momentum.} Another technique prevalent in deep-learning is that of averaging the gradient updates along the trajectory to tame the raw gradient. This technique is known as momentum \cite{Polyak1964,Rumelhart1986,Sutskever2013}. One can understand momentum as an application of a first-order RK scheme to a special Hamiltonian equation with friction using the loss as potential function \cite{ghosh2023implicit,kovachki2021continuous,franca2020,shi2019acceleration,muehlebach2019dynamical}. The use of higher-order RK methods to solve this Hamiltonian system with friction has been explored in \cite{franca2020} leading to improved stability. Here we take a different approach by simply averaging RK updates instead of raw gradient updates leading to improved test performance as explained in  Section \ref{section:momentum}.

\paragraph{Adaptive learning rates.}  Higher-order RK methods can follow the exact solutions of ODE very precisely \cite{hairer1993solving1}. However, they suffer when the ODE is stiff \cite{hairer_solving_ode_2}, which means that its vector field has large local variations. One common method for handling this is to use an adaptive step-size \cite{hairer_solving_ode_2} that automatically reduces the step-size in these regions of large variation. We propose a modification of a recent adaptive learning rate, the Drift-Adjusted Learning rate (DAL) from \cite{rosca2023on}, in conjunction with RK methods in Section \ref{section:adaptive_learning_rate}. Note that other learning rates like the Polyak step-size \cite{Polyak1964, orvieto2022dynamics,dherin2024corridors} or the NGN step-size \cite{orvieto2024adaptive} may also be worth studying in this context.

\section{Background on RK methods}
\label{section:background}

RK methods are a family of iterative methods that numerically approximate solutions to ODEs by calculating a weighted average of gradient estimates at several points within a single step (similarly to the lookahead optimizer from \cite{zhang2019lookahead} or the aggregated momentum from \cite{lucas2019aggregated}). These methods have an order associated to them, and the higher the order the more closely the numerical solution follows the exact solution. Neural network training can be seen as numerically solving the Gradient Flow (GF) differential equation. Since the exact solutions of GF are always stable, we expect that higher-order methods following these solutions more closely will benefit from increased stability in training loss. We explain this now in more detail.  

Consider a general ODE of the form $\dot \theta = f(\theta)$, where $f:\mathbb R^k\rightarrow \mathbb R^k$ is a function called the ODE \emph{vector field}. We use the notation $\dot \theta(t)$ to denote the derivative of the curve $\theta(t)$, which geometrically is the tangent vector to the curve at time $t$. A curve $\theta(t)$ is a solution of the differential equation $\dot \theta = f(\theta)$ if its tangent vector at $t$ is exactly given by the vector field $f(\theta)$ evaluated at $\theta(t)$.

\paragraph{Gradient Flow.} In the context of deep learning, the vector field is the negative gradient field: $f(\theta) := -\nabla L(\theta)$, where $L(\theta)$ is the loss, yielding the \emph{gradient flow} ODE: $\dot \theta = -\nabla L(\theta).$ The exact solutions $\theta(t)$ of this differential equation are the continuous paths of steepest descent, along which the loss decreases continuously and stably at the rate prescribed by the gradient norm; namely, along a gradient-flow solution $\theta(t)$, the loss derivative w.r.t.~time is negative:
$$
\frac {d L(\theta(t))}{dt} = \nabla L(\theta(t))^T \dot \theta(t) = - \| \nabla L(\theta(t))\|^2  \leq 0
$$
Following the solutions of this ODE closely, through numerical iterations, then ensures decreasing the loss in a way that is more stable (i.e. less loss oscillations) as we are closer to the gradient flow.

\paragraph{Order of an ODE solver.}
The simplest possible ODE solver is the Euler method, which in deep learning is called gradient descent: $$\theta' = \theta + h f(\theta),$$ which approximates the exact solution $\theta(t)$ of the differential equation $\dot \theta = f(\theta)$ after a time $t=h$ starting from the point $\theta$. Gradient descent is a \emph{first-order} solver, since its error with respect to the solution of gradient flow after a step of size $h$ is of order $\mathcal O(h^2)$. It is easy to see this by comparing the Taylor expansion of the exact solution of the ODE $\dot \theta = f(\theta)$ with the numerical method. The Taylor expansion of the exact solution $\theta(t)$ of $\dot \theta = f(\theta)$ is easy to compute:
\begin{eqnarray}
    \theta(h) 
    & = & \theta + \dot \theta(0) h + \frac 1{2!} \ddot \theta(0) h^2 + \mathcal O(h^3) \ \\
    & = & \theta + hf(\theta) + \frac {h^2}2 f'(\theta)f(\theta) + \mathcal O(h^3), \label{equation:exact_expansion} \label{equation:exact_solution_expansion}
\end{eqnarray}
where we obtained the second derivative of $\theta(t)$ by differentiating the ODE $\dot \theta = f(\theta)$ on both sides w.r.t.~to time: $\ddot \theta = f'(\theta) \dot \theta = f'(\theta) f(\theta)$. From the exact solution expansion in \eqref{equation:exact_solution_expansion}, we see immediately that the Euler method is just a first order expansion, since it coincides with the first-order term in the learning rate only yielding an error of size $\mathcal O(h^2)$. More generally, we say that
\begin{definition} An iterative ODE solver as above is of order $k$ if the error after one step is of order $\mathcal O(h^{k+1})$. In other words, for a method of order $k$, we have that
$$
\| \theta'- \theta(h)\| = \mathcal O(h^{k+1}).
$$
where $\theta'$ is the method iterator after one step of size $h$ from $\theta$ and $\theta(h)$ is the exact solution of the ODE starting at $\theta$ after a time $h$.
\end{definition}

\paragraph{Runge-Kutta methods for solving ODEs.}The Euler method is the simplest ODE solver from a vast family of ODE solvers. RK methods \cite{hairer1993solving1,hairer_solving_ode_2,hairer2006geometric} generalize the Euler method in that each step is computed from a weighted average of gradients evaluated in a neighborhood of the previous iterate:
\begin{equation} \label{equation:rk_method}
\theta' = \theta + h\sum_{i=1}^s b_i f(\theta_i(h)),
\end{equation}
where the $b_i$'s sum to one and the $\theta_i(h)$'s are points in a neighborhood of the previous iterate $\theta$ with $\theta_i(h) \rightarrow \theta$ as $h\rightarrow 0$, recovering the gradient flow for infinitesimally small learning rates. The neighborhood points $\theta_i(h)$ where the gradients are evaluated are computed themselves as a solution of an implicit system of equations\footnote{Note that if $a_{ij} = 0$ if $i\geq j$ then Eq. \eqref{equation:rk_neighoring_points} can be solved explicitly; in this case the method is called \emph{explicit}, otherwise it is called \emph{implicit}.} given by  
\begin{equation}\label{equation:rk_neighoring_points}
\theta_i = \theta + h \sum_{j=1}^s a_{ij}f(\theta_j).
\end{equation} 
 The main idea behind RK methods is that one can find special vectors $b = (b_i)$ and matrices $A = (a_{ij})$ such that the resulting ODE solver is of higher order (i.e., with lower error). This is done by expanding the solution of the differential equation as well as the RK method into a Taylor series and choosing the coefficients $a_{ij}$ and $b_i$ such that both series coincide up to the desired order.
For instance, the Taylor series up to 2nd order of the RK method in \eqref{equation:rk_method} is (see Appendix \ref{appendix:runge_kutta_updates} Eqs.~\eqref{equation:RK_taylor_expansion}-\eqref{equation:RK_taylor_expansion_last} for details):
\begin{equation}
\theta'  
 =  \theta +  h \left(\sum_{i=1}^s b_i\right) f(\theta) +
h^2 \left(\sum_{ij=1}^s b_i a_{ij}\right)f'(\theta)f(\theta) + \mathcal O(h^3).
\end{equation}
Comparing this with the exact solution Taylor expansion in \eqref{equation:exact_expansion}, we obtain the conditions
\begin{equation}\label{equation:rk_conditions}
\sum_{i=1}^s b_i = 1 
\quad\textrm{and}\quad
\sum_{ij=1}^s b_ia_{ij} = \frac 12
\end{equation}
which ensure the method is of order 1 and 2, respectively. Finding RK weights that satisfy these and higher-order conditions is a difficult problem. We provide in Appendix \ref{appendix:runge_kutta_updates} an illustration of the procedure by computing the weights for all possible second-order RK methods, and we give concrete coefficients for RK methods of order 3 and 4 as well.
The number of gradient evaluations used in the RK method is called the number of \emph{stages}. One key result is that for a method to be of order $s$ we need at least (and possibly more) $s$ stages. \emph{This means that improving the method order requires increasing the number of gradient evaluations at each step \cite{hairer1993solving1}}.

\section{Benchmarking Runge-Kutta Methods in Deep Learning}\label{section:benchmarking rk}

In the context of deep learning, we will call a \emph{vanilla RK update}, the gradient update obtained by applying a RK method $A = (a_{ij})$ and $b = (b_i)$ to the gradient flow ODE: $\dot \theta = -\nabla L(\theta)$, where $L$ is the neural loss. We will denote the RK gradient by $g^*(\theta, h)$ defined by
\begin{equation}
    g^*(\theta, h) = \sum_{i=1}^s b_i g(\theta_i) \quad \textrm{where} \quad
    \theta_i = \theta - h\sum_{j=1}^s a_{ij}g(\theta_j),
\end{equation}
where $g(\theta) = \nabla L(\theta)$ denotes the gradient of the loss $L(\theta)$.
The corresponding vanilla RK update with fixed learning rate $h$ is then:
\begin{equation} 
\theta' = \theta - hg^*(\theta, h).
\end{equation}
Observe that the RK gradient $g^*(\theta, h)$ depends on the step-size. This dependence allows the update to be closer to the exact gradient flow $\theta(h)$ that continuously minimizes the loss without any instability. In theory, this is the main advantage of this type of update: they are more stable than raw gradient updates for higher learning rates, allowing for faster and stable convergence.

\subsection{Empirical evaluation of vanilla Runge-Kutta}

Benchmarking new optimizers is notoriously challenging \cite{dahl2023benchmarking,kasimbeg2025accelerating}. This is in part because many of the tricks and techniques needed to achieve SOTA performance are often optimizer-specific and potentially disadvantage a new optimizer. For instance, standard practices like weight decay or batch normalization can conflict with RK's underlying mechanics, if not specifically adapted. We therefore adopt the following strategy: from a collection of highly tuned strong baseline settings, taken from workloads within the init2winit framework \cite{init2winit2021github, dahl2023benchmarking}, we assess vanilla RK's competitiveness by substituting the original optimizer and tuning only RK's learning rate and batch size. This provides a deliberately challenging setup but allows for a naive ``worst-case'' comparison.

We see that vanilla RK updates can be competitive (even within this adverse setting) with the additional benefit of having only the learning rate train to tune and no accumulators, as shown in Table \ref{table:rk_sota}.
Note that while vanilla RK excels on simple tasks, its competitiveness drops on more complex workloads. We suspect this is because common training tricks (e.g., batch norm, weight decay, cosine schedule), often necessary for such workloads, are detrimental to RK methods.

\begin{table}[htbp] 
  \centering
  \caption{Comparison of vanilla RK4 (see Appendix \ref{appendix:runge_kutta_updates}) against strong baselines for various workloads. Accuracy represents mean best test accuracy during training over a fixed number of training steps and standard error over 5 independent trials.
See Appendix \ref{appendix:sota_experiment_details} for experiment details. \label{table:rk_sota}}
  
  \begin{tabular}{
    l l c c | c
  }
    \toprule 
    \textbf{Dataset} & \textbf{Model} & {\shortstack[c]{\textbf{Baseline} \\ \textbf{Optimizer}}} & {\shortstack[c]{\textbf{Baseline} \\ \textbf{Accuracy (\%)}}} & {\shortstack[c]{\textbf{RK4} \\ \textbf{Accuracy (\%)}}} \\ 
    \midrule 

    MNIST           & DNN          &  Adam      & 96.64 $\pm$ 7e-4 & 96.88 $\pm$ 6e-4 \\
    MNIST           & Simple CNN   &  Momentum  & 98.64 $\pm$ 2e-4 & 98.53 $\pm$ 1e-4 \\
    Fashion-MNIST   & DNN          &  Adam      & 85.6  $\pm$ 4e-4 & 86.02 $\pm$ 4e-4 \\
    Fashion-MNIST   & Simple CNN   &  Momentum  & 90.12 $\pm$ 4e-4 & 90.65 $\pm$ 2e-3 \\
    CIFAR-10        & WRN 28-10    &  Momentum  & 97.10 $\pm$ 1e-4 & 96.78 $\pm$ 4e-3 \\
    CIFAR-100       & WRN 28-10    &  Momentum  & 81.89 $\pm$ 2e-3 & 78.81 $\pm$ 9e-4 \\
    ImageNet        & ViT          &  NAdamW    & 78.5 $\pm$ 2e-3 & 61.2 $\pm$ 2e-3 \\
    \bottomrule 
  \end{tabular}
\end{table}

\begin{figure}[htbp] 
\centering 

\includegraphics[width=0.49\textwidth]{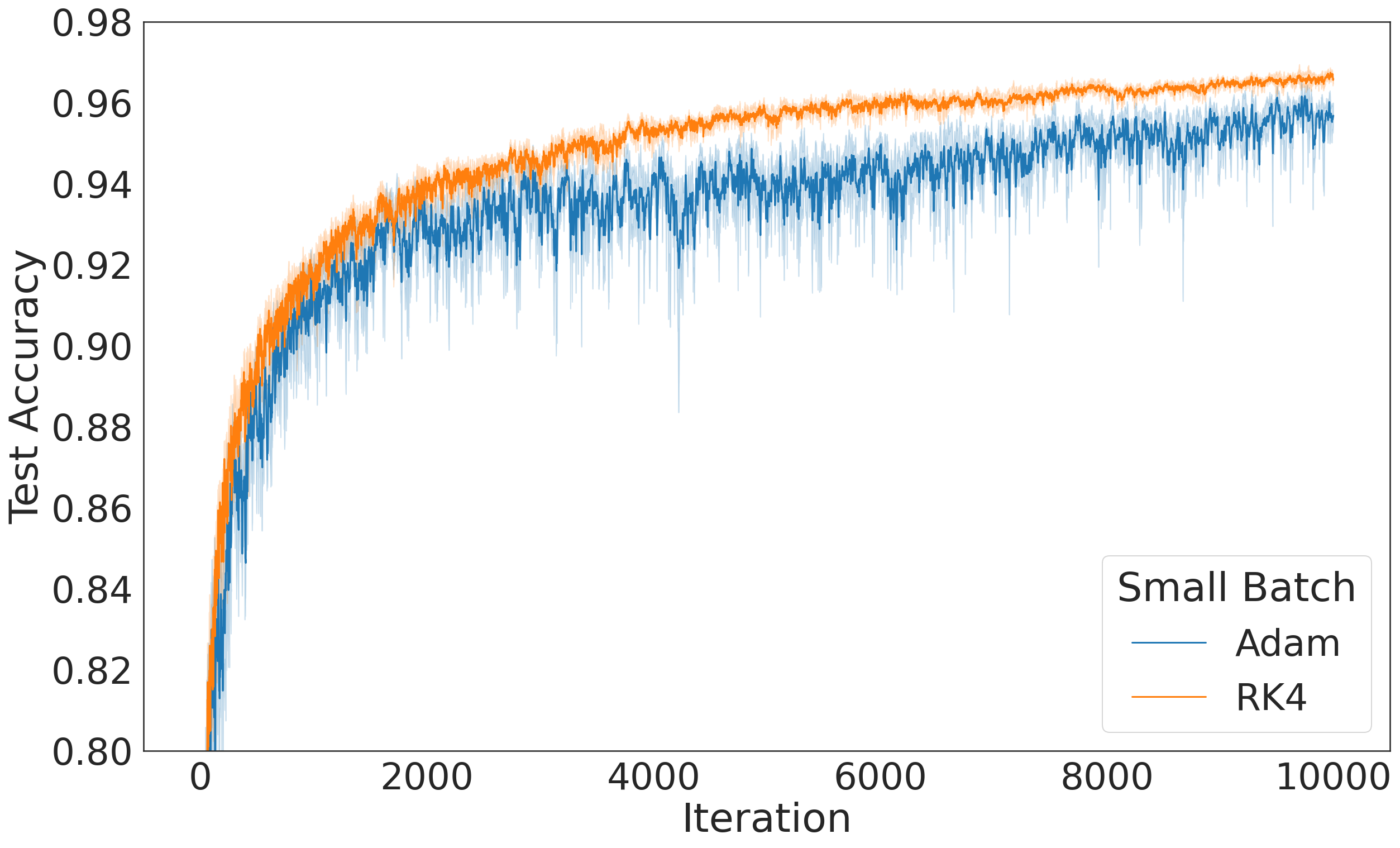}
\includegraphics[width=0.49\textwidth]{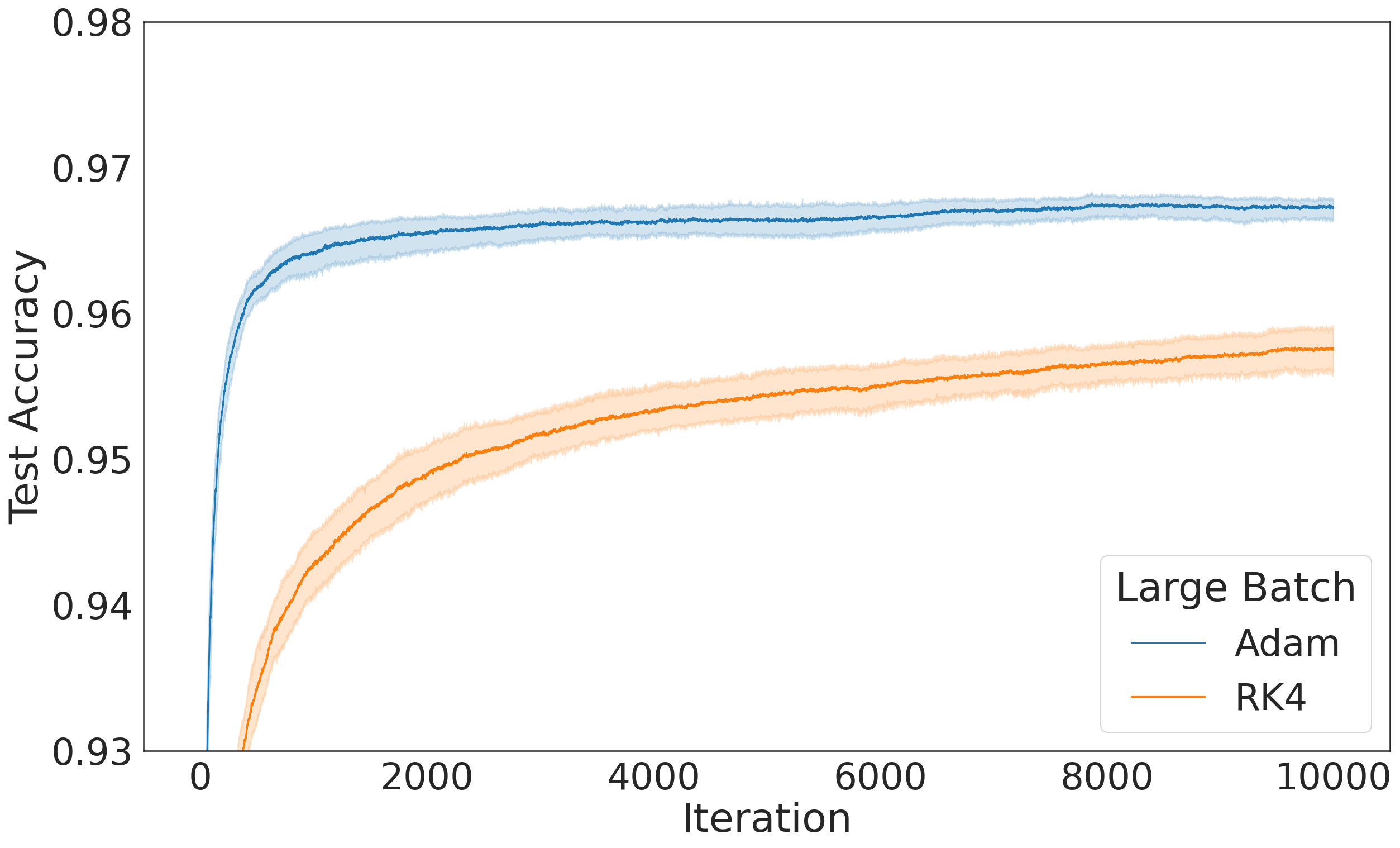}
\includegraphics[width=0.49\textwidth]{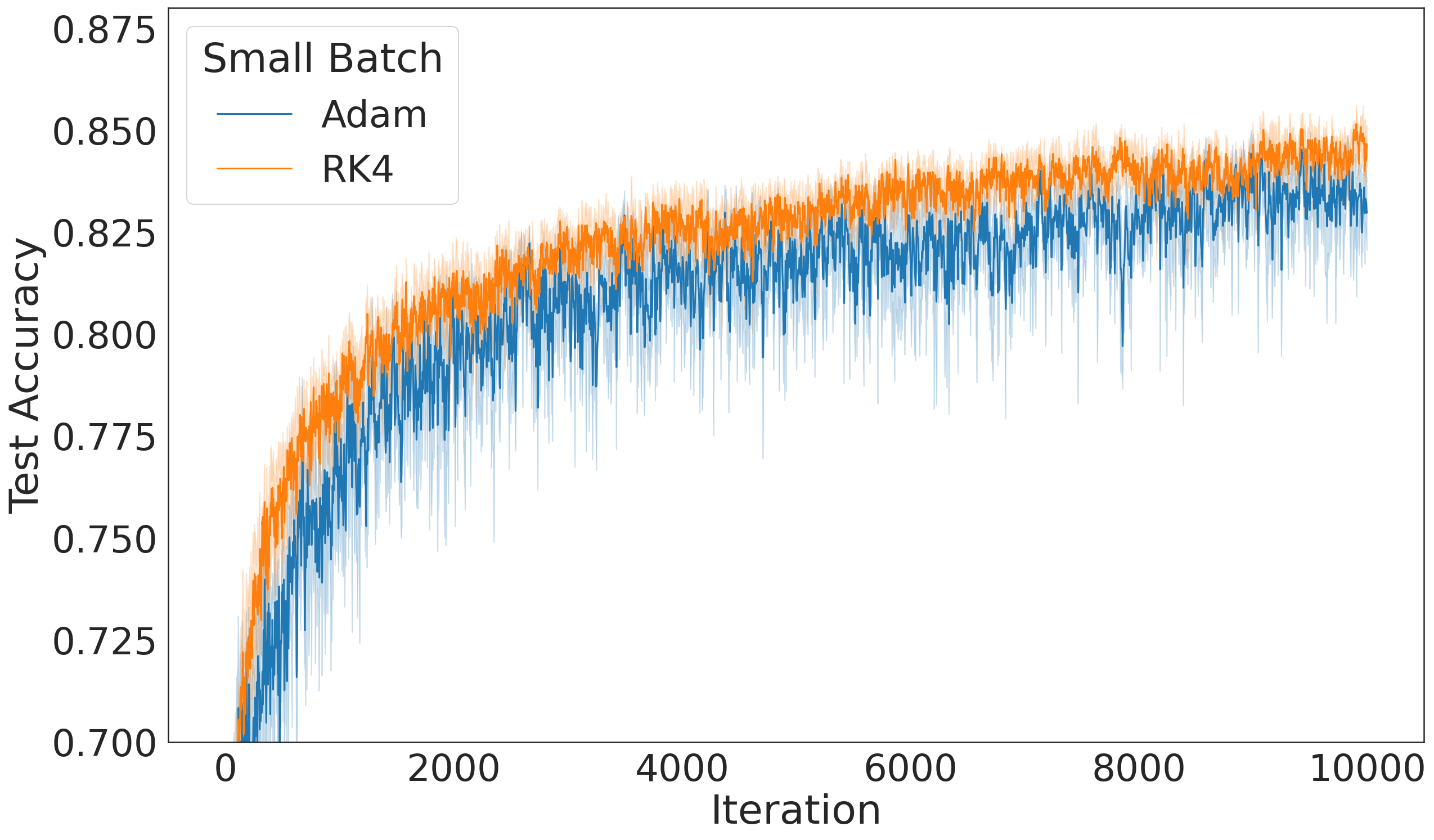}
\includegraphics[width=0.49\textwidth]{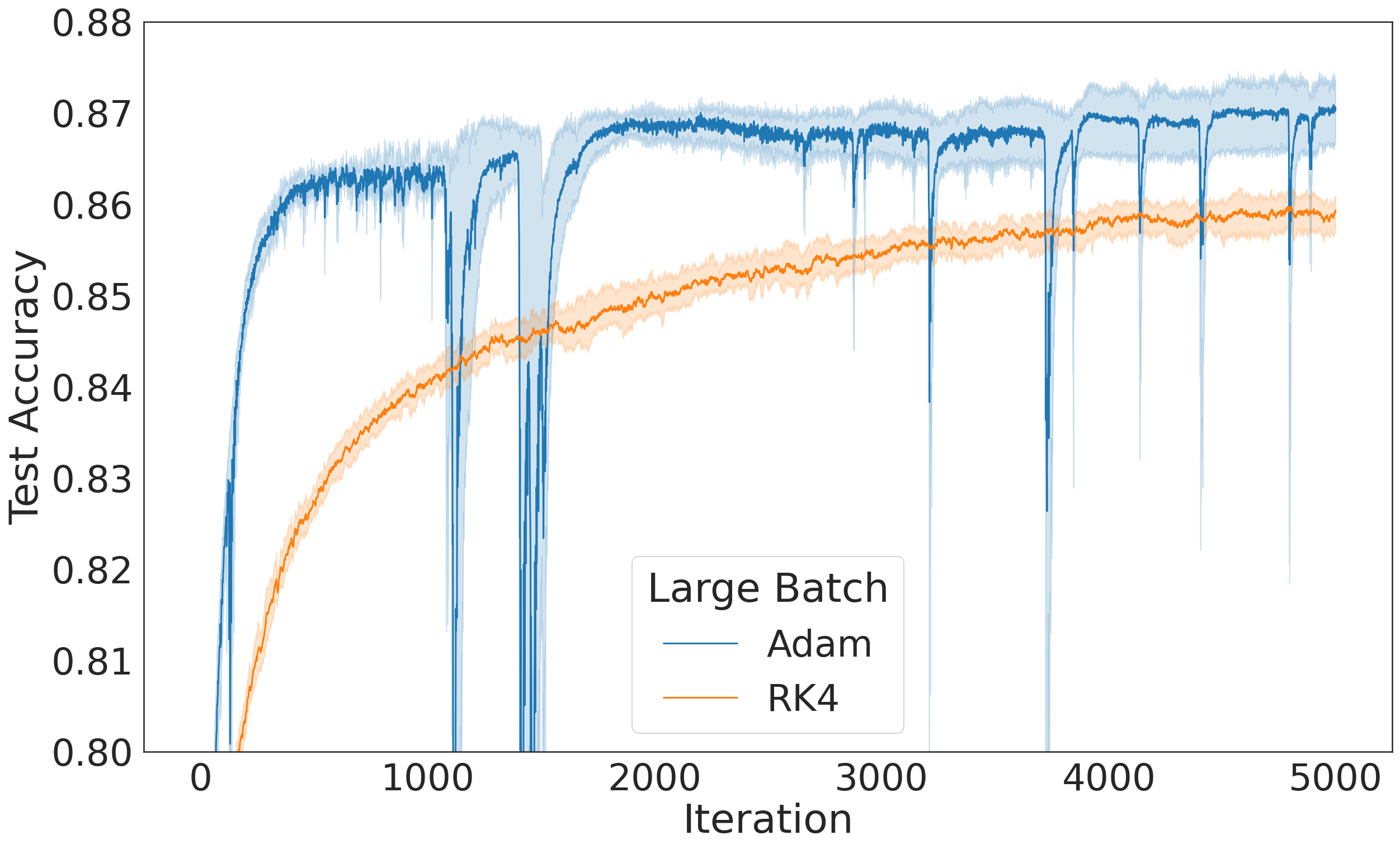}
\caption{
For MNIST ({\bf top row}) and Fashion MNIST ({\bf bottom row}) trained on MLP, vanilla RK4 achieves better test accuracy compared to Adam for small batches ({\bf left column}) while we observe a generalization gap for large batches ({\bf right column}). See Appendix \ref{appendix:figure_1} for experiment details.\label{figure:small_batch_vs_large_batch}
}
\end{figure}

\subsection{Limitations}
\label{section:limitations}
Although theoretically Runge-Kutta methods can provide benefits in training stability and simplicity in hyperparameter tuning, there are still practical challenges when naively applied to deep learning workloads. Most notably, these challenges are 1) the increase in compute time per step, 2) convergence difficulties in case of gradient field stiffness, and 3) a generalization gap in the large batch setting. These, which describe below, can all impact the overall efficiency and performance of RK methods in many settings.

\paragraph{Wall-clock time v.s.~steps for vanilla RK.}

An order-$k$ RK update requires at least $k$ gradient evaluations per step. While seemingly costly, in fact RK's wall-clock time is comparable to Adam's provided all gradient data fits into device memory simultaneously. This can be seen with full-batch MNIST/Fashion MNIST training, where RK4 and Adam times are nearly identical; see Appendix \ref{appendix:figure_1}  Fig.~\ref{figure:time_versus_step}. However, RK's wall-clock time sharply increases if gradient data exceeds memory capacity. For instance, when training a ResNet model on CIFAR-10, RK4's time is similar to Adam's for batch size of 1 but roughly doubles for batch size of 8192; see Appendix \ref{appendix:figure_time_vs_step_cifar} Fig.~\ref{figure:time_vs_step_cifar}. Consequently, performance can become prohibitive for large workloads under current memory constraints. In short, RK updates are  competitive in I/O-bound situations (where data fits in memory) but face significant wall-clock time challenges in memory-constrained scenarios.

\paragraph{Gradient flow stiffness.}

Runge-Kutta methods struggle with \emph{stiff} ODEs $\dot \theta = f(\theta)$; i.e., where the vector field $f(\theta)$ exhibits large variations in reaction to small changes in $\theta$, see \cite{hairer_solving_ode_2}. In neural network optimization, such stiffness can arise from phenomena influencing the vector field at different scales such as interactions between smooth large-scale geometry and bumpy smaller-scale geometry of the loss surface \cite{visualization}. This issue has been previously addressed in Physics-Informed Neural Networks (PINNs) using implicit first-order RK methods where stiffness has been shown to correlate with large Hessian eigenvalue gaps \cite{li2023implicit,wang2021understanding}. While implicit RK methods generally manage stiffness better than explicit ones \cite{hairer_solving_ode_2}, the computational cost of solving their required non-linear systems given by \eqref{equation:rk_neighoring_points} is often prohibitive for higher-order methods in neural networks. The following sections investigate modifications to the vanilla RK4 method which aim to mitigate this stiffness and help lessen the generalization gap we see in the large batch setting; cf.~Fig.~\ref{figure:small_batch_vs_large_batch}.

\paragraph{Generalization gap in large batch settings.}

In large batch settings, Runge-Kutta (RK) methods are stable because they closely follow the gradient flow. This, however,  also prevents them from benefiting from the beneficial flatness bias induced by first-order discretization errors \cite{barrett2021implicit,cattaneo2023implicit,ghosh2023implicit}. The absence of this implicit regularization leads to a generalization gap causing RK methods to underperform first-order methods like Adam \cite{cattaneo2023implicit}, even on simple datasets such as MNIST and Fashion-MNIST in the large batch setting; see Fig.~\ref{figure:small_batch_vs_large_batch} (right column). On the other hand, for small batch settings, the implicit regularization due to the stochastic noise helps RK4 bridge the generalization gap existing in the large batch setting; see Fig.~\ref{figure:small_batch_vs_large_batch} (left column).

In the next section, we propose three modifications which aim to combat ODE stiffness and allow us to bridge the generalization gap in the large batch setting observed in Fig.~\ref{figure:small_batch_vs_large_batch}. Namely, we explore 1) modifying the gradient field via preconditioning, 2) adaptive learning rates, and 3) adding momentum.

\section{Modifications to vanilla Runge-Kutta optimization}
\label{section:modifications to rk}

In this section we explore how three modern training techniques—preconditioning, adaptive learning rates, and momentum—can be integrated with vanilla RK methods to mitigate ODE stiffness and improve performance, allowing us to bridge the generalization gap observed in Fig.~\ref{figure:small_batch_vs_large_batch}. For our experiments, we focus on ``pure'' MLP-based workloads, specifically MNIST and Fashion MNIST, since competitive performance is achievable without complex, optimizer-specific tricks.

\subsection{Preconditioning Runge-Kutta by modifying the ODE}
\label{section:preconditioning}

As discussed in Section \ref{section:limitations}, stiffness presents a significant challenge when applying RK methods to the gradient flow ODE of a neural loss, as it can cause the gradient field to change dramatically in response to even minor parameter variations. One way to mitigate this problematic behavior is through preconditioning; e.g., \cite{Duchi2011, diederik2014adam, amari1998natural, martens2015optimizing, gupta2018shampoo}, see also \cite{xie2025structured} and \cite{bernstein2024modular}. One way to view preconditioning is that it offers the flexibility to transform the gradient flow ODE into one which alleviates this stiffness problem. That is, at each step $k$ of the optimization process, we modify our differential equation $
    \dot \theta = f_k(\theta)
$
and solve it with an RK method with step $h$ and initial condition $\theta_k$ yielding a different RK update than the original RK update for step $k$. The only requirement is that the neural loss $L$ decreases along the exact solutions of $\dot \theta = f_k(\theta)$ and that critical points of the vector field $f_k(\theta)$ remain the same as the critical points of the loss gradient. We can test this by computing the derivative of the loss w.r.t.~time along the solutions of $\dot \theta = f_k(\theta)$ remains negative.\footnote{In mathematical terms, this means that we can choose any differential equation at step $k$ as long as the loss function remains a Lyapunov function for that differential equation \cite{bacciotti2001liapunov}.}

Using this trick, we modify the gradient flow ODE to precondition the gradient. The following lemma outlines this step-wise modification using a matrix (potentially dependent on the step and parameters), which generalizes prior work \cite{Cortes2006, fiscko2023towards}.

\begin{lemma}[Preconditioning lemma]
\label{lemma:preconditioning}
Consider the preconditioned gradient flow  \mbox{$\dot \theta(t) = - A(\theta(t)) g(\theta(t))$}, where $g(\theta) = \nabla L(\theta)$ is the loss gradient and $A(\theta)$ is a positive definite symmetric matrix possibly depending on the parameter $\theta$. Then, the loss $L$ decreases on the exact solutions $\theta(t)$ of the preconditioned gradient flow equation at the following rate
\begin{equation}
    \frac{dL(\theta(t))}{dt} = - \|g(\theta(t))\|^2_A,
\end{equation}
where $\|v\|^2_A = v^TAv$ is the norm of $v$ in the metric given by $A$.
\end{lemma}
\begin{proof}
 We obtain immediately that $\frac{dL(\theta(t))}{dt} = \nabla L(\theta(t))^T \dot \theta(t) = -g(\theta(t))^T A(\theta(t)) g(\theta(t))$. The positivity of $\|g(\theta(t))\|^2_A$ for all vectors is guaranteed by the fact that $A(\theta)$ is a Riemannian metric (i.e., $A(\theta)$ is symmetric and positive definite for all $\theta$).
\end{proof}

\paragraph{AdaGrad-like preconditioning for RK updates.} 
AdaGrad \cite{Duchi2011} is one of the earliest forms of adaptive preconditioning and works by adapting the learning rate for each parameter based on the history of its gradients. More specifically, AdaGrad  uses a matrix $G_n = \sum_{l=1}^n g(\theta_l)g(\theta_l)^T$ to accumulate past gradient outer products, capturing information about the loss-surface metric through the learning trajectory. Viewed through the lens of Lemma \ref{lemma:preconditioning}, for step $n$, Adagrad takes
\begin{equation}
A_n(\theta) = (\epsilon + \operatorname{diag} G_n)^{-1/2},
\end{equation}

where here $\epsilon$ is a small constant introduced only to improve numerical stability.

The insight that preconditioning aims to smooth the loss surface's local geometry and reduce ODE stiffness points towards a more natural choice for the preconditioning matrix $A_n$. Specifically, a loss surface given by $L(\theta)$ possesses a natural geometry  \cite{pouplin2023curvature,barrett2021implicit}, with its local distance metric characterized by the matrix
\begin{equation}
\mathcal{G}(\theta) = \mathbb{I} + g(\theta)g(\theta)^T, \quad \text{where}\quad g(\theta) = \nabla L(\theta).
\end{equation}

Thus, preconditioning by $\mathcal{G}^{-1}(\theta)$ naturally scales the gradient field, thereby smoothing the local geometry. However, computing $\mathcal{G}^{-1}(\theta)$ is often impractical due to the matrix inversion. Inspired by AdaGrad, we instead use a diagonal approximation of $g(\theta) g(\theta)^T$ and (similar to AdaGrad) average over the trajectory which leads us to our \emph{modified AdaGrad} preconditioner given by
\begin{equation}
A'_n(\theta) = (1 + \operatorname{diag} G_n)^{-1/2}.
\end{equation}

This yields a preconditioner similar to AdaGrad's but which performs better than AdaGrad in our experiments and succeeds in decreasing the generalization gap with Adam in the large batch regime; as shown in  Fig.~\ref{figure:bridging_the_gap_preconditioning}. See Appendix \ref{appendix:geometry_of_preconditioning} for more details around this preconditioning choice.

\begin{figure}[htbp] 
\centering 
\includegraphics[width=0.49\textwidth]{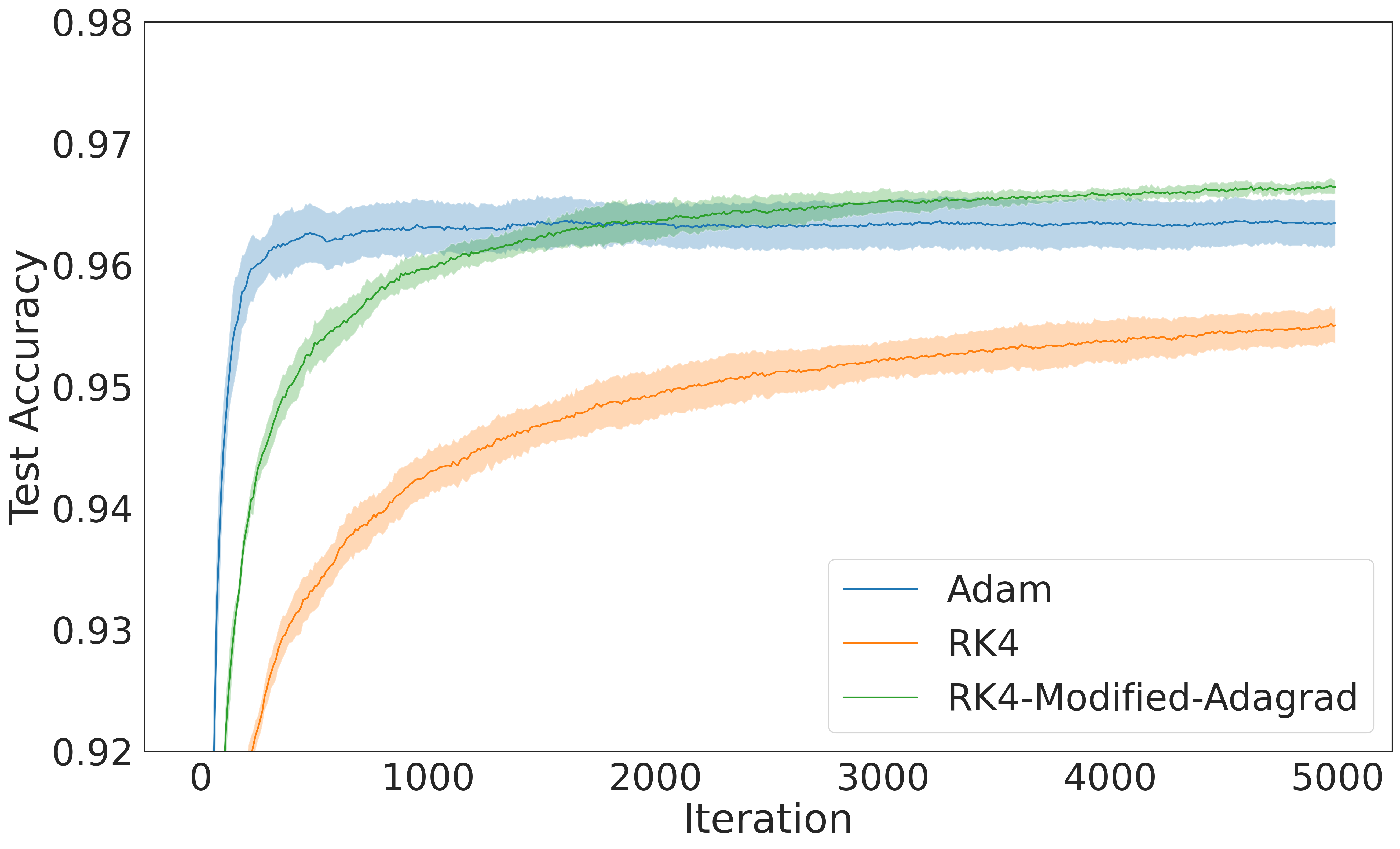}
\includegraphics[width=0.49\textwidth]{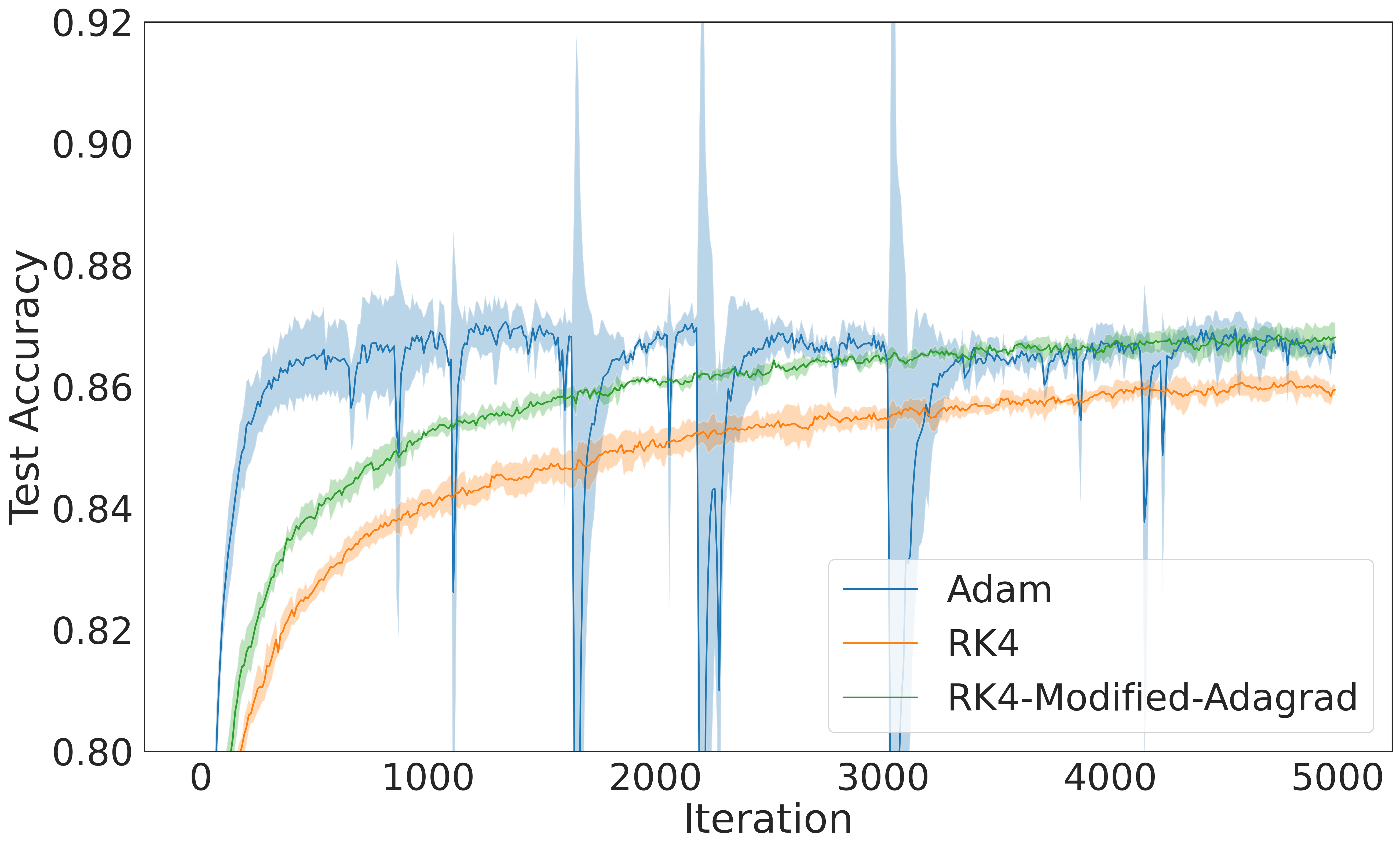}
\caption{
RK4 with AdaGrad-like preconditioning helps bridge the gap between Adam and vanilla RK4 in the large batch regime for MLP trained on MNIST (\textbf{left}) and Fashion-MNIST (\textbf{right}). See Appendix \ref{appendix:bridging_the_gap_preconditioning} for experiment details. Additional experiments on CIFAR-10 and a ResNet-18 model yield similar results; see Fig.~\ref{figure:conditioners_cifar} and Appendix \ref{section:cifar10_rk4_modified_adagrad}. \label{figure:bridging_the_gap_preconditioning}}
\end{figure}

\subsection{Making use of adaptive learning rates}
\label{section:adaptive_learning_rate}
Adaptive learning rates provide a promising tool for addressing the challenge of stiffness of the ODE. By automatically decreasing the learning rate in stiff regions and increasing the learning rate in more tame regions, we are able to more fully benefit from the precision offered by higher-order RK methods. The appropriate size for the learning rate is determined by measuring how the gradient varies in a neighborhood of a given point which, following \cite{dherin2022why}, can be computed via the derivative of the gradient along the gradient flow:
\begin{equation}
\label{equation:derivative_of_gradient}
\frac {d g(\theta)}{dt} = \nabla g(\theta) \dot \theta = - H(\theta)g(\theta),
\end{equation}
where $g$ and $H$ denote the gradient and the Hessian of the loss, respectively. 

Equation \eqref{equation:derivative_of_gradient} implies that the gradient $g(\theta)$ along the gradient flow path changes slowly when the term $\|H(\theta)g(\theta)\|$ is small, and changes rapidly when this term is large. Therefore, an adaptive learning rate should be larger precisely when $\|H(\theta)g(\theta)\|$ is small (reflecting these slow gradient changes) and smaller when it is large (reflecting rapid changes). In short, the learning rate should be inversely proportional to $\|H(\theta)g(\theta)\|$, leading to the adaptive learning rate DAL-$p$ as in \cite{rosca2021discretization}:
\begin{equation}
    h_{\textrm{DAL}}(\theta) := 2\left(\dfrac{\|g(\theta)\|}{\|H(\theta)g(\theta)\|}\right)^p, \quad \text{for some } p >0.
\end{equation}
Unfortunately, using $h_{\textrm{DAL}}$ directly with RK methods can cause divergence, as it may lead to excessively large learning rates. To mitigate this instability, we introduce a rescaling technique similar to the Polyak step-size normalization in \cite{orvieto2024adaptive}, resulting in the rescaled DAL-$p$ learning rate given by
\begin{equation}
\label{equation:dalr}
    h_{\textrm{DALR}}(\theta) := \dfrac{c}{1 + \dfrac c2 \left(\dfrac{\|H(\theta)g(\theta)\|}{\|g(\theta)\|}\right)^p}.
\end{equation}
The new learning rate \eqref{equation:dalr} is now capped by the (tunable) value $c$ and substantially improves upon vanilla RK4, even beating tuned Adam on MNIST and Fashion-MNIST; as shown in Fig.~\ref{figure:bridging_the_gap_adaptive_learning_rate}.

\begin{figure}[htbp] 
\centering 
\includegraphics[width=0.49\textwidth]{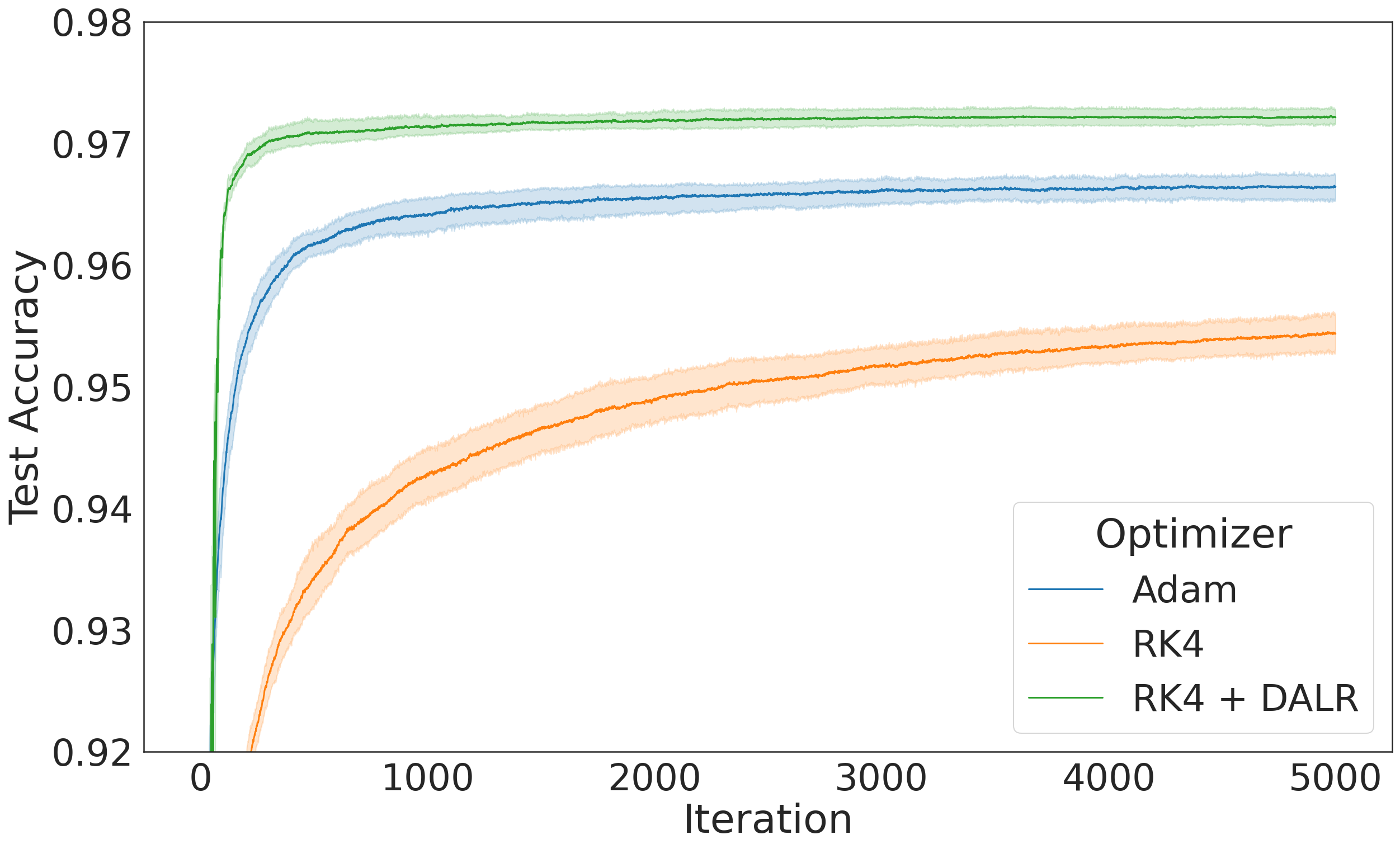}
\includegraphics[width=0.49\textwidth]{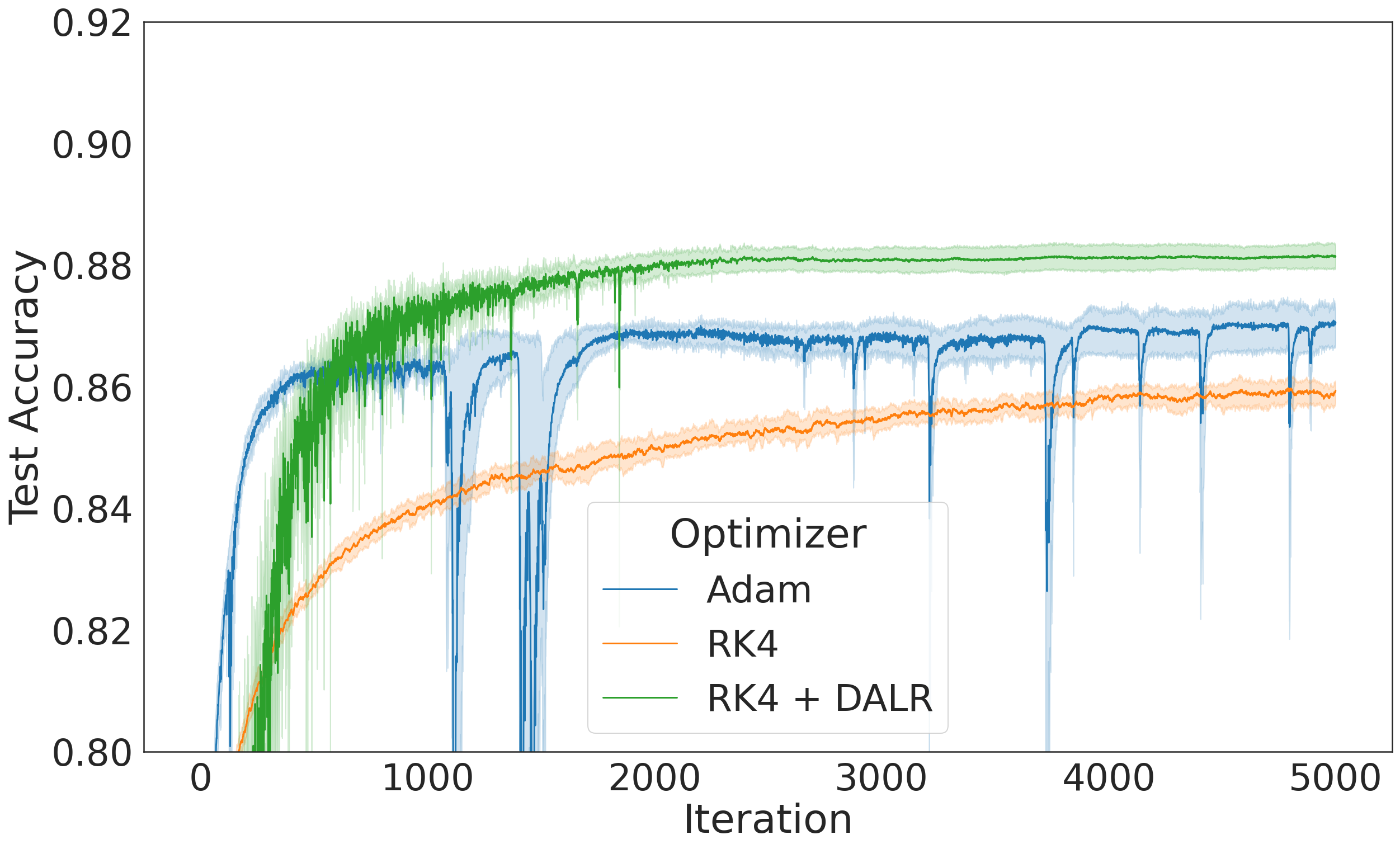}
\caption{
RK4 combined with an adaptive learning rate bridges the generalization gap in the large batch setting and is competitive with Adam for MLP trained on MNIST (\textbf{left}) and Fashion-MNIST (\textbf{right}). See Appendix \ref{appendix:bridging_the_gap_adaptive_learning_rate} for experiment details. Additional experiments on CIFAR-10 and a ResNet-18 model yield similar results; see Fig.~\ref{figure:cifar_momentum_dalr} and Appendix \ref{section:cifar10_rk4_momentum_adaptive_lr}.  \label{figure:bridging_the_gap_adaptive_learning_rate}}
\end{figure}

\subsection{Incorporating momentum with Runge-Kutta updates}
\label{section:momentum}
Momentum \cite{Sutskever2013} is a core component of many state-of-the-art optimizers like Adam \cite{kingma2015adam} and its variants (e.g., Nadam \cite{dozat2016incorporating}, AdamW \cite{loshchilov2019decoupled}). Traditionally, momentum based optimizers employ an exponential moving average (EMA) of past raw gradients, weighted by a hyperparameter $\beta$. We adapt this principle by applying the EMA to the RK gradients instead of the raw gradients, yielding the following RK momentum scheme
\begin{eqnarray}
    m_{n+1} & = & \beta m_n + g^*(\theta_n, h) \\
    \theta_{n+1} & = & \theta_n - h m_{n+1},
\end{eqnarray}
where the lookahead step uses a Runge–Kutta estimate of the gradient $g^*(\theta_n, h)$ in place of the standard Nesterov approximation~\cite{Su2016}, and where $m_0=0$. Accumulating these more precise gradients is beneficial and helps bridge the gap between vanilla RK and Adam, as shown in Fig.~\ref{figure:bridging_the_gap_momentum}. 

\begin{figure}[htbp] 
\centering 
\includegraphics[width=0.49\textwidth]{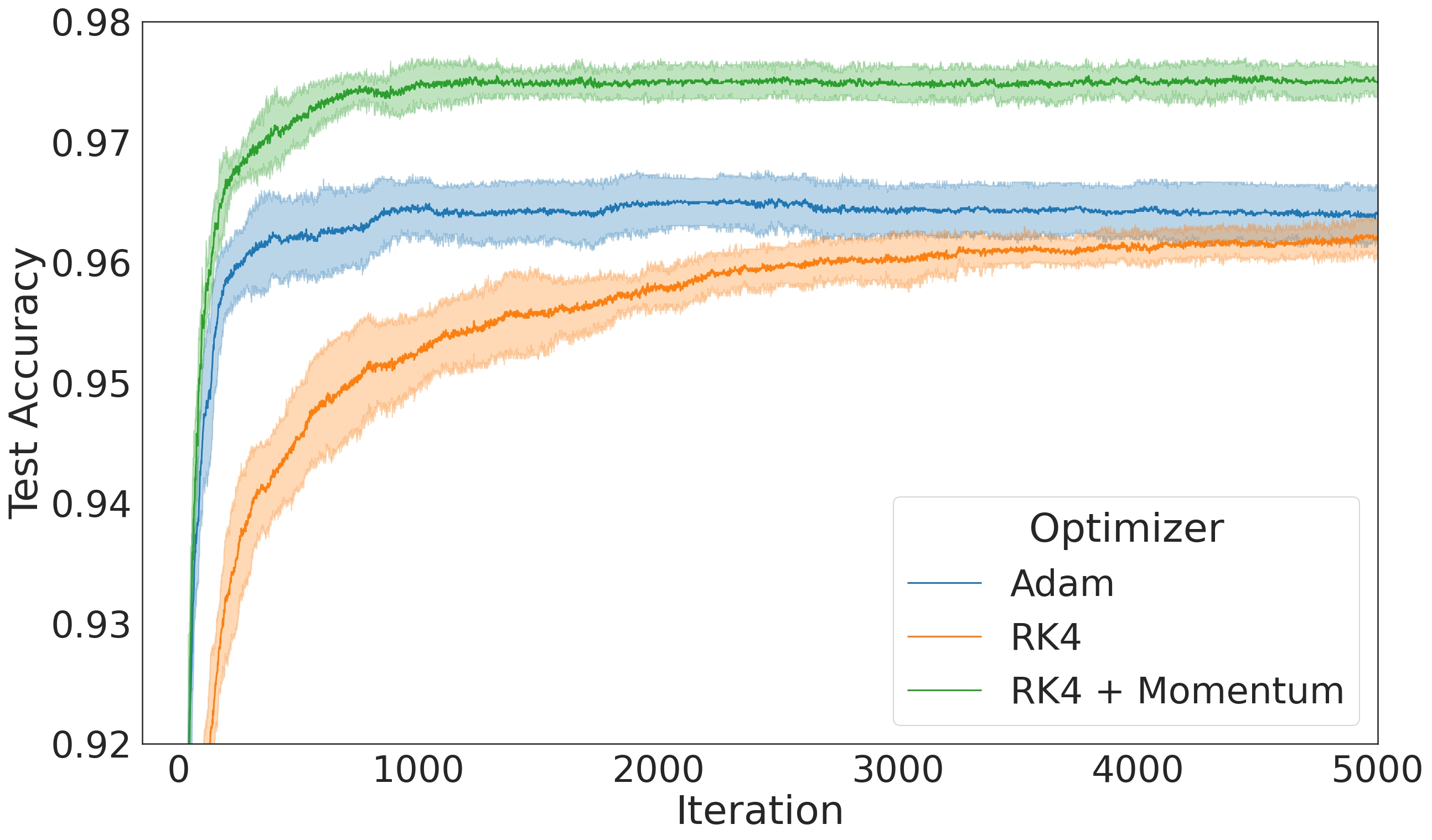}
\includegraphics[width=0.49\textwidth]{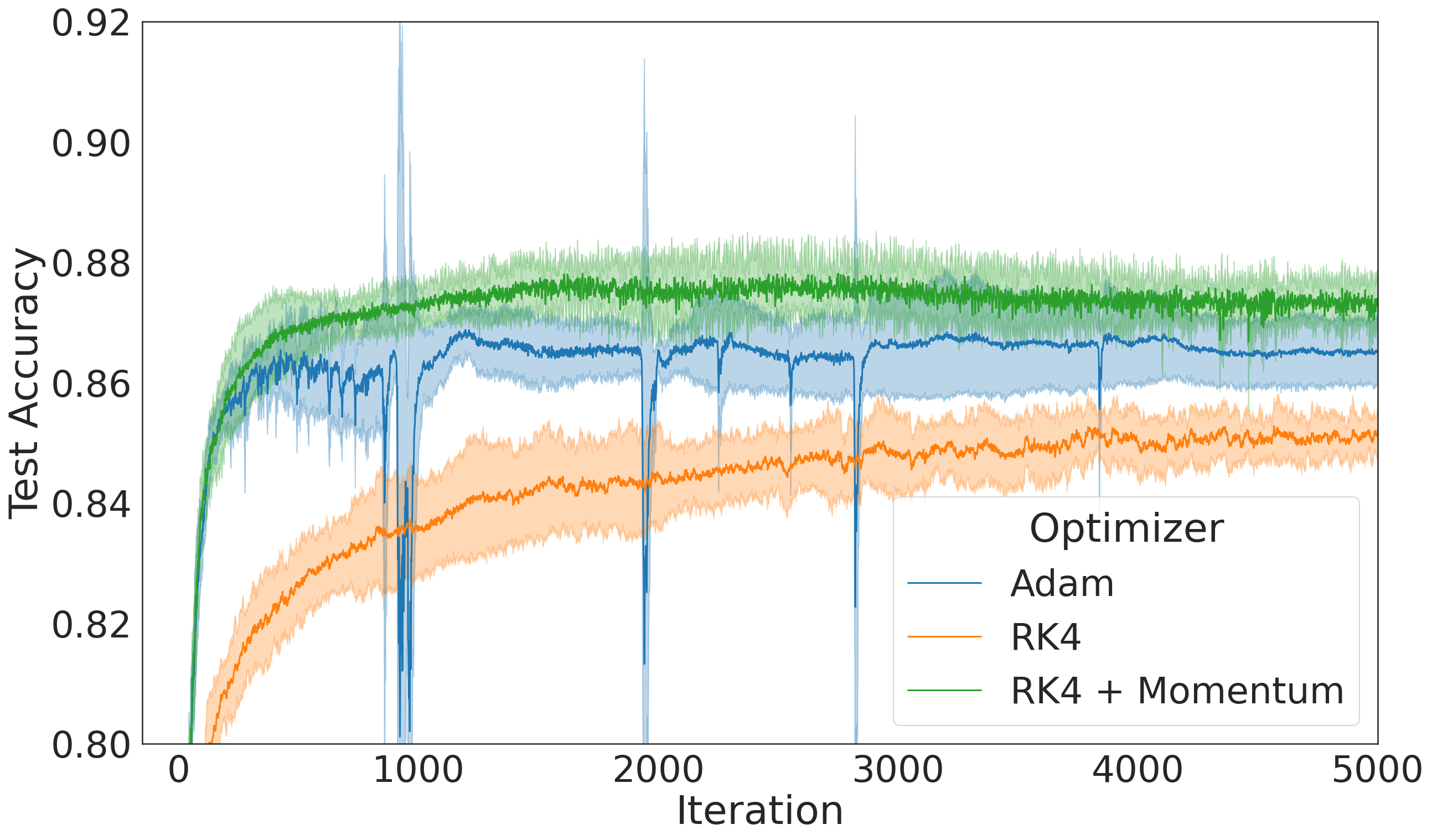}
\caption{Combining momentum with RK4 bridges the generalization gap in the large batch regime and achieves better test accuracy than both  Adam and vanilla RK4 on MNIST (\textbf{left}) and Fashion MNIST (\textbf{right}). See Appendix \ref{appendix:bridging_the_gap_momentum} for experiment details. Additional experiments on CIFAR-10 and a ResNet-18 model yield similar results; see Fig.~\ref{figure:cifar_momentum_dalr} and Appendix \ref{section:cifar10_rk4_momentum_adaptive_lr}. \label{figure:bridging_the_gap_momentum} }
\end{figure}

\section{Conclusion and Future Work}

This paper explores the application and potential of higher-order Runge-Kutta (RK) ODE solvers in deep learning. These higher-order RK methods follow more precisely the exact solutions of the gradient flow ODE minimizing the loss, yielding more stable training paths. We initiated our study with an empirical evaluation of vanilla ${4^\text{th}}$-order RK optimizers against strong baseline settings on a variety of classic deep learning workloads, addressing both their benefits and limitations. We then explored how  modern training techniques (preconditioning, adaptive learning rates and momentum) could be adapted to address these limitations, verifying through experiments that these modifications indeed improve upon vanilla RK4 and, in some cases, outperform Adam as well.

Despite the promising results, this study has certain limitations. Our investigation into the various modifications to vanilla RK4 methods has thus far deliberately focused on relatively simple workloads such as MNIST/Fashion-MNIST and MLP models. This focused approach was chosen because it allows us to more precisely isolate and measure the effects of each modification when compared to vanilla RK4, minimizing the confounding effects introduced by complex, optimizer-specific training tricks (e.g., schedules, batch normalization, weight decay, label smoothing, etc) which are often necessary to achieve competitive performance in larger-scale settings. Further investigation into how RK-based optimizers can be effectively tailored or how these heuristics might be co-adapted to allow for competitive performance on more complex workloads is an interesting area of future research. Ultimately, this work seeks to help bridge the gap between deep learning optimization and numerical methods of ODEs, highlighting this intersection as a promising area for future research where established solvers can significantly advance deep learning techniques.

\section{Acknowledgments}
\label{section:acknowledgements}
We would like to thank Joe Toth, George Dahl, Sammy Jerome, and Corinna Cortes for helpful discussions, suggestions, and feedback during the development of this work.

\bibliographystyle{plain}

\begin{thebibliography}{10}

\bibitem{amari1998natural}
Shun-ichi Amari.
\newblock Natural gradient works efficiently in learning.
\newblock {\em Neural computation}, 10(2):251--276, 1998.

\bibitem{anil2020scalable}
Rohan Anil, Vineet Gupta, Tomer Koren, Kevin Regan, and Yoram Singer.
\newblock Scalable second order optimization for deep learning.
\newblock {\em arXiv preprint arXiv:2002.09018}, 2020.

\bibitem{ayadi2020runge_kutta}
Imen Ayadi and Gabriel Turinici.
\newblock Stochastic runge-kutta methods and adaptive sgd-g2 stochastic
  gradient descent.
\newblock In {\em 2020 25th International Conference on Pattern Recognition
  (ICPR)}, 2021.

\bibitem{bacciotti2001liapunov}
Andrea Bacciotti and Lionel Rosier.
\newblock {\em Liapunov Functions and Stability in Control Theory}.
\newblock Lecture Notes in Control and Information Sciences. Springer, London,
  1 edition, 2001.

\bibitem{barrett2021implicit}
David~G.T. Barrett and Benoit Dherin.
\newblock Implicit gradient regularization.
\newblock In {\em ICLR}, 2021.

\bibitem{bernstein2024modular}
Jeremy Bernstein and Laker Newhouse.
\newblock Modular duality in deep learning.
\newblock {\em arXiv preprint arXiv:2410.21265}, 2024.

\bibitem{Betancourt2018}
Michael Betancourt, Michael~I Jordan, and Ashia~C Wilson.
\newblock On symplectic optimization.
\newblock {\em arXiv preprint arXiv:1802.03653}, 2018.

\bibitem{brown1989some}
A.~A. Brown and M.~C. Bartholomew-Biggs.
\newblock Some effective methods for unconstrained optimization based on the
  solution of systems of ordinary differential equations.
\newblock {\em J. Optim. Theory Appl.}, 62(2):211–224, August 1989.

\bibitem{cattaneo2023implicit}
Matias~D Cattaneo, Jason~M Klusowski, and Boris Shigida.
\newblock On the implicit bias of adam.
\newblock {\em arXiv:2309.00079}, 2023.

\bibitem{Cortes2006}
J.~Cortés.
\newblock Finite-time convergent gradient flows with applications to network
  consensus.
\newblock {\em Automatica}, 42(11):1993--2000, 2006.

\bibitem{dahl2023benchmarking}
George~E. Dahl, Frank Schneider, Peter Mattson, et~al.
\newblock Benchmarking neural network training algorithms.
\newblock {\em arXiv preprint arXiv:2306.07179}, 2023.

\bibitem{defazio2024road}
Aaron Defazio, Xingyu Yang, Harsh Mehta, Konstantin Mishchenko, Ahmed Khaled,
  and Ashok Cutkosky.
\newblock The road less scheduled.
\newblock {\em arXiv preprint arXiv:2405.15682}, 2024.

\bibitem{dherin2022why}
Benoit Dherin, Michael Munn, Mihaela Rosca, and David Barrett.
\newblock Why neural networks find simple solutions: The many regularizers of
  geometric complexity.
\newblock In {\em NeurIPS}, 2022.

\bibitem{dherin2024corridors}
Benoit Dherin and Mihaela Rosca.
\newblock Corridor geometry in gradient-based optimization, 2024.

\bibitem{diederik2014adam}
Kingma Diederik.
\newblock Adam: A method for stochastic optimization.
\newblock {\em (No Title)}, 2014.

\bibitem{dosovitskiy2020image}
Alexey Dosovitskiy, Lucas Beyer, Alexander Kolesnikov, Dirk Weissenborn,
  Xiaohua Zhai, Thomas Unterthiner, Mostafa Dehghani, Matthias Minderer, Georg
  Heigold, Sylvain Gelly, et~al.
\newblock An image is worth 16x16 words: Transformers for image recognition at
  scale.
\newblock {\em arXiv preprint arXiv:2010.11929}, 2020.

\bibitem{dozat2016incorporating}
Timothy Dozat.
\newblock Incorporating nesterov momentum into adam.
\newblock In {\em ICLR Workshop}, 2016.

\bibitem{Duchi2011}
John Duchi, Elad Hazan, and Yoram Singer.
\newblock Adaptive subgradient methods for online learning and stochastic
  optimization.
\newblock {\em Journal of Machine Learning Research}, 12:2121--2159, 2011.

\bibitem{fiscko2023towards}
Carmel Fiscko, Aayushya Agarwal, Yihan Ruan, Soummya Kar, Larry Pileggi, and
  Bruno Sinopoli.
\newblock Towards hyperparameter-agnostic dnn training via dynamical system
  insights.
\newblock {\em arXiv preprint arXiv:2310.13901}, 2023.

\bibitem{franca2020}
Guilherme Fran{\c{c}}a, Jeremias Sulam, Daniel Robinson, and Ren{\'e} Vidal.
\newblock Conformal symplectic and relativistic optimization.
\newblock In {\em NeurIPS}, 2020.

\bibitem{franca2018admm}
Guilherme França, Daniel~P Robinson, and René Vidal.
\newblock Admm and accelerated admm as continuous dynamical systems.
\newblock In {\em International Conference on Machine Learning}, pages
  1554--1562. PMLR, 2018.

\bibitem{Franca2021}
Guilherme França, Daniel~P Robinson, and René Vidal.
\newblock Gradient flows and proximal splitting methods: A unified view on
  accelerated and stochastic optimization.
\newblock {\em Physical Review E}, 103(5), 2021.

\bibitem{ghosh2023implicit}
Avrajit Ghosh, He~Lyu, Xitong Zhang, and Rongrong Wang.
\newblock Implicit regularization in heavy-ball momentum accelerated stochastic
  gradient descent.
\newblock {\em ICLR}, 2023.

\bibitem{init2winit2021github}
Justin~M. Gilmer, George~E. Dahl, Zachary Nado, Priya Kasimbeg, and Sourabh
  Medapati.
\newblock {init2winit}: a jax codebase for initialization, optimization, and
  tuning research.
\newblock {\em github}, 2023.

\bibitem{gupta2018shampoo}
Vineet Gupta, Tomer Koren, and Yoram Singer.
\newblock Shampoo: Preconditioned stochastic tensor optimization.
\newblock In {\em International Conference on Machine Learning}, pages
  1842--1850. PMLR, 2018.

\bibitem{hairer2006geometric}
Ernst Hairer, Marlis Hochbruck, Arieh Iserles, and Christian Lubich.
\newblock Geometric numerical integration.
\newblock {\em Oberwolfach Reports}, 3(1):805--882, 2006.

\bibitem{hairer1993solving1}
Ernst Hairer, Syvert~P N{\o}rsett, and Gerhard Wanner.
\newblock {\em Solving Ordinary Differential Equations I: Nonstiff Problems},
  volume~8 of {\em Springer Series in Computational Mathematics}.
\newblock Springer, 2 edition, 1993.

\bibitem{hairer_solving_ode_2}
Ernst Hairer and Gerhard Wanner.
\newblock {\em Solving Ordinary Differential Equations II: Stiff and
  Differential-Algebraic Problems}, volume~14 of {\em Springer Series in
  Computational Mathematics}.
\newblock Springer, Berlin, Heidelberg, 2 edition, 1996.

\bibitem{he2016resnet}
Kaiming He, Xiangyu Zhang, Shaoqing Ren, and Jian Sun.
\newblock Deep residual learning for image recognition.
\newblock In {\em CVPR}, 2016.

\bibitem{resnets}
Kaiming He, Xiangyu Zhang, Shaoqing Ren, and Jian Sun.
\newblock Deep residual learning for image recognition.
\newblock In {\em CVPR}, 2016.

\bibitem{optax2020github}
Matteo Hessel, David Budden, Fabio Viola, Mihaela Rosca, Eren Sezener, and Tom
  Hennigan.
\newblock Optax: composable gradient transformation and optimisation, in jax!,
  2020.

\bibitem{kasimbeg2025accelerating}
Priya Kasimbeg, Frank Schneider, Runa Eschenhagen, Juhan Bae,
  Chandramouli~Shama Sastry, Mark Saroufim, BOYUAN FENG, Less Wright, Edward~Z.
  Yang, Zachary Nado, Sourabh Medapati, Philipp Hennig, Michael Rabbat, and
  George~E. Dahl.
\newblock Accelerating neural network training: An analysis of the algoperf
  competition.
\newblock In {\em ICLR}, 2025.

\bibitem{kingma2015adam}
Diederik~P. Kingma and Jimmy Ba.
\newblock Adam: A method for stochastic optimization.
\newblock In {\em ICLR}, 2015.

\bibitem{kovachki2021continuous}
Nikola~B Kovachki and Andrew~M Stuart.
\newblock Continuous time analysis of momentum methods.
\newblock {\em Journal of Machine Learning Research}, 22(17):1--40, 2021.

\bibitem{cifar10}
Alex Krizhevsky.
\newblock Learning multiple layers of features from tiny images.
\newblock {\em https://www.cs.toronto.edu/~kriz/learning-features-2009-TR.pdf},
  2009.

\bibitem{Lee_SmoothManifolds}
John~M. Lee.
\newblock {\em Introduction to Smooth Manifolds}, volume 218 of {\em Graduate
  Texts in Mathematics}.
\newblock Springer, 2nd edition, 2012.

\bibitem{visualization}
Hao Li, Zheng Xu, Gavin Taylor, Christoph Studer, and Tom Goldstein.
\newblock Visualizing the loss landscape of neural nets.
\newblock In {\em NeurIPS}, 2018.

\bibitem{li2023implicit}
Ye~Li, Song-Can Chen, and Sheng-Jun Huang.
\newblock Implicit stochastic gradient descent for training physics-informed
  neural networks.
\newblock {\em arXiv preprint arXiv:2303.01767}, 2023.

\bibitem{liu2020admin}
Liyuan Liu, Xiaodong Liu, Jianfeng Gao, Weizhu Chen, and Jiawei Han.
\newblock Understanding the difficulty of training transformers.
\newblock In {\em Proceedings of the 2020 Conference on Empirical Methods in
  Natural Language Processing (EMNLP 2020)}, April 2020.

\bibitem{loshchilov2017decoupled}
Ilya Loshchilov and Frank Hutter.
\newblock Decoupled weight decay regularization.
\newblock {\em arXiv preprint arXiv:1711.05101}, 2017.

\bibitem{loshchilov2019decoupled}
Ilya Loshchilov and Frank Hutter.
\newblock Decoupled weight decay regularization.
\newblock In {\em ICLR}, 2019.

\bibitem{lucas2019aggregated}
James Lucas, Shengyang Sun, Richard Zemel, and Roger Grosse.
\newblock Aggregated momentum: Stability through passive damping.
\newblock In {\em International Conference on Learning Representations}, 2019.

\bibitem{martens2015optimizing}
James Martens and Roger Grosse.
\newblock Optimizing neural networks with kronecker-factored approximate
  curvature.
\newblock In {\em International conference on machine learning}, pages
  2408--2417. PMLR, 2015.

\bibitem{mccandlish2018empirical}
Sam McCandlish, Jared Kaplan, Dario Amodei, and OpenAI~Dota Team.
\newblock An empirical model of large-batch training.
\newblock {\em arXiv preprint arXiv:1812.06162}, 2018.

\bibitem{muehlebach2019dynamical}
Michael Muehlebach and Michael~I Jordan.
\newblock A dynamical systems perspective on nesterov acceleration.
\newblock In {\em International Conference on Machine Learning}, pages
  4656--4662. PMLR, 2019.

\bibitem{orvieto2022dynamics}
Antonia Orvieto, Simon Lacoste-Julien, and Nicolas Loizou.
\newblock Dynamics of sgd with stochastic polyak stepsizes: Truly adaptive
  variants and convergence to exact solution.
\newblock In {\em NeurIPS}, 2022.

\bibitem{orvieto2024adaptive}
Antonio Orvieto and Lin Xiao.
\newblock {An Adaptive Stochastic Gradient Method with Non-negative
  Gauss-Newton Stepsizes}.
\newblock {\em arXiv preprint arXiv:2407.04358}, 2024.

\bibitem{Polyak1964}
Boris~T. Polyak.
\newblock Some methods of speeding up the convergence of iteration methods.
\newblock {\em USSR Computational Mathematics and Mathematical Physics},
  4(5):1--17, 1964.

\bibitem{pouplin2023curvature}
Alison Pouplin, Hrittik Roy, Sidak~Pal Singh, and Georgios Arvanitidis.
\newblock On the curvature of the loss landscape.
\newblock {\em arXiv preprint arXiv:2307.04719}, 2023.

\bibitem{qin2020training}
Chongli Qin, Yan Wu, Jost~Tobias Springenberg, Andy Brock, Jeff Donahue,
  Timothy Lillicrap, and Pushmeet Kohli.
\newblock Training generative adversarial networks by solving ordinary
  differential equations.
\newblock In {\em Advances in Neural Information Processing Systems}, 2020.

\bibitem{rosca2021discretization}
Mihaela Rosca, Yan Wu, Benoit Dherin, and David~G.T. Barrett.
\newblock Discretization drift in two-player games.
\newblock In {\em ICML}, 2021.

\bibitem{rosca2023on}
Mihaela Rosca, Yan Wu, Chongli Qin, and Benoit Dherin.
\newblock On a continuous time model of gradient descent dynamics and
  instability in deep learning.
\newblock In {\em TMLR}, 2023.

\bibitem{Rumelhart1986}
David~E. Rumelhart, Geoffrey~E. Hinton, and Ronald~J. Williams.
\newblock Learning representations by back-propagating errors.
\newblock {\em Nature}, 323:533--536, 1986.

\bibitem{shi2019acceleration}
Bin Shi, Simon~S Du, Weijie~J Su, and Michael~I Jordan.
\newblock Acceleration via symplectic discretization of high-resolution
  differential equations.
\newblock In {\em Advances in Neural Information Processing Systems}, pages
  5744--5752, 2019.

\bibitem{shi2023distributed}
Hao-Jun~Michael Shi, Tsung-Hsien Lee, Shintaro Iwasaki, Jose Gallego-Posada,
  Zhijing Li, Kaushik Rangadurai, Dheevatsa Mudigere, and Michael Rabbat.
\newblock A distributed data-parallel pytorch implementation of the distributed
  shampoo optimizer for training neural networks at-scale.
\newblock {\em arXiv preprint arXiv:2309.06497}, 2023.

\bibitem{su2024improving}
Dan Su, Qihai Jiang, Enhong Liu, and Mei Liu.
\newblock Improving optimizers by runge-kutta method: A case study of sgd and
  adam.
\newblock In {\em 2024 12th International Conference on Intelligent Control and
  Information Processing (ICICIP)}, 2024.

\bibitem{Su2016}
Weijie Su, Stephen Boyd, and Emmanuel Cand\`{e}s.
\newblock A differential equation for modeling nesterov's accelerated gradient
  method: theory and insights.
\newblock {\em Journal of Machine Learning Research}, 17:1--43, 2016.

\bibitem{Sutskever2013}
Ilya Sutskever, James Martens, George Dahl, and Geoffrey Hinton.
\newblock On the importance of initialization and momentum in deep learning.
\newblock In {\em Proceedings of the 30th International Conference on Machine
  Learning}, volume~28, pages 1139--1147. PMLR, 2013.

\bibitem{takase2025spike}
Sho Takase, Shun Kiyono, Sosuke Kobayashi, and Jun Suzuki.
\newblock Spike no more: Stabilizing the pre-training of large language models,
  2025.

\bibitem{Tieleman2012}
Tijmen Tieleman and Geoffrey Hinton.
\newblock Rmsprop: Divide the gradient by a running average of its recent
  magnitude.
\newblock COURSERA Neural Networks for Machine Learning, 2012.

\bibitem{wang2021understanding}
Sifan Wang, Yujun Teng, and Paris Perdikaris.
\newblock Understanding and mitigating gradient flow pathologies in
  physics-informed neural networks.
\newblock {\em SIAM Journal on Scientific Computing}, 43(5):A3055--A3081, 2021.

\bibitem{wortsman2024smallscale}
Mitchell Wortsman, Peter~J Liu, Lechao Xiao, Katie~E Everett, Alexander~A
  Alemi, Ben Adlam, John~D Co-Reyes, Izzeddin Gur, Abhishek Kumar, Roman Novak,
  Jeffrey Pennington, Jascha Sohl-Dickstein, Kelvin Xu, Jaehoon Lee, Justin
  Gilmer, and Simon Kornblith.
\newblock Small-scale proxies for large-scale transformer training
  instabilities.
\newblock In {\em ICLR}, 2024.

\bibitem{xie2025structured}
Shuo Xie, Tianhao Wang, Sashank Reddi, Sanjiv Kumar, and Zhiyuan Li.
\newblock Structured preconditioners in adaptive optimization: A unified
  analysis.
\newblock {\em arXiv preprint arXiv:2503.10537}, 2025.

\bibitem{zagoruyko2016wide}
Sergey Zagoruyko and Nikos Komodakis.
\newblock Wide residual networks.
\newblock {\em arXiv preprint arXiv:1605.07146}, 2016.

\bibitem{Zeiler2012}
Matthew~D. Zeiler.
\newblock Adadelta: An adaptive learning rate method.
\newblock {\em arXiv preprint arXiv:1212.5701}, 2012.

\bibitem{zhang2018direct}
Jingzhao Zhang, Aryan Mokhtari, Suvrit Sra, and Ali Jadbabaie.
\newblock Direct runge-kutta discretization achieves acceleration.
\newblock In {\em Advances in Neural Information Processing Systems}, 2018.

\bibitem{zhang2019lookahead}
Michael~R Zhang, James Lucas, Geoffrey Hinton, and Jimmy Ba.
\newblock Lookahead optimizer: k steps forward, 1 step back.
\newblock In {\em Advances in neural information processing systems}, pages
  9591--9601, 2019.

\end{thebibliography}

\clearpage
\newpage

\appendix

\section{Runge-Kutta updates}
\label{appendix:runge_kutta_updates}

This section presents the formulas for Runge-Kutta (RK) updates up to the fourth order (RK2, RK3, and RK4). These methods provide approximations to the exact solutions of a differential equation $\dot \theta = f(\theta)$, where $f$ is the differential-equation vector-field. For explicit second-order methods, it's possible to parameterize the entire family of these methods. We include this well-known computation \cite{hairer1993solving1} to illustrate the general approach for finding RK coefficients. In the main paper, we benchmarked RK4, which is the classical $4^{\text{th}}$ order method, and it has an error of size $\mathcal O(h^5)$.

Recall that a RK method is given by a matrix $A=(a_{ij})$ and a vector $b = (b_i)$. The corresponding RK update is of the form
\begin{equation} \label{equation:rk_method_appendix}
\theta' = \theta + h\sum_{i=1}^s b_i f(\theta_i),
\end{equation}
where the points $\theta_i$ are the solutions of the following system of equations
\begin{equation}\label{equation:rk_neighoring_points_appendix}
\theta_i = \theta + h \sum_{j=1}^s a_{ij}f(\theta_j).
\end{equation} 
The idea is to find $A$ and $b$ such that the Taylor expansion of the RK update in Eq. \eqref{equation:rk_method_appendix} coincides up to order $k+1$ with the Taylor series of the exact solution of the ODE, yielding a RK method of order $k$. The first step is to compute the Taylor expansion of the numerical method. This is easy for the first two orders, but rapidly becomes extremely complex and requires the introduction of sophisticated mathematics (see \cite{hairer1993solving1}). Let us compute this Taylor expansion here up to order 2, then we will illustrate the full process in the case of explicit second order methods:

\begin{eqnarray}\label{equation:RK_taylor_expansion}
\theta' 
& = & \theta +  h\sum_{i=1}^s b_i f(\theta_i) \\
& = & \theta +  h\sum_{i=1}^s b_i f\Big(\theta + h \sum_{j=1}^s a_{ij}f(\theta_j)\Big)\\
& = & \theta + h\sum_{i=1}^s b_i \Bigg(
f(\theta) + f'(\theta)
    \Big(
    h \sum_{j=1}^s a_{ij}f(\theta + \mathcal O(h))  \Big)
    + \mathcal O(h^2)
\Bigg)\\
& = & \theta +  h \left(\sum_{i=1}^s b_i\right) f(\theta) +
h^2 \left(\sum_{ij=1}^s b_i a_{ij}\right)f'(\theta)f(\theta) + \mathcal O(h^3).
\label{equation:RK_taylor_expansion_last}
\end{eqnarray}

Therefore, the condition
\begin{equation}\label{equation:rk_conditions_appendix}
\sum_{i=1}^s b_i = 1 
\end{equation}
ensures that the method is at least a first order method, while the condition
\begin{equation}
\sum_{ij=1}^s b_ia_{ij} = \frac 12
\end{equation}
ensures the method is at least of order 2. 

The next paragraph applies this argument to explicit second order methods.

 \paragraph{General second-order methods.} For explicit second-order methods, we can compute all their possible weight, while also illustrating the general principle how to compute the RK weights $a=(a_{ij})$ and $b=(b_i)$. These methods have only two gradient evaluations (i.e. they have two stages), and a step of the method will have an error of order $\mathcal O(h^3)$. Since the method has two stages the vector $b$ has two components only, and the matrix $a$ is a lower-diagonal $2\times 2$ matrix. This means that we can parameterize the probability vector $b$ with a single number $\alpha \in [0, 1]$, namely $b_1 = 1 - \alpha$ and $b_2 = \alpha$. On the other hand, the lower-triangular $2 \times 2$ matrix $a$ has a single non-zero entry $a_{21} = \beta$. Therefore the two points where we will evaluate the gradients will be $\theta_1 = \theta$ and $\theta_2 = \theta + h \beta f(\theta)$. This yields the gradient update rule:
\begin{equation}\label{equation:second_order_rk}
\theta' = \theta + h \Bigg( (1 - \alpha) f(\theta) + \alpha f(\theta + h\beta f(\theta))\Bigg),
\end{equation}
which parameterizes all the two-stages explicit RK methods.
Now, the goal is to find $\alpha$ and $\beta$ so that the error between one step of the method above and the exact solutions of $\dot \theta = f(\theta)$ is of order $\mathcal O(h^3)$. To figure that out, we can compute the Taylor expansion of the exact solution up to order $\mathcal O(h^3)$ as well as that of the update rule above and adjust the $\alpha$ and $\beta$ for the two series to coincide up to order $\mathcal O(h^3)$. Let us do that. First of all, the Taylor's expansion of the solution $\theta(h)$ of the ODE $\dot \theta(t) = f(\theta(t))$ starting at $\theta(0) = \theta$ is given at the first orders by 
\begin{eqnarray}
    \theta(h) 
    & = & \theta + \dot \theta(0) h + \frac 1{2!} \ddot \theta(0) h^2 + \mathcal O(h^3) \ \\
    & = & \theta + hf(\theta) + \frac {h^2}2 f'(\theta)f(\theta) + \mathcal O(h^3),
\end{eqnarray}
where we obtained the second derivative of $\theta(t)$ by differentiating the ODE $\dot \theta = f(\theta)$ on both sides w.r.t.~to time: $\ddot \theta = f'(\theta) \dot \theta = f'(\theta) f(\theta)$.
On the other hand for the two-stage RK update rule \eqref{equation:second_order_rk} we have the following expansion in the learning rate:
$$
\theta' = \theta +h f(\theta) + h^2 \alpha \beta f'(\theta) f(\theta) + \mathcal O(h^3).
$$
So, as long as $\alpha \beta = \frac 12$, that is $\beta = \frac 1{2\alpha}$, the method will be of second order, yielding all possible two-stage second-order RK methods. They are parameterized by $\alpha \in (0,1]$ with updates: 
\begin{equation}\label{equation:order2_RK_methods}
\theta' = \theta +  h \Bigg(
(1-\alpha) f(\theta) + \alpha f\big(\theta + \frac{h}{2\alpha}f(\theta)\big)
\Bigg),
\end{equation}
 A popular choice is to take $\alpha = 1/2$ and $\beta = 1$, leading to the \emph{improved Euler method} (also know as the Heun method) which is an order of magnitude more precise than the Euler method (the error is $\mathcal O(h^3)$ compared to $\mathcal O(h^2)$ for the Euler method) but it necessitates two gradient evaluations.
 
Below, we list popular $2^\text{nd}$, $3^\text{rd}$, and $4^{\text{th}}$ order RK methods for reference.

\paragraph{RK2: Heun $2^\text{nd}$ order method.} As we saw, there all the second order RK methods are parameterized by a single coefficient $\alpha \in (0, 1]$. A popular choice is $\alpha = 1/2$ leading to the Heun method or improved Euler method:
\begin{eqnarray}
\theta_1 & = & \theta \\
\theta_2 & = & \theta +  h  f(\theta_1) \\
\theta' & = & \theta + \frac h2 \Big( f(\theta_1) +  f(\theta_2) \Big)
\end{eqnarray}
  
\paragraph{RK3: Kutta $3^\text{rd}$ order method.} There are many more ways to produce 3rd order methods than second order ones. The computation idea is the same but it becomes much more intricate. For instance, here is a popular 3rd order method:
\begin{eqnarray}\label{equation:RK3}
\theta_1 & = & \theta \\
\theta_2 & = & \theta + \frac h2  f(\theta_1) \\
\theta_3 & = & \theta -  h  f(\theta_1) +  2h  f(\theta_2) \\
\theta' & = & \theta + \frac h6 \Big(f(\theta_1) + 4  f(\theta_2) + f(\theta_3) \Big)
\end{eqnarray}

\paragraph{RK4: Classical $4^{\text{th}}$ order method.} Similarly, there are many more ways to produce $4^{\text{th}}$ order methods than second order ones. Here is a very popular $4^{\text{th}}$ order method:

\begin{eqnarray}\label{equation:RK4}
\theta_1 & = & \theta \\
\theta_2 & = & \theta + \frac h2  f(\theta_1) \\
\theta_3 & = & \theta + \frac h2  f(\theta_2) \\
\theta_4 & = & \theta +  h  f(\theta_3) \\
\theta' & = & \theta + h \frac 16 \Big(f(\theta_1) + 2  f(\theta_2) + 2f(\theta_3)  + f(\theta_4) \Big)
\end{eqnarray}

\clearpage
\newpage

\section{Geometry of preconditioning}
\label{appendix:geometry_of_preconditioning}

We discuss here the mathematics and intuition underlying the preconditioning of RK updates, giving some motivation for our particular preconditioning choice in Section \ref{section:preconditioning}.

\paragraph{Vectors and co-vectors.} Let $V$ be a vector space. A vector in $v$ is usually represented in coordinates by a column vector. On the other hand, elements of the dual space $V^*$, which are linear maps from $V$ to $\mathbb R$, are represented by row vectors and called \emph{co-vectors}. The application of a co-vector $\mu\in V^*$ to a vector $v$ yields the number $\mu \cdot v$, where $\cdot$ is the matrix product of the row vector $\mu$ by the column vector $v$.

In machine learning, the loss gradient $\nabla L(\theta)$ is typically a co-vector (i.e., a row vector). This is the linear map that tells us how much the loss changes if we move from the parameter $\theta$ in the direction $v$, namely $L(\theta + v) \simeq L(\theta) + \nabla L(\theta) \cdot v$.

On the other hand, the velocity $\dot \theta(t)$ of a curve $\theta(t)$ in parameter space, is a \emph{vector} (represented by a column vector) since by definition of the derivative we have
\begin{equation}
\dot \theta(t) = \operatorname{lim}_{\substack{\epsilon \rightarrow 0}} \frac{\theta(t + \epsilon) - \theta(t)}{\epsilon} .
\end{equation}
This is important for ODE's because on both sides of a differential equation 
\begin{equation}
    \dot \theta(t) = f(\theta(t))
\end{equation}
we need to be vectors (i.e. column vectors). 

\paragraph{The gradient flow ODE.}
The standard gradient flow equation $\dot \theta(t) = \nabla L(\theta(t))$ suffers from the problem of equating a vector with a co-vector: the gradient flow written in this form is not \emph{invariant under change of coordinates}. This means that the gradient flow equation does not assume the same form after a change of coordinates because vectors and co-vectors transform differently. This issue has already been noticed recently in \cite{bernstein2024modular} for optimizers like gradient descent that add a vector (the parameter) to a co-vector (the update), creating a type mismatch. Luckily, if we have a Riemannian metric around; i.e., a positive symmetric matrix $G(\theta)$ defined for each point $\theta$ of the parameter space, we can transform a co-vector into vector via what is sometimes called the \emph{musical isomorphism} (see \cite{Lee_SmoothManifolds} for more details): $\mu \rightarrow G^{-1}\mu$. Therefore, the correct way to write the gradient flow ODE is with the help of an underlying metric tensor $G(\theta)$ on the parameter space given by
\begin{equation}
    \dot \theta(t) = G(\theta(t))^{-1} \nabla L(\theta(t)).
\end{equation}
In this form, the gradient flow ODE has the exact same form in every coordinate system. 
Therefore writing the gradient flow ODE in an invariant form coincides exactly with preconditioning the gradient field by the inverse metric tensor $A = G^{-1}$.

\paragraph{Natural metric on the loss surface.} 
There is a natural metric on the loss surface, which has been identified in \cite{barrett2021implicit} Appendix A.2 and studied in \cite{pouplin2023curvature}:
\begin{equation}
\mathcal{G}(\theta) = \mathbb{I} + g(\theta)g(\theta)^T, \quad \text{where}\quad g(\theta) = \nabla L(\theta).
\end{equation}
Intuitively, this metric says that if you move by a small amount $v$ from a point $\theta$ in parameter space, then you actually move a distance approximated by $\|v\|_{\mathcal G} = \sqrt{v^T \mathcal G(\theta)v}$ on the loss surface. It thus seems natural to precondition the gradient field by this metric tensor; i.e. $\mathcal G(\theta)^{-1}\nabla L(\theta)$ since it would scale back large gradients. Luckily, this is easy thanks to the next lemma:
\begin{lemma}
The gradient $g(\theta)$ is an eigenvector of the matrix $\mathcal G(\theta) = \mathbb{I} + g(\theta) g(\theta)^T$ with eigenvalue $(1 + \|g(\theta)\|^2)$. This means that $G^\alpha(\theta) g(\theta) = (1 + \|g(\theta\|^2)^\alpha g(\theta)$. 
\end{lemma}
\begin{proof}
First of all, we have that $\mathcal G(\theta)g(\theta) = g(\theta) + g(\theta) g(\theta)^T g(\theta) =(1+ \|g(\theta)\|^2) g(\theta)$, which proves the first part. The second part comes from the general fact that if a matrix $G$ is invertible with eigenvalue $\lambda$ for eigenvector $v$ then $v$ is also an eigenvector of $\mathcal G^{\alpha}$ but with eigenvalue $\lambda^\alpha$
\end{proof}

\paragraph{Adagrad-like conditioning.}
Unfortunately, the resulting preconditioner $\mathcal G^{-1}(\theta) \nabla L(\theta)$ was not helpful. It seems that integrating past information from the trajectory (as done in Adam or AdaGrad for instance) is important here. We decided to accumulate past information in a similar way as done in AdaGrad \cite{Duchi2011} through a diagonal approximation $\operatorname{diag} G_n$ where $G_n = \sum_{l=1}^n g(\theta_l)g(\theta_l)^T$. This yields the preconditioner 
\begin{equation}
A'_n(\theta) = (1 + \operatorname{diag} G_n)^{-\nicefrac{1}{2}}.
\end{equation}
We chose the square root of the inverse above because it worked better in practice, and because it is also the choice made in Adam and AdaGrad. Note that our preconditioner differs only by the constant factor 1 which in AdaGrad is an hyperparameter $\epsilon$ that needs to be tune:
\begin{equation}
A_n(\theta) = (\epsilon + \operatorname{diag} G_n)^{-\nicefrac{1}{2}},
\end{equation}
We tried both preconditioners in Appendix \ref{appendix:bridging_the_gap_preconditioning} Fig.~\ref{figure:conditioners_mnist}
 and Fig.~\ref{figure:conditioners_fashion_mnist}, and ours worked noticeably better than AdaGrad.

\clearpage
\newpage

\section{Experiment Details}
\label{appendix:experiment_details}

We used two different TPU configurations:
\begin{itemize}
    \item \texttt{SMALL-TPU}: a single Google Cloud TPU V2
    \item \texttt{COLAB-TPU}: a single Google Cloud TPU V3 accessed via a Google Colab.
    \item \texttt{LARGE-TPU}: a single Google Cloud TPU V4
\end{itemize}

\subsection{Experiment Details for Figure
 \ref{figure:small_batch_vs_large_batch}}
\label{appendix:figure_1}
 
\paragraph{Left column (small batch regime):} RK4 has an advantage over Adam in the small batch regime (batch size is 16) for both MNIST (top) and Fashion MNIST (bottom). In both cases, we trained a 3 hidden layer MLP with 500 neurons in each layer trained for 10000 steps on 5 seeds with metrics reported at each iteration. 
{\bf Top (MNIST):} Adam learning rate was tuned to 0.001 while the decay parameters were set to Optax defaults \cite{optax2020github}, and RK4 learning rate was tuned to 0.003. The learning curves for this experiment are in Fig. \ref{figure:learning_curves_small_batch_mnist}.
{\bf Bottom (Fashion MNIST):} Adam learning rate was tuned to 0.001 while the decay parameters were set to Optax defaults \cite{optax2020github}, and RK4 learning rate was tuned to 0.003. The learning curves for this experiment are in Fig. \ref{figure:learning_curves_fashion_mnist_small_batch}. The 5 random seeds for each experiment sweep took
roughly 10 minutes of training on the \texttt{COLAB-TPU} configuration for each workload.

\paragraph{Right column (large batch regime):} RK4 suffers from a generalization gap over Adam. In the large batch regime (batchsize=60000) for both MNIST (top) and Fashion MNIST (bottom). In both cases, we trained a 3 hidden layer MLP with 500 neurons in each layer trained for 10000 steps on 5 seeds with metrics reported at each iteration. 
{\bf Top (MNIST):}  Adam learning rate was tuned to 0.001 while the decay parameters were set to Optax defaults \cite{optax2020github}, and RK4 learning rate was tuned to 0.004. The learning curves for this experiment are in Fig. \ref{figure:learning_curve_large_batch_mnist}.
{\bf Bottom (Fashion MNIST):} Adam learning rate was tuned to 0.001 while the decay parameters were set to Optax defaults \cite{optax2020github}, and RK4 learning rate was tuned to 0.003. The learning curves for this experiment are in Fig. \ref{figure:learning_curves_fashion_mnist_large_batch}. The 5 random seeds for each experiment took
roughly 4 hours of training on a \texttt{COLAB-TPU} configuration for each workload.

\begin{figure}[h] 
\centering 
\includegraphics[width=1.0\textwidth]{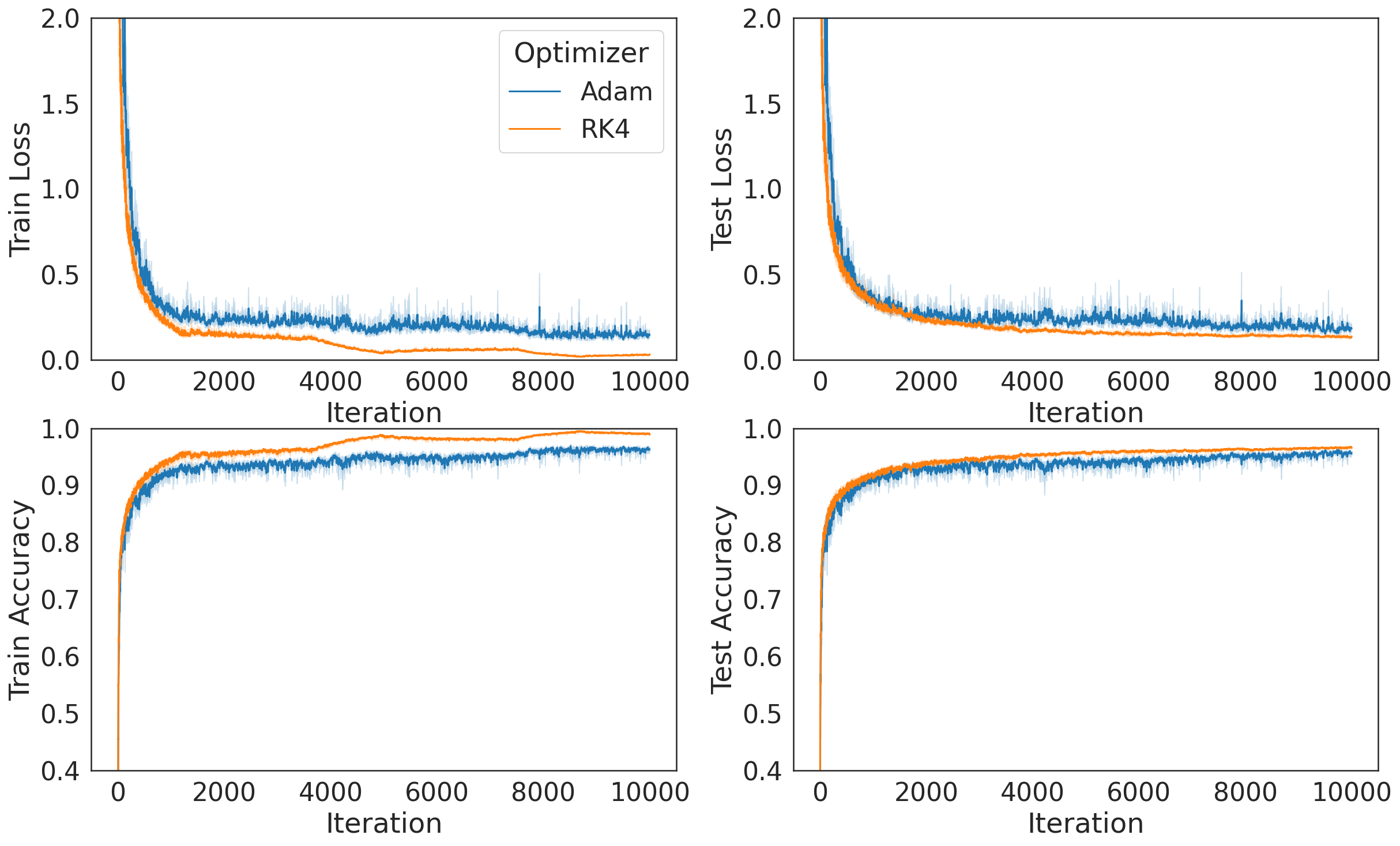}
\caption{
RK4 is competitive with Adam on MNIST with a 3 hidden layer MLP of 500 neurons each when trained with small batches of size 16 for 10000 steps. Adam's learning rate was tuned to 0.001 while the decay parameters were set to Optax defaults \cite{optax2020github}, and RK4 learning rate was tuned to 0.003; 5 seeds.}
\label{figure:learning_curves_small_batch_mnist}
\end{figure}

\begin{figure}[h] 
\centering 
\includegraphics[width=1.0\textwidth]{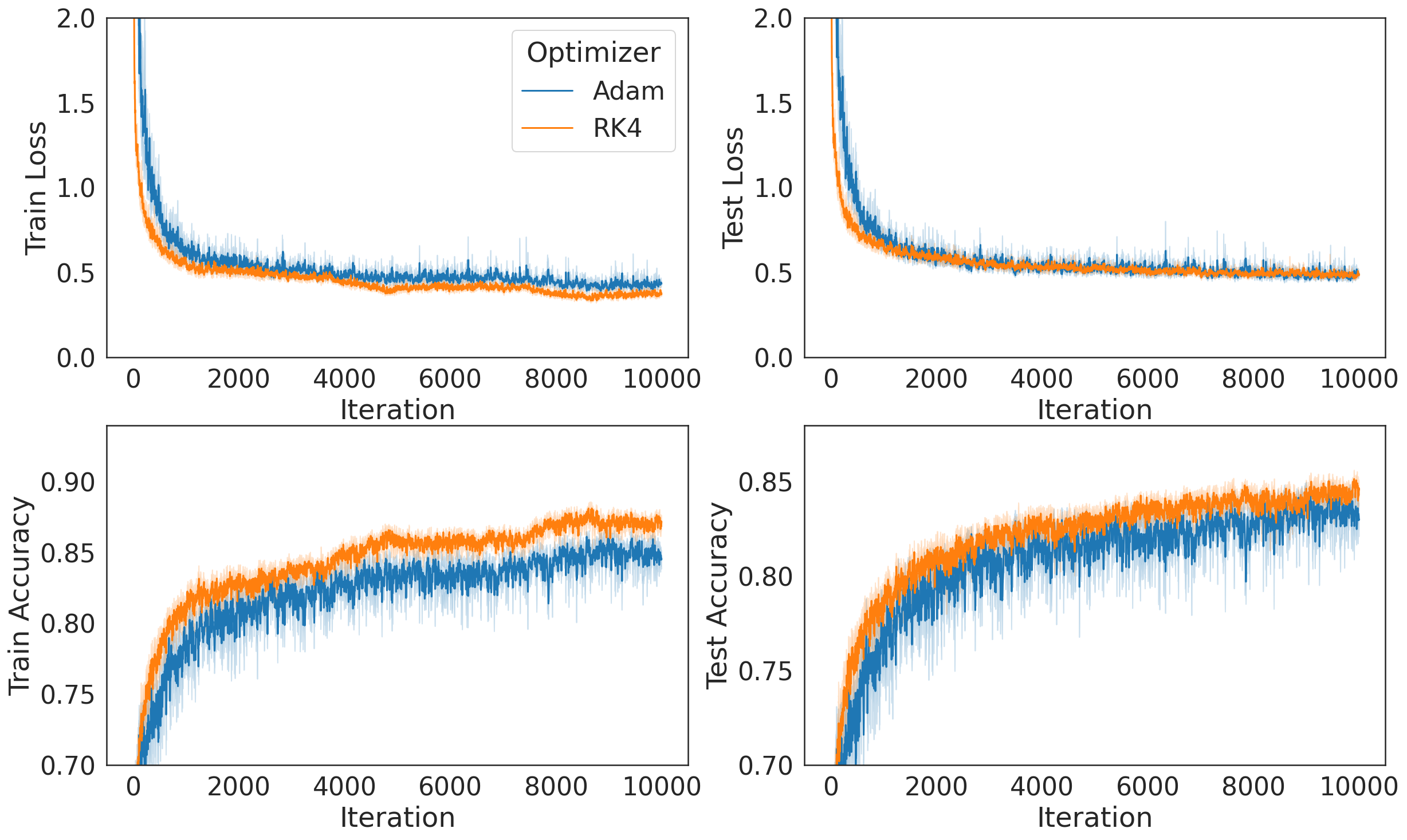}
\caption{
RK4 is competitive with Adam on Fashion MNIST with a 3 hidden layer MLP of 500 neurons each when trained with small batches of size 16 for 10000 steps. Adam's learning rate was tuned to 0.001 while the decay parameters were set to Optax defaults \cite{optax2020github}, and RK4 learning rate was tuned to 0.003; 5 seeds.}
\label{figure:learning_curves_fashion_mnist_small_batch}
\end{figure}

\begin{figure}[h] 
\centering 
\includegraphics[width=1.0\textwidth]{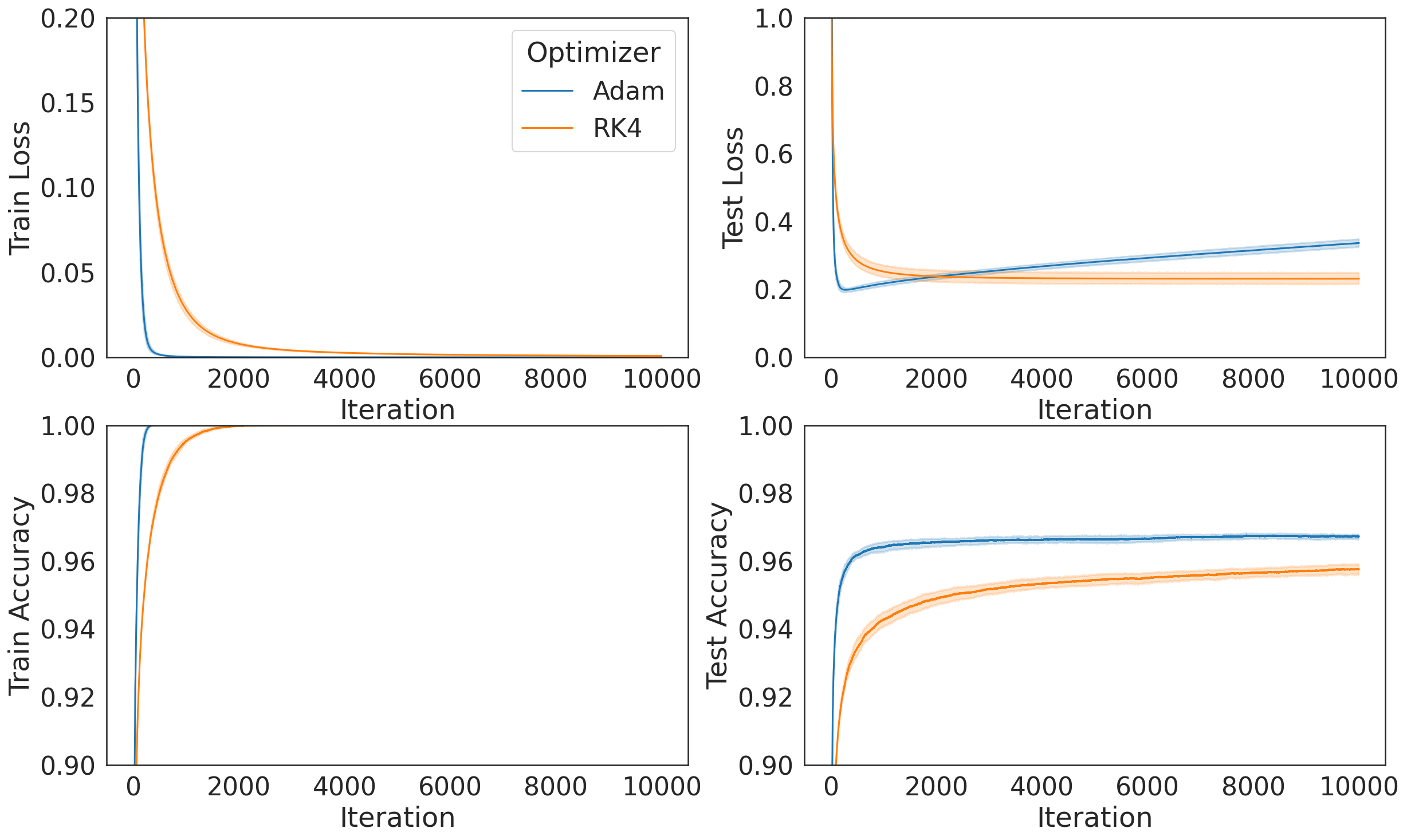}
\caption{
RK4 suffers from a generalization gap w.r.t. Adam on MNIST with a 3 hidden layer MLP of 500 neurons each when trained with full batches of size 60000 for 10000 steps. Adam's learning rate was tuned to 0.001 while the decay parameters were set to Optax defaults \cite{optax2020github}, and RK4 learning rate was tuned to 0.004; 5 seeds.}
\label{figure:learning_curve_large_batch_mnist}
\end{figure}

\begin{figure}[h] 
\centering 
\includegraphics[width=1.0\textwidth]{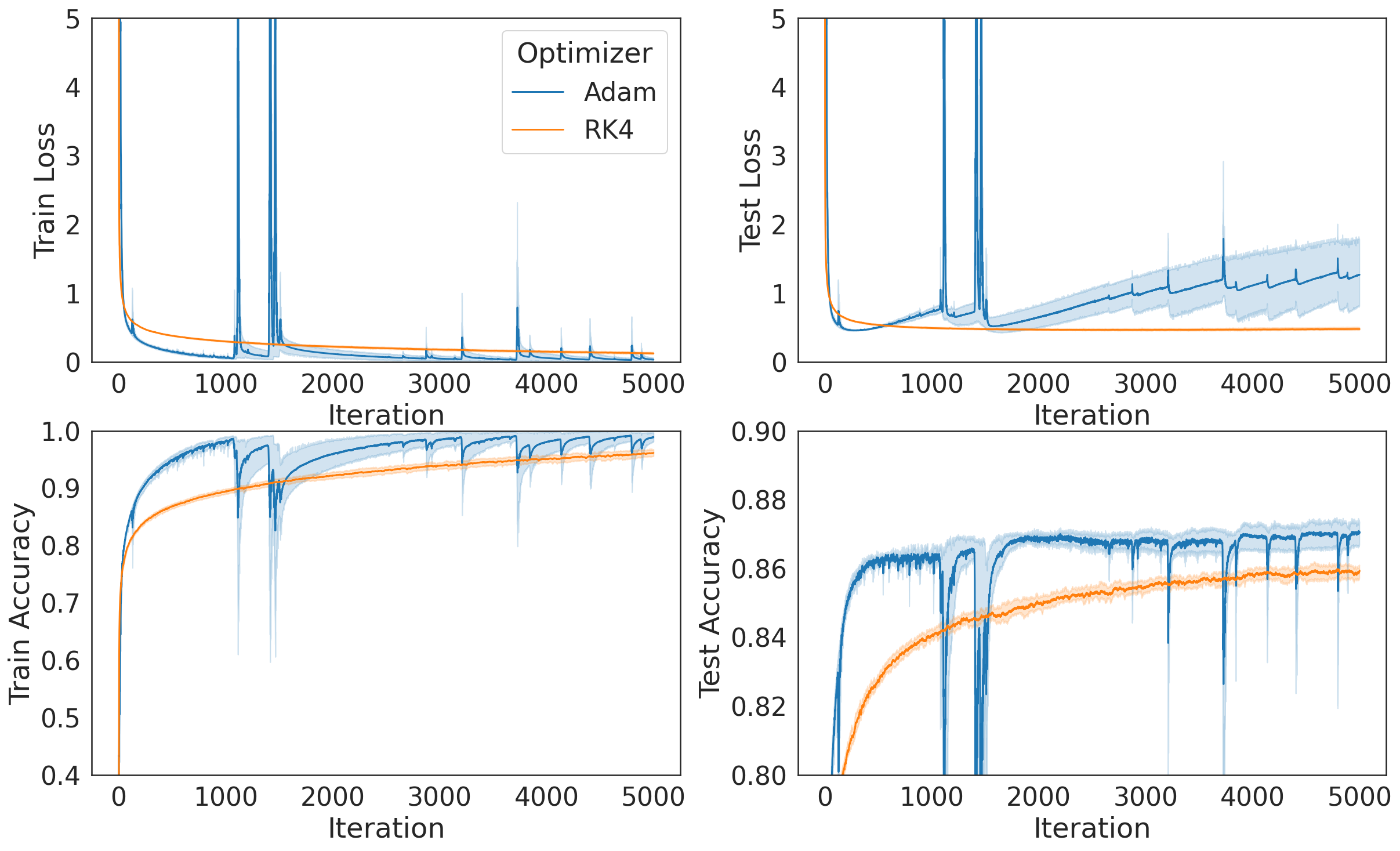}
\caption{
RK4 suffers from a generalization gap w.r.t. Adam on Fashion MNIST with a 3 hidden layer MLP of 500 neurons each when trained with full batches of size 60000 for 10000 steps. Adam's learning rate was tuned to 0.001 while the decay parameters were set to Optax defaults \cite{optax2020github}, and RK4 learning rate was tuned to 0.003; 5 seeds.}
\label{figure:learning_curves_fashion_mnist_large_batch}
\end{figure}

\begin{figure}[h] 
\centering 
\includegraphics[width=0.49\textwidth]{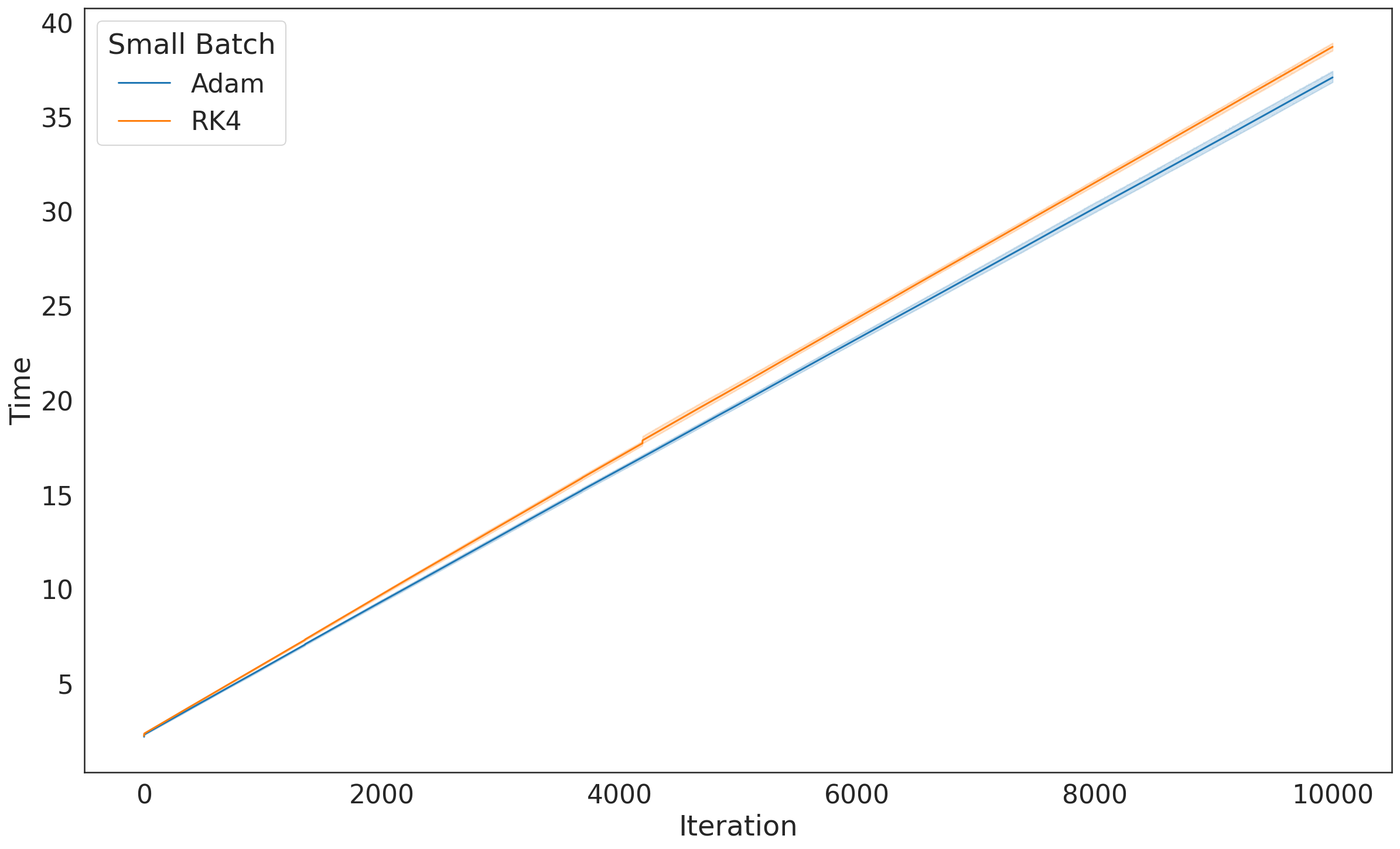}
\includegraphics[width=0.49\textwidth]{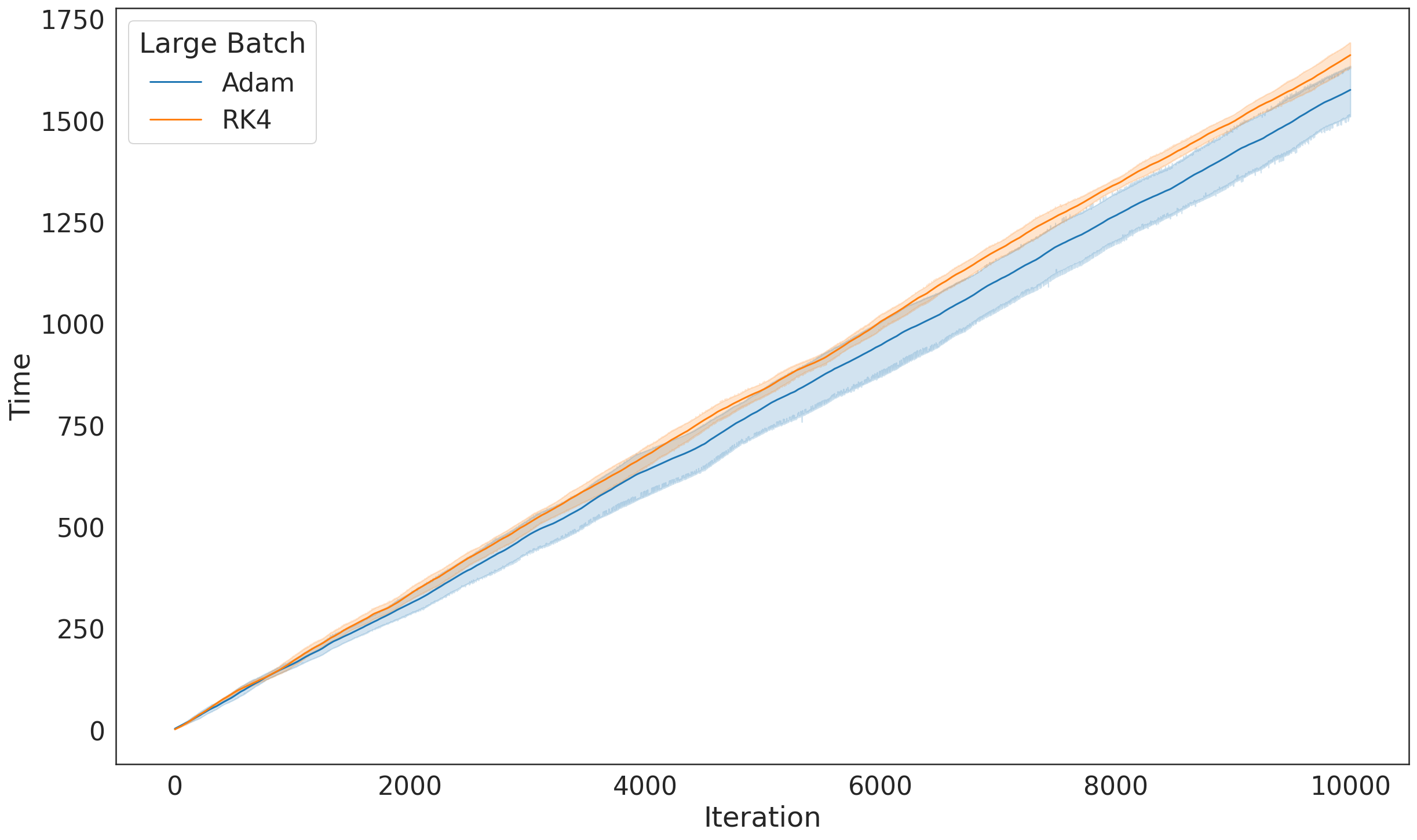}
\includegraphics[width=0.49\textwidth]{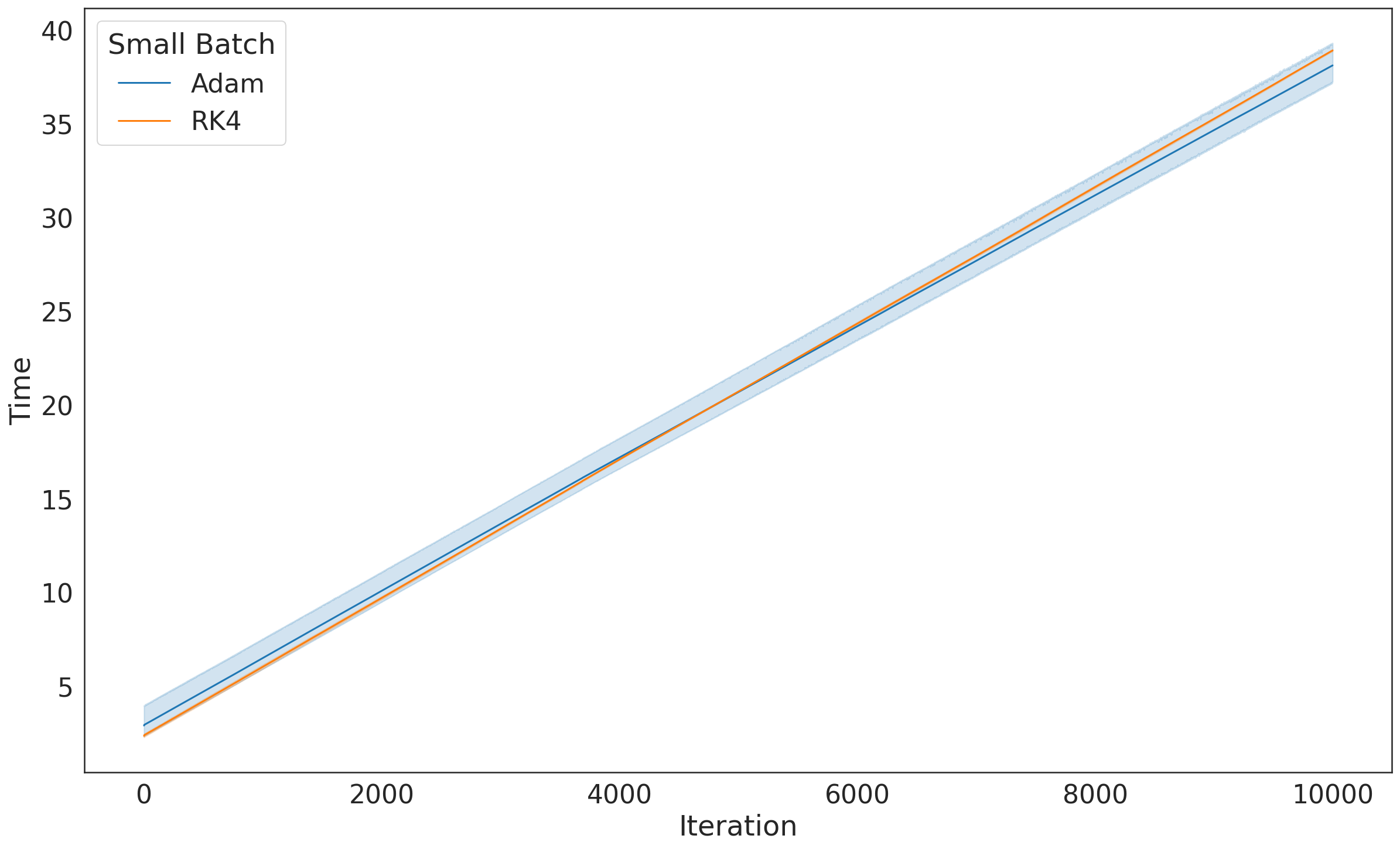}
\includegraphics[width=0.49\textwidth]{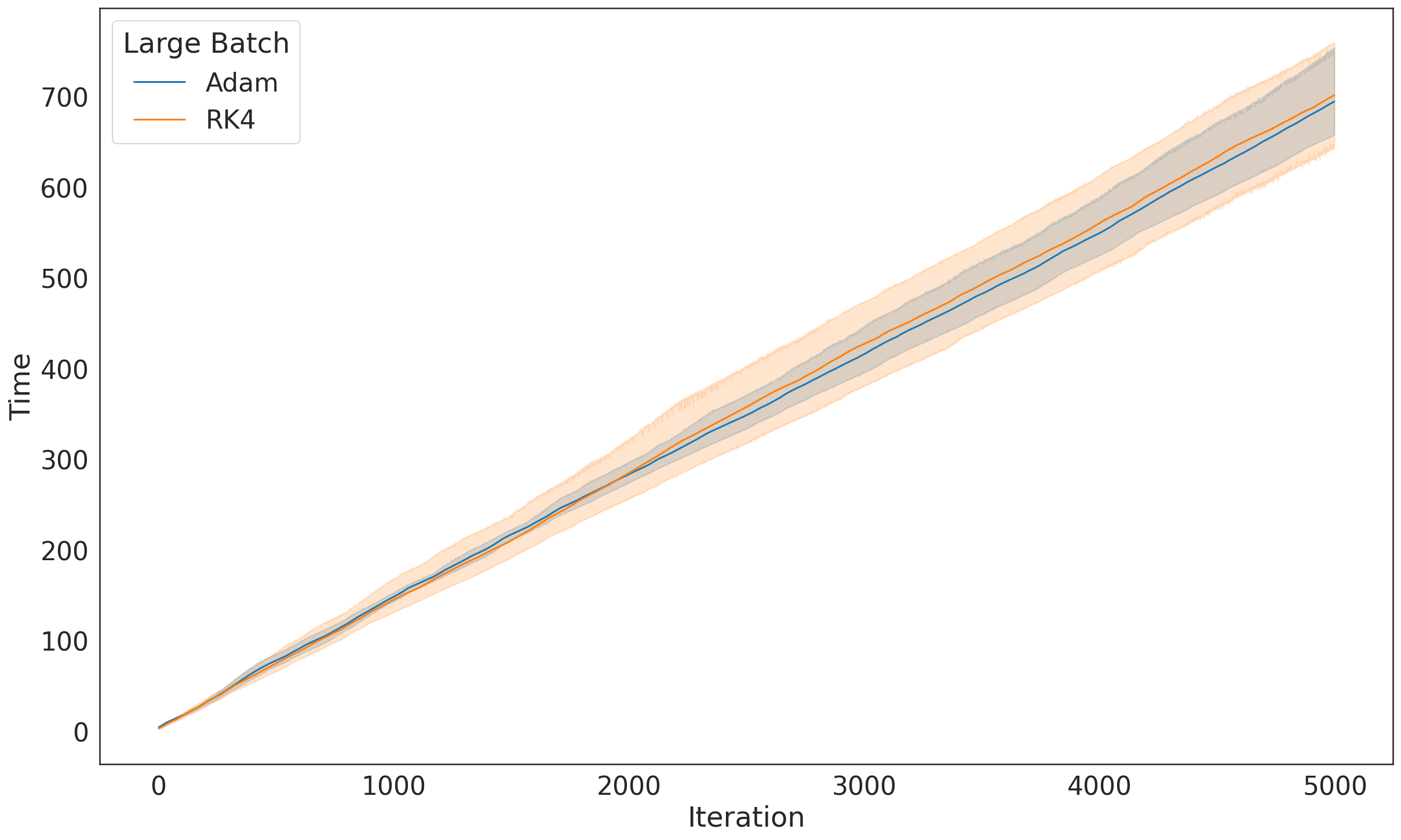}
\caption{
Wall-clock time versus steps between Adam and vanilla RK4 for MNIST ({\bf top row}) and Fashion MNIST ({\bf bottom row}) in the small batch regime ({\bf left column}) and the large batch regime ({\bf right column}).
}
\label{figure:time_versus_step}
\end{figure}

\clearpage
\newpage

\subsection{Experiment Details for Fig.
 \ref{figure:bridging_the_gap_preconditioning}}
\label{appendix:bridging_the_gap_preconditioning}

The modified AdaGrad preconditioning for RK4 helps bridge the generalization gap with Adam on MNIST (left) and Fashion MNIST (right). For the MNIST dataset, the tuned learning rates were as follows: Adam (0.002), RK4 (0.004), RK4 + ADGR (0.032), and RK4 + Modified-ADGR (0.064). Similarly, for the Fashion MNIST dataset, the learning rates were: Adam (0.002), RK4 (0.003), RK4 + ADGR (0.016), and RK4 + Modified-ADGR (0.032). The model architecture consists of three fully connected layers, each with 500 neurons. Models are trained using the full training dataset in each iteration (step). Fig. ~\ref{figure:conditioners_mnist} and Fig. ~\ref{figure:conditioners_fashion_mnist} display the training curves for these experiments alongside the previously presented test curves for comparison. The random seeds took roughly 5h of training on the \texttt{COLAB-TPU} configuration for each workload.

\begin{figure}[h] 
\centering 
    \begin{minipage}{0.48\textwidth}
        \centering
        \includegraphics[width=1.0\textwidth]{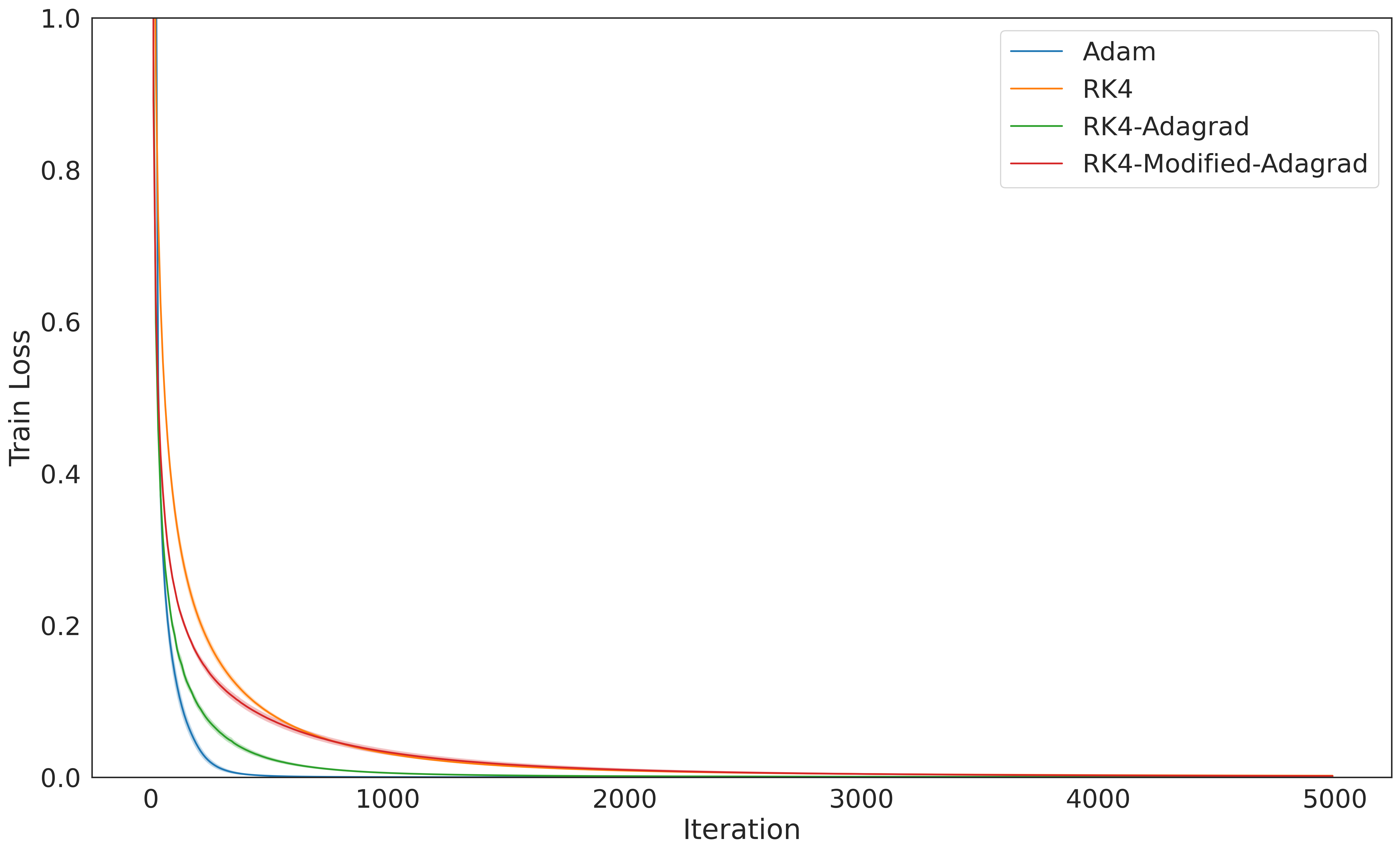}
    \end{minipage}
    \hfill 
    \begin{minipage}{0.48\textwidth}
        \centering
        \includegraphics[width=1.0\textwidth]{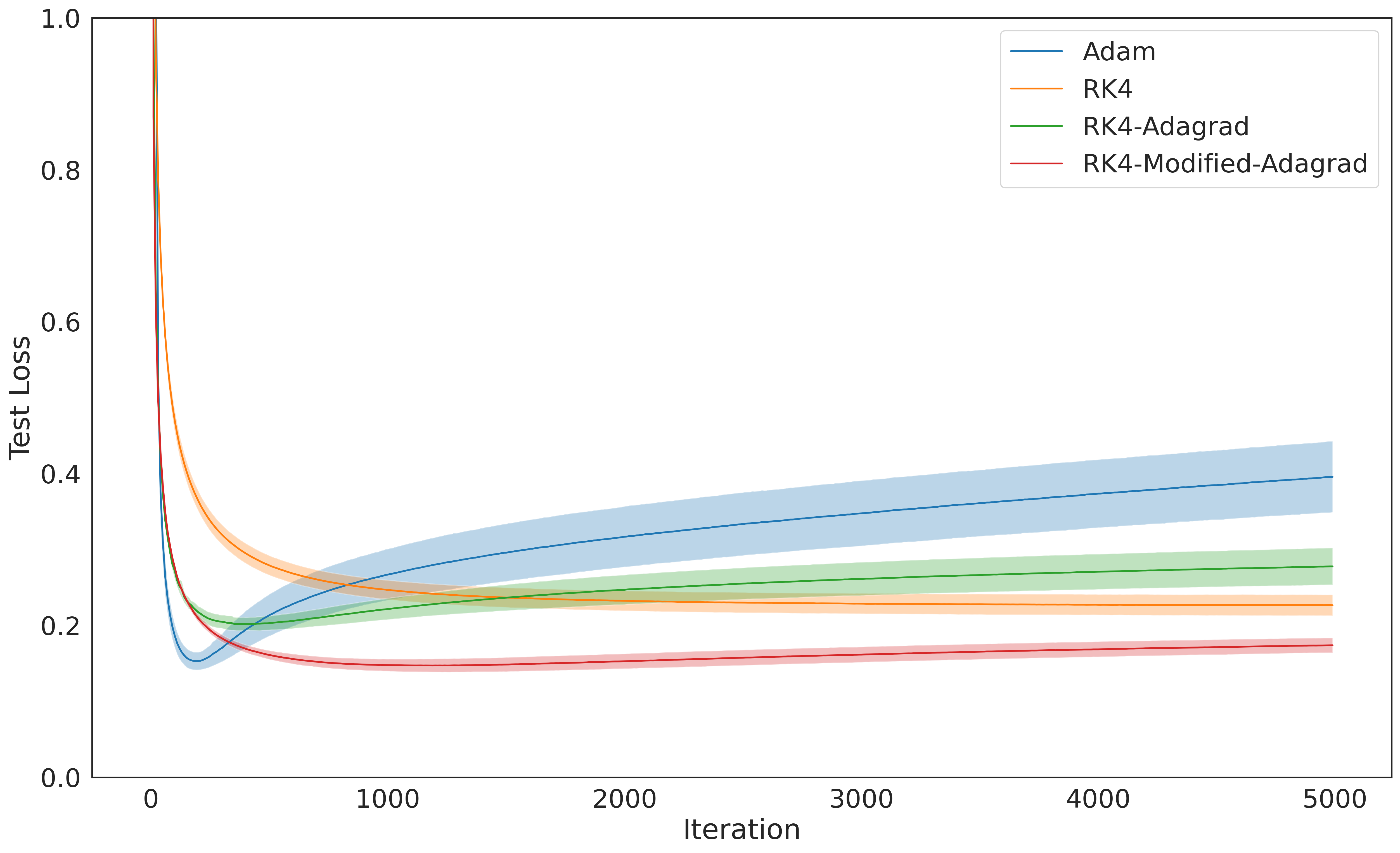}
    \end{minipage}
    \\
    \begin{minipage}{0.48\textwidth}
        \centering
        \includegraphics[width=1.0\textwidth]{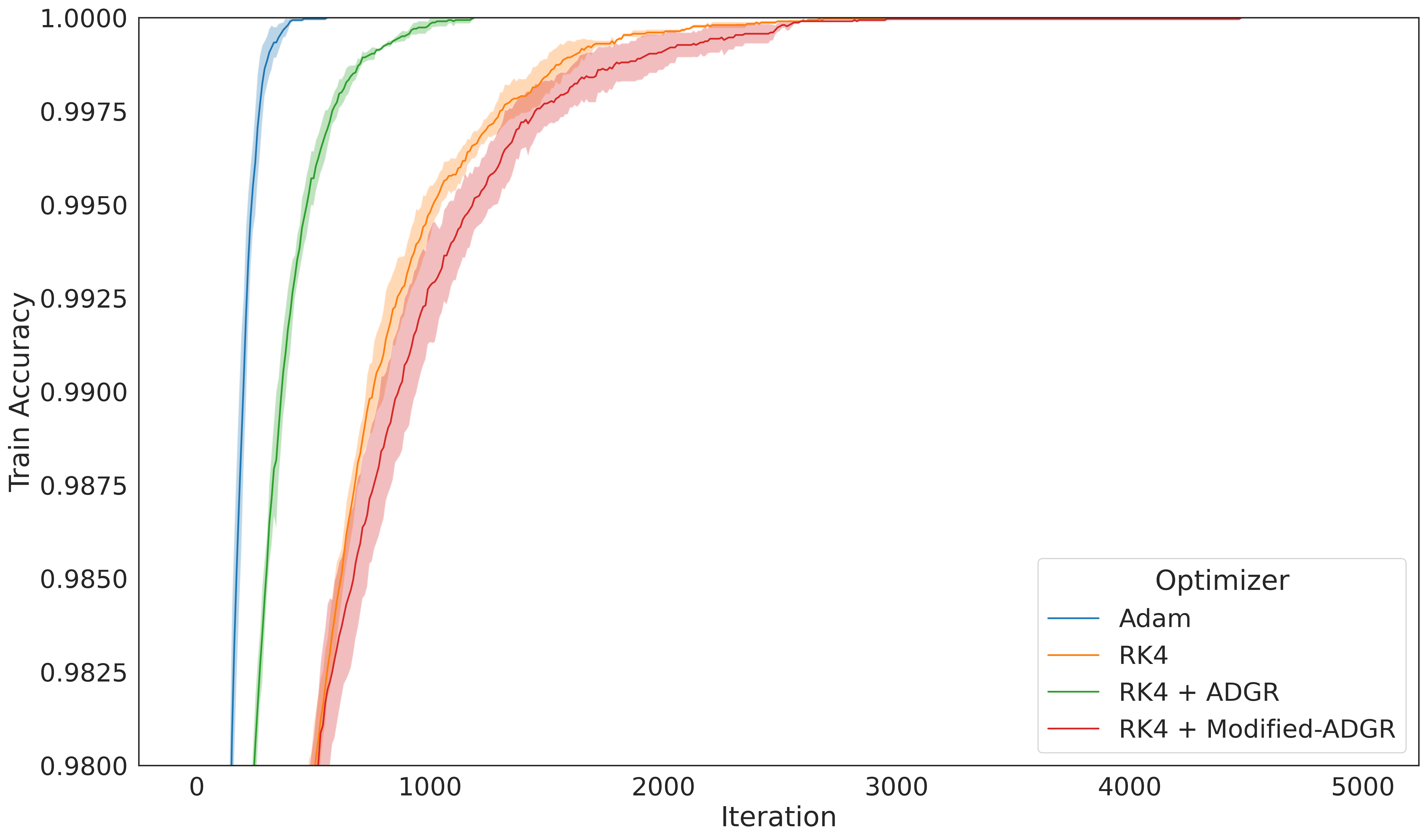}
    \end{minipage}
    \hfill 
    \begin{minipage}{0.48\textwidth}
        \centering
        \includegraphics[width=1.0\textwidth]{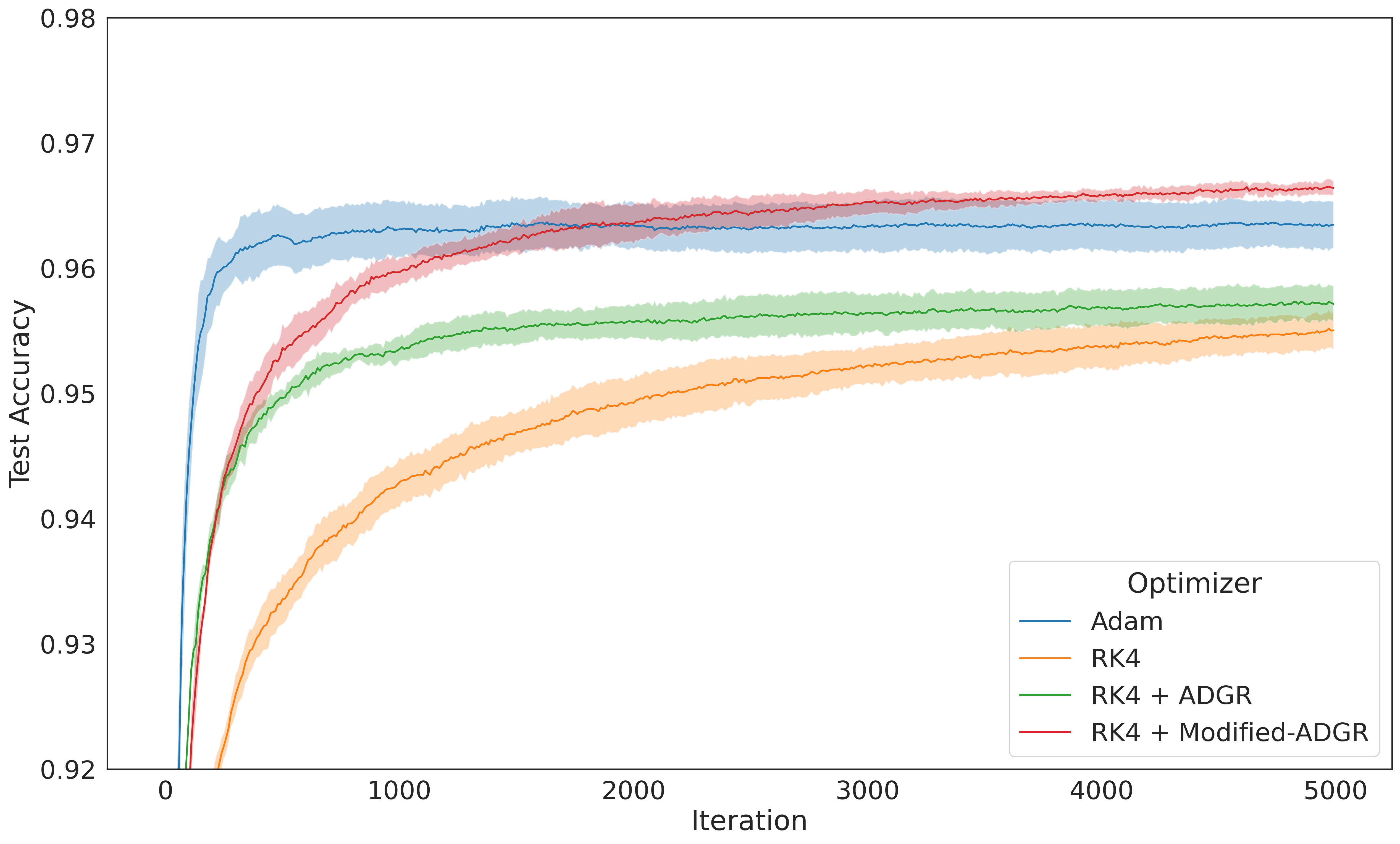}
    \end{minipage}
\caption{
RK4 with preconditioning competes effectively with Adam on MNIST when using a 3-layer MLP (500 neurons per layer) trained with full-batch gradient descent for 5,000 steps.}
\label{figure:conditioners_mnist}
\end{figure}

\begin{figure}[h] 
\centering 
    \begin{minipage}{0.48\textwidth}
        \centering
        \includegraphics[width=1.0\textwidth]{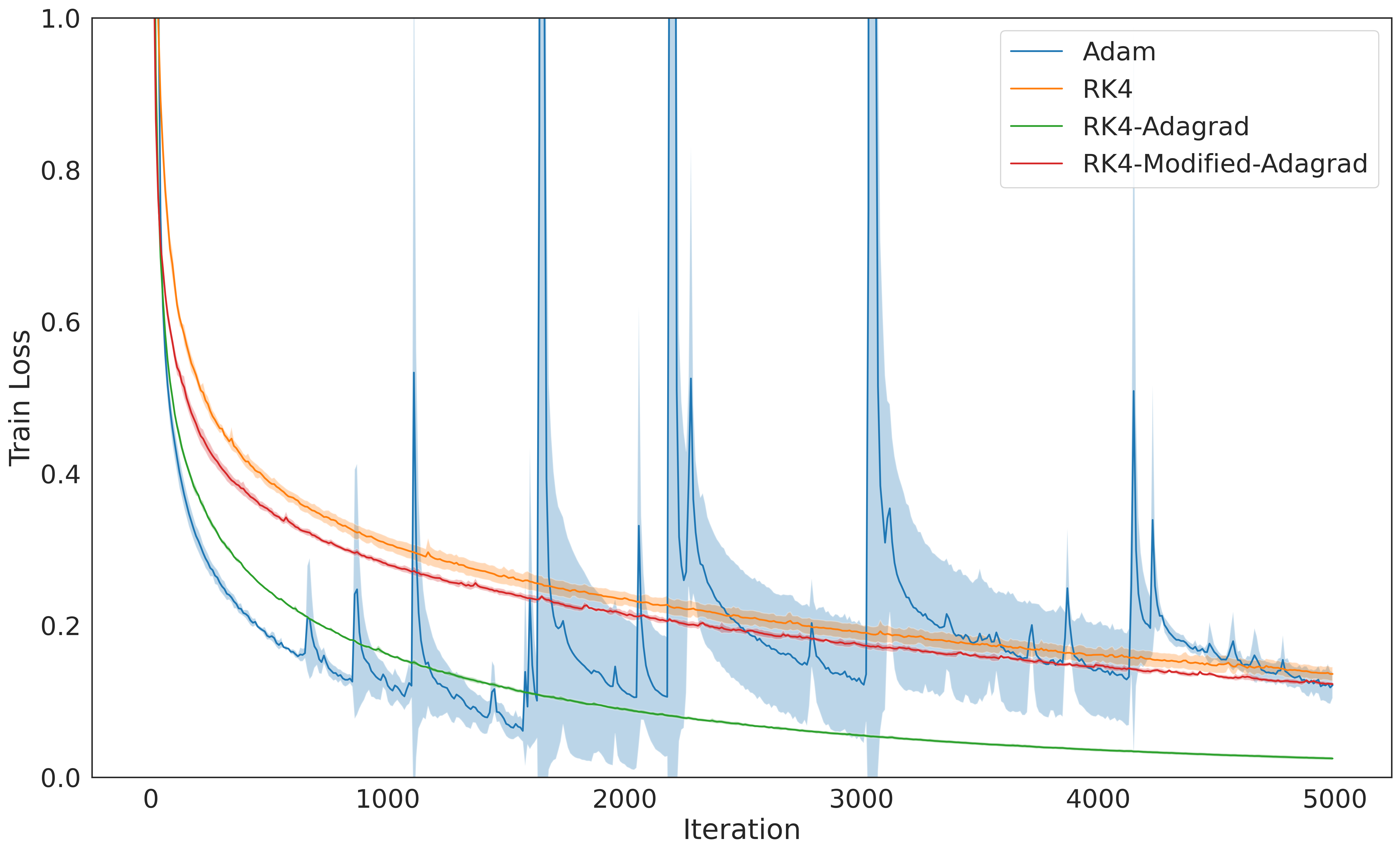}
    \end{minipage}
    \hfill 
    \begin{minipage}{0.48\textwidth}
        \centering
        \includegraphics[width=1.0\textwidth]{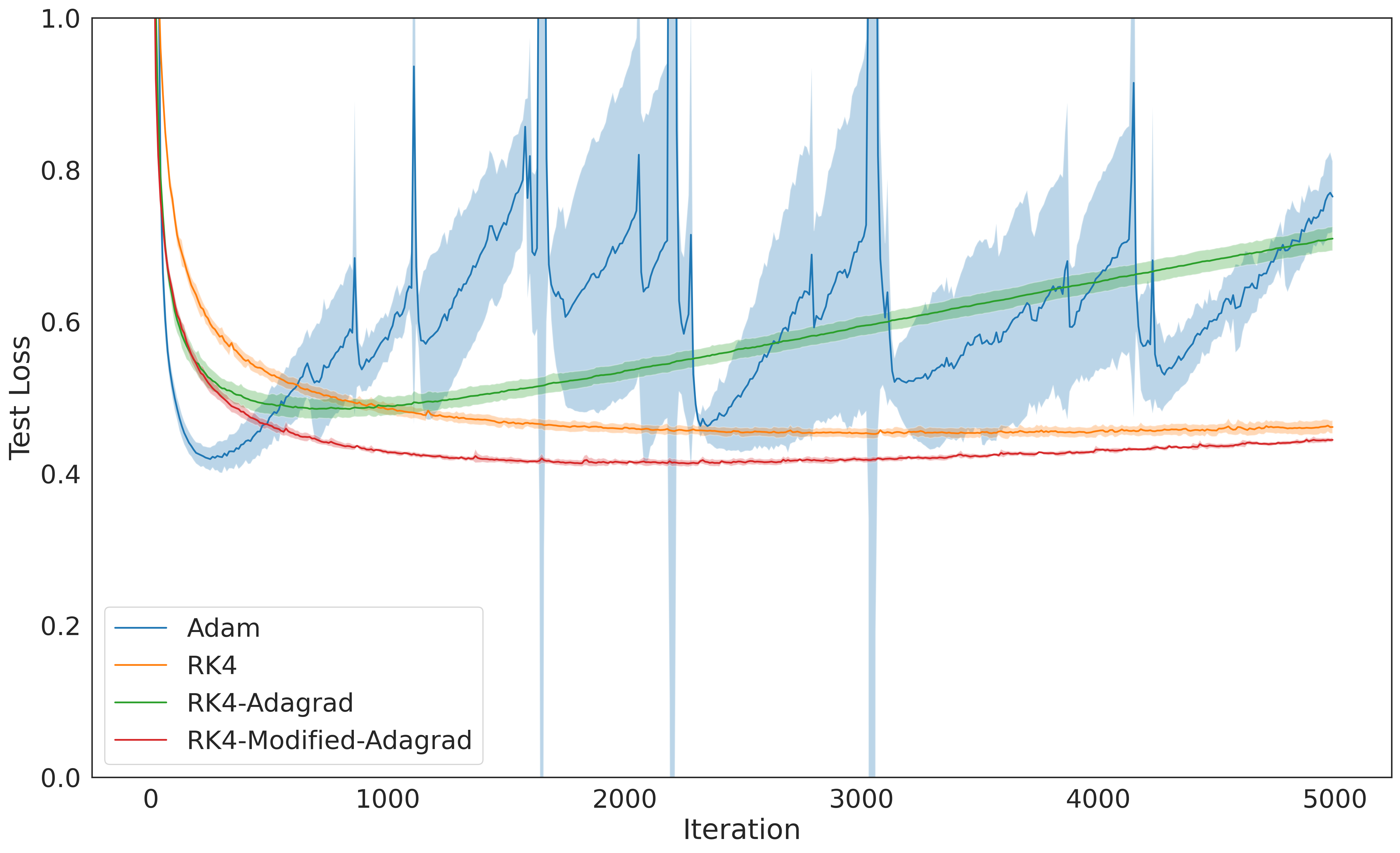}
    \end{minipage}
    \\
    \begin{minipage}{0.48\textwidth}
        \centering
        \includegraphics[width=1.0\textwidth]{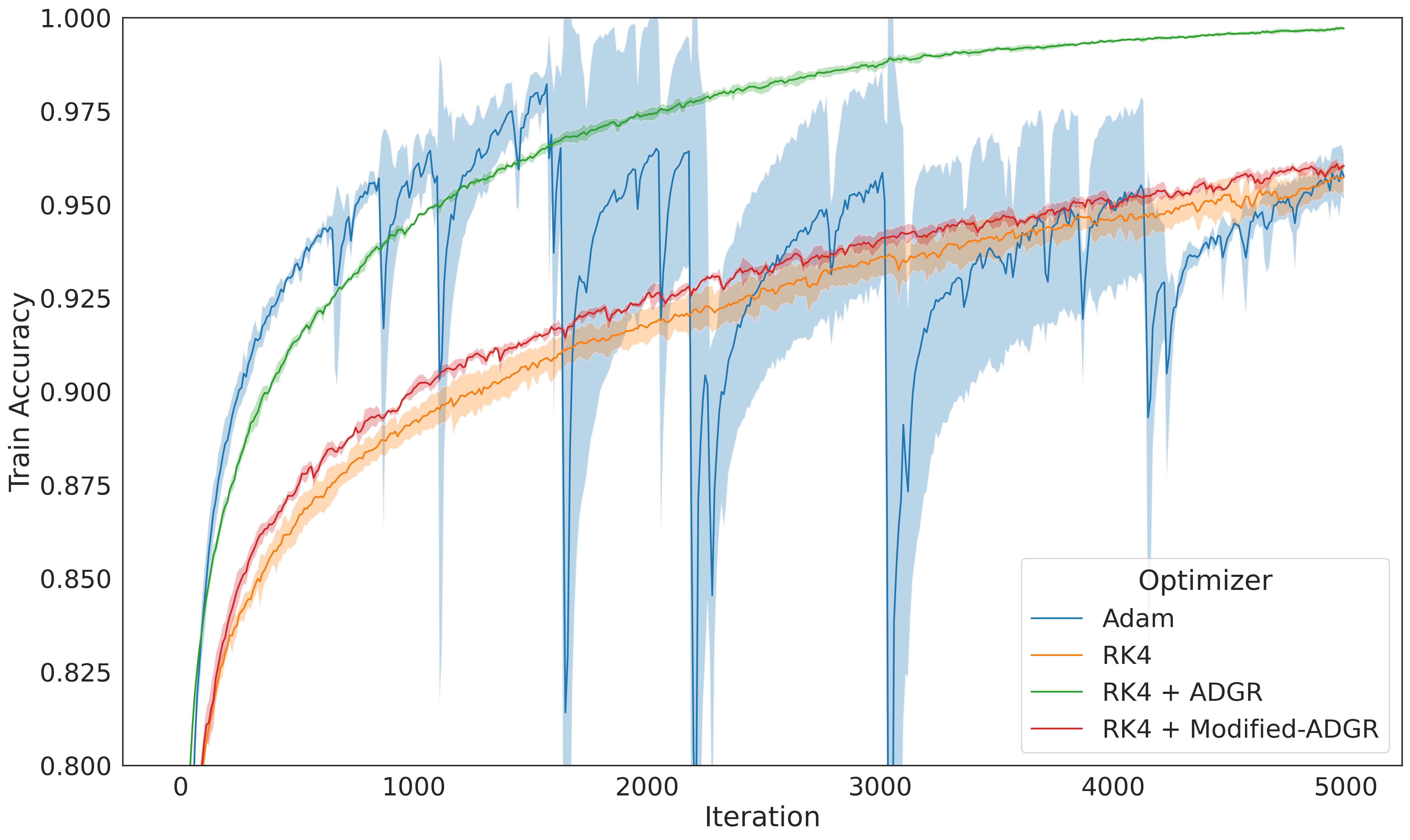}
    \end{minipage}
    \hfill 
    \begin{minipage}{0.48\textwidth}
        \centering
        \includegraphics[width=1.0\textwidth]{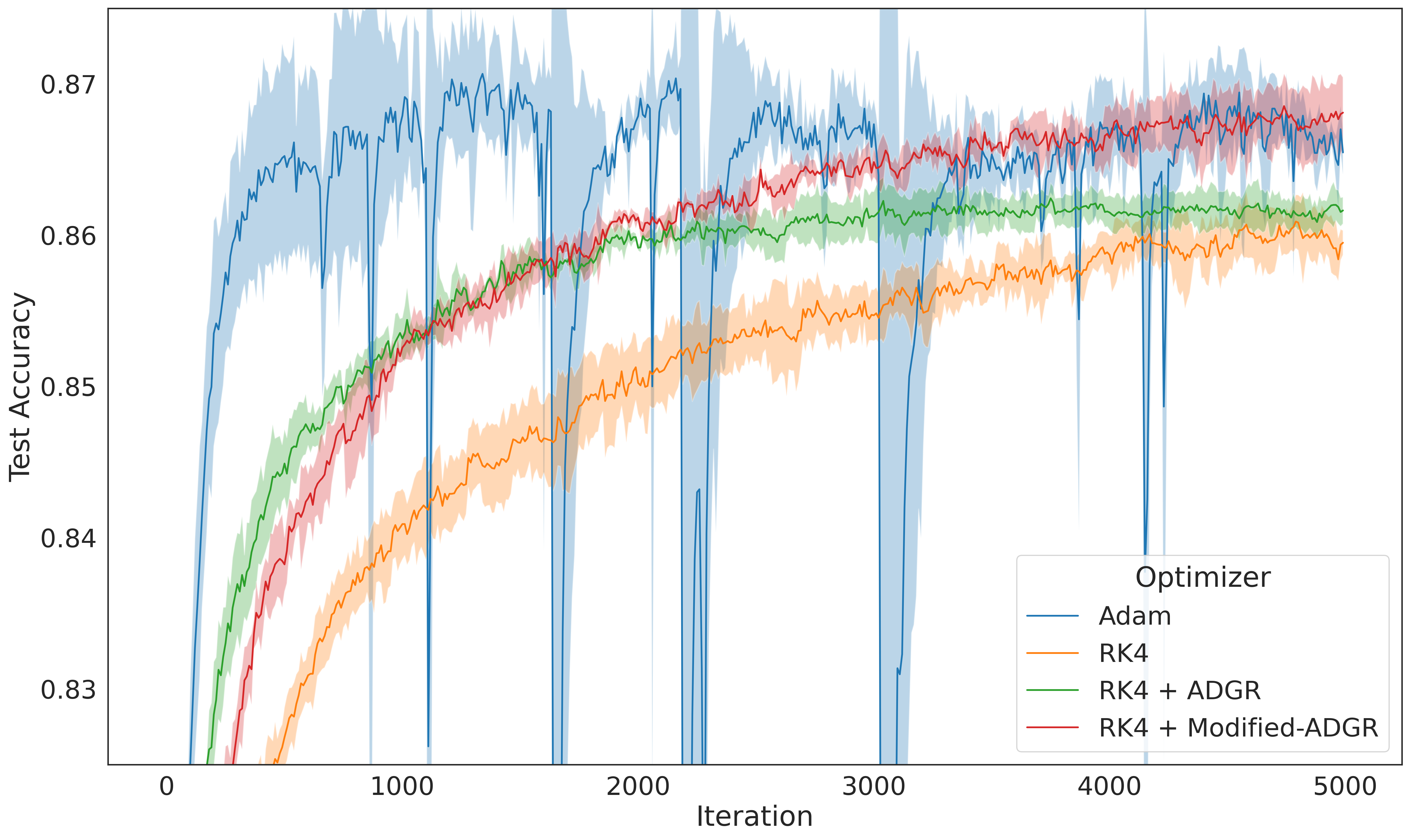}
    \end{minipage}
\caption{
RK4 with preconditioning competes effectively with Adam on Fashion MNIST when using a 3-layer MLP (500 neurons per layer) trained with full-batch gradient descent for 5,000 steps.}
\label{figure:conditioners_fashion_mnist}
\end{figure}

\clearpage
\newpage

\subsection{Experiment Details for Fig.
 \ref{figure:bridging_the_gap_adaptive_learning_rate}}
\label{appendix:bridging_the_gap_adaptive_learning_rate}

RK4 used with the DALR adaptive learning rate bridges the gap between and even surpasses Adam on MNIST (left) and Fashion MNIST (right). The experiment settings for Adam and RK4 for both workloads are the exact same as for the experiments in Fig. \ref{figure:small_batch_vs_large_batch} right column (large batch setting). The DALR adaptive learning rate has been tuned with values $p=0.8$ and $c=4.0$ for MNIST (left), and $p=0.8$ and $c=1.0$ for Fashion MNIST (right). The learning curves for MNIST are in Fig. \ref{figure:dalr_mnist} and those for Fashion MNIST are in Fig. \ref{figure:dalr_fashion_mnist}.
The 5 random seeds took roughly 6h of training on the \texttt{COLAB-TPU} configuration for each workload.

\begin{figure}[h] 
\centering 
    \includegraphics[width=1.0\textwidth]{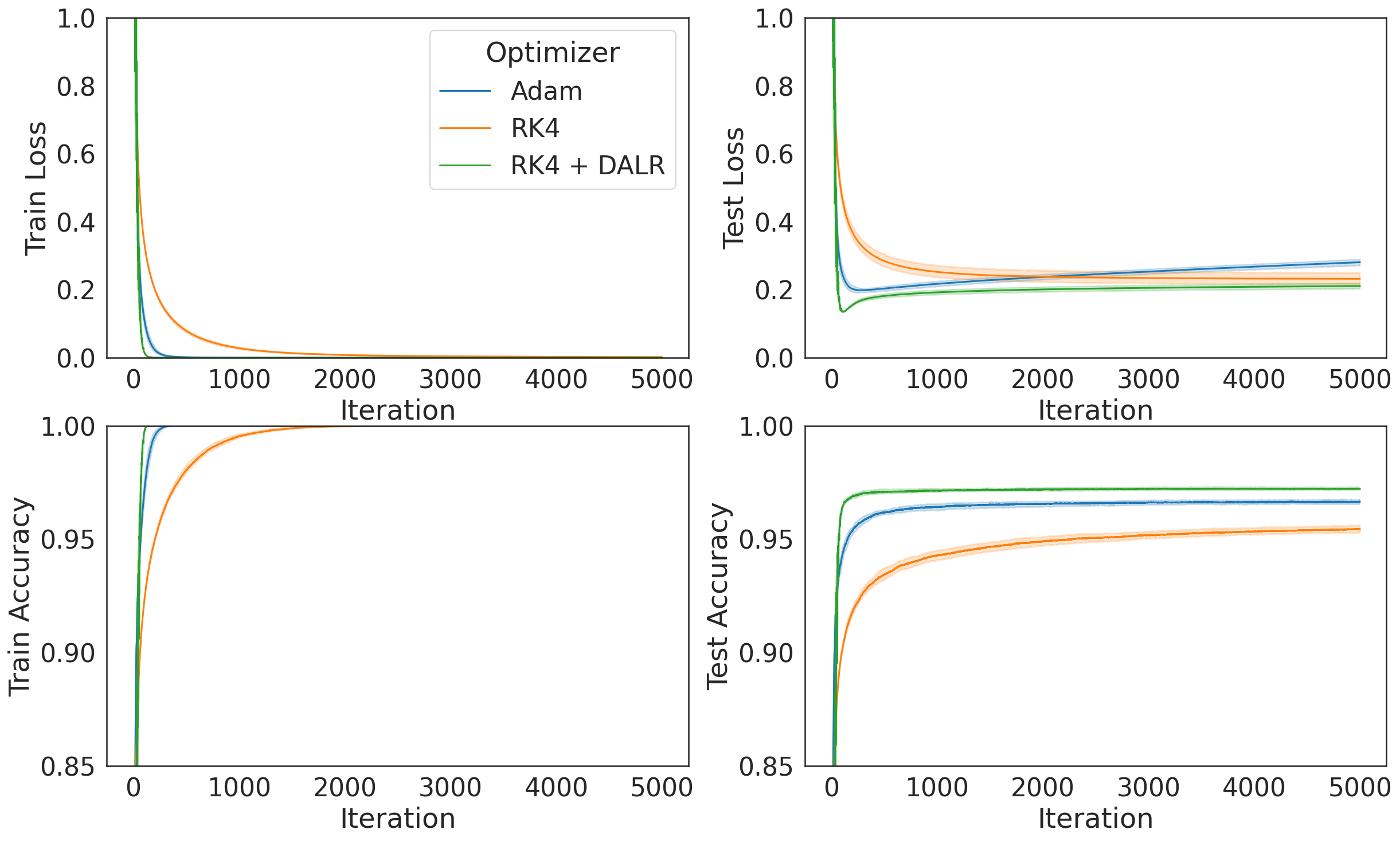}
\caption{
RK4 with adaptive learning rate (DALR) competes effectively with Adam on MNIST when using a 3-layer MLP (500 neurons per layer) trained with full-batch gradient descent for 5,000 steps.}
\label{figure:dalr_mnist}
\end{figure}

\begin{figure}[h] 
\centering 
    \includegraphics[width=1.0\textwidth]{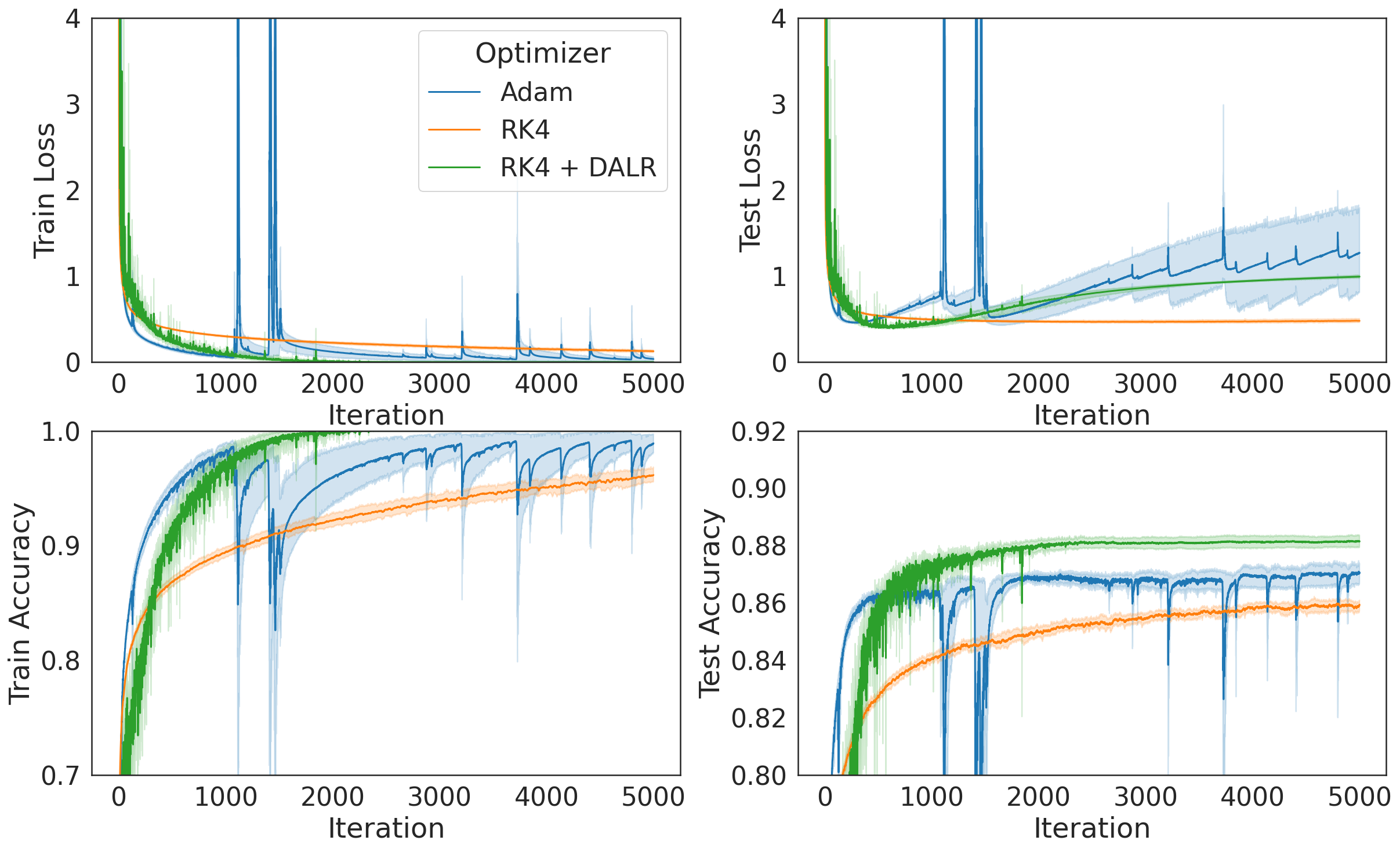}
\caption{
RK4 with adaptive learning rate (DALR) competes effectively with Adam on Fashion MNIST when using a 3-layer MLP (500 neurons per layer) trained with full-batch gradient descent for 5,000 steps.}
\label{figure:dalr_fashion_mnist}
\end{figure}

\clearpage
\newpage 

\subsection{Experiment Details for Fig.
 \ref{figure:bridging_the_gap_momentum}}
\label{appendix:bridging_the_gap_momentum}

RK4 used with momentum (see Section~\ref{section:momentum}) surpasses Adam on MNIST (left) and Fashion MNIST (right). The experiment settings for Adam and RK4 for both workloads are the exact same as for the experiments in the second row. The momentum has been tuned with values $\beta=0.95$ and learning rate 0.004 for MNIST (left) and $\beta=0.95$ and learning rate 0.001 for Fashion MNIST (right). The learning curves for MNIST are in Fig. \ref{figure:momentum_mnist} and those for Fashion MNIST are in Fig. \ref{figure:momentum_fashion_mnist}.

\begin{figure}[h] 
\centering 
    \includegraphics[width=1.0\textwidth]{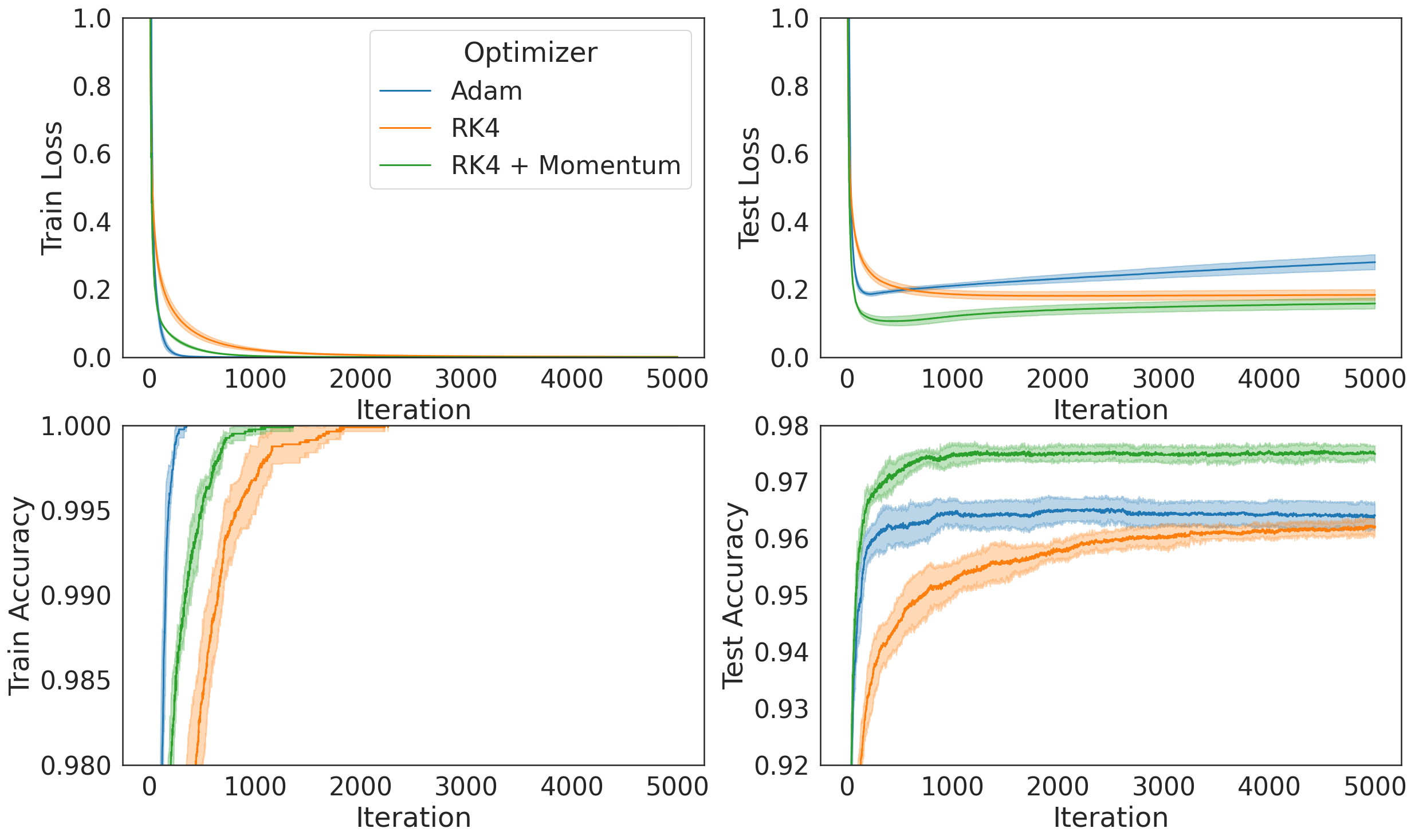}
\caption{
RK4 with momentum competes effectively with Adam on MNIST when using a 3-layer MLP (500 neurons per layer) trained with full-batch gradient descent for 5,000 steps.}
\label{figure:momentum_mnist}
\end{figure}

\begin{figure}[h] 
\centering 
    \includegraphics[width=1.0\textwidth]{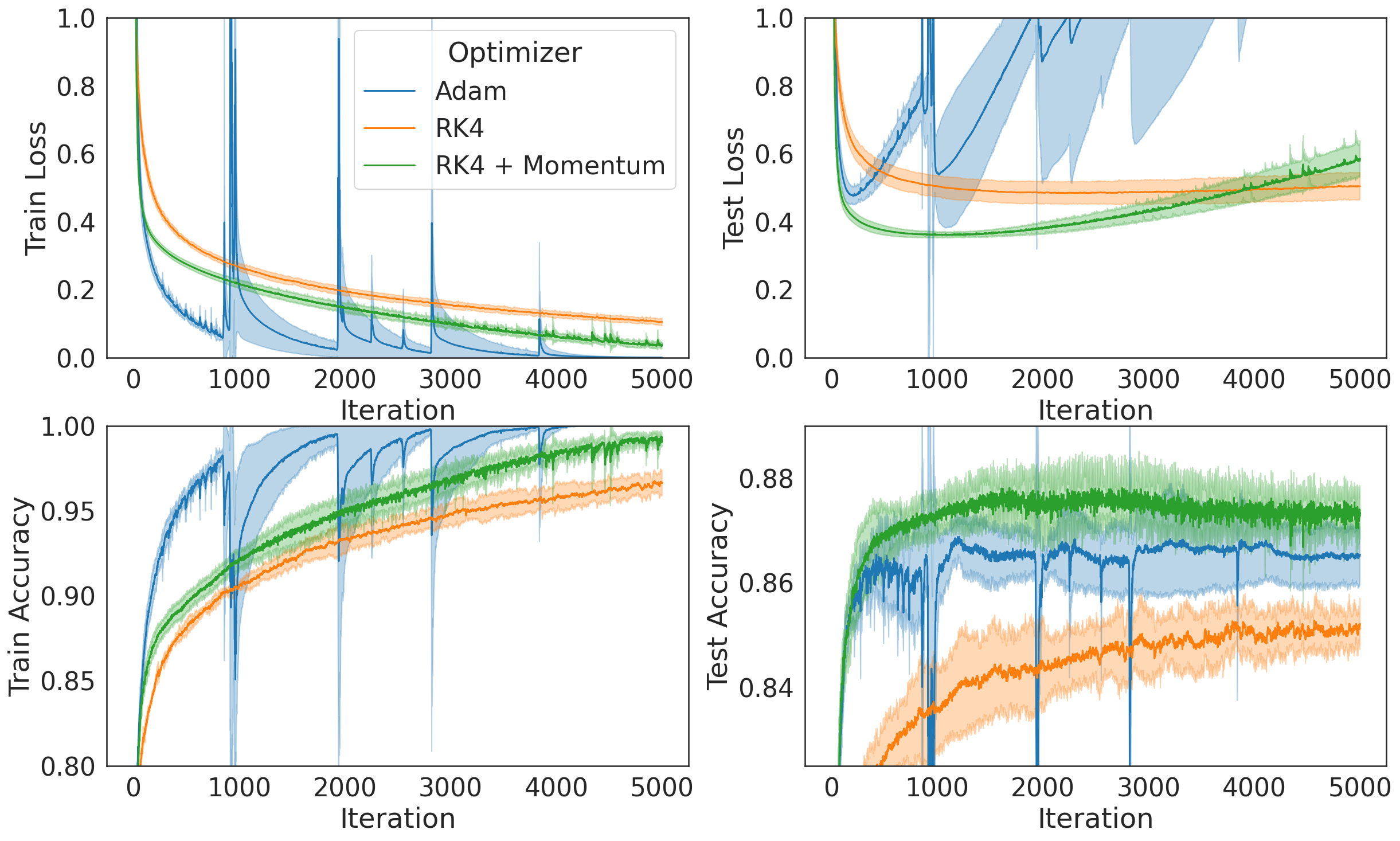}
\caption{
RK4 with momentum competes effectively with Adam on Fashion MNIST when using a 3-layer MLP (500 neurons per layer) trained with full-batch gradient descent for 5,000 steps.}
\label{figure:momentum_fashion_mnist}
\end{figure}

All 5 random seeds for each experiment took 5 hours and 10 minutes of training on the \texttt{SMALL-TPU} configuration for each workload.

\clearpage
\newpage

\subsection{Experiment Details for Table
 \ref{table:rk_sota}} 
\label{appendix:sota_experiment_details}
 
 This table presents a comparison of the best test set accuracy for various workloads, each comprising a specific dataset and model pairing. Models within each workload were trained using either the RK4 optimizer or a competitive baseline optimizer. For a fair comparison, all models for a given workload were trained for an identical number of steps. The reported accuracy and standard for each optimizer and workload combination is derived from measuring peak test accuracy, repeated across 5 independent trials with each independent trial employing a distinct random seed per trial. This seed influenced model weight initialization, data shuffling, and the stochastic elements of any data augmentation techniques.
 
For each of the experiments described below, we used one the following TPU configurations 
\begin{itemize}
    \item Google Cloud TPU V2 arranged as a 4x4 TPU pod
    \item Google Cloud TPU V5 arranged as 2x4 TPU pod
\end{itemize}
 
 \paragraph{MNIST on DNN} We trained a 3 layer DNN with 500 neurons per layer for 10,000 steps measuring training evaluation metrics every step. For the baseline, we used Adam optimizer with learning rate 0.001 and batch size 16. The 5 random seeds took roughly 3 hrs of training on TPUv2. For the RK4 optimizer, we used a learning rate of 0.003 and a batch size of 16. The 5 random seeds took roughly 7 hrs of training on TPUv2.
 
 \paragraph{MNIST on CNN} We trained a 2 layer CNN model, both layers having kernel size [3, 3], the first convolutional layer with 20 filters and the second convolutional layer with 10 filters. The model was trained for 10,000 steps measuring evaluation metrics every 10 steps. The 5 random seeds took roughly 3 hrs of training on TPUv2. For the baseline, we used stochastic gradient descent with a momentum of 0.9, constant learning rate of 0.01, batch size 64 and l2 weight decay factor of 0.002. For the RK4 optimizer, we used a constant learning rate of 0.1 with batch size of 64 and no weight decay. The 5 random seeds took roughly 7 hrs of training on TPUv2.
 
 \paragraph{Fashion-MNIST on DNN} We trained a 3 layer DNN with 500 neurons per layer for 10,000 steps measuring training evaluation metrics every step. For the baseline, we used Adam optimizer with learning rate 0.001 and batch size 16. The 5 random seeds took roughly 3 hrs of training on TPUv2. For the RK4 optimizer, we used a learning rate of 0.003 and a batch size of 16.  The 5 random seeds took roughly 7 hrs of training on TPUv2.
 
  \paragraph{Fashion-MNIST on CNN} We trained a 2 layer CNN model, both layers having kernel size [3, 3], the first convolutional layer with 20 filters and the second convolutional layer with 10 filters. The model was trained for 10,000 steps measuring evaluation metrics every 10 steps. For the baseline, we used stochastic gradient descent with a momentum of 0.9, constant learning rate of 0.01, batch size 512 and l2 weight decay factor of 0.0005. The 5 random seeds took roughly 3 hrs of training on TPUv2. For the RK4 optimizer, we used a constant learning rate of 0.5 with batch size of 64 and no weight decay. The 5 random seeds took roughly 7 hrs of training on TPUv2.
 
  \paragraph{CIFAR-10 on WRN} We trained a Wide-ResNet (WRN), specifically the WRN 28-10 model, as proposed by \cite{zagoruyko2016wide}. This network has a depth of 28 convolutional layers and a widening factor of 10. The WRN 28-10 was trained for 117,187 steps (or 200 epochs) using stochastic gradient descent with a momentum of 0.9, a learning rate of 0.1 with cosine decay and batch size of 128. The 5 random seeds took roughly 19 hrs of training on TPUv2. For the RK4 optimizers, we used a base learning rate of 0.5 with cosine decay and a batch size of 512. The 5 random seeds took roughly 32 hrs of training on TPUv2. 

  \paragraph{CIFAR-100 on WRN} We trained a Wide-ResNet (WRN), specifically the WRN 28-10 model, as proposed by \cite{zagoruyko2016wide}. This network has a depth of 28 convolutional layers and a widening factor of 10. The WRN 28-10 was trained for 117,187 steps using stochastic gradient descent with a momentum of 0.9, a learning rate of 0.1 with cosine decay, a batch size of 128 and l2 weight decay factor of 0.0005. The 5 random seeds took roughly 2 hrs of training on TPUv5. For the RK4 optimizers, we used a base learning rate of 0.75 with cosine decay and a batch size of 256. The 5 random seeds took roughly 3 hrs of training on TPUv5.
  
  \paragraph{ImageNet on ViT} We trained a Vision Transformer (ViT) model, ViT-B/16 \cite{dosovitskiy2020image}, on the ImageNet ILSVRC 2012 dataset. This model processes images by dividing them into 16x16 patches, which are then linearly embedded and supplied with positional embeddings before being fed into a series of Transformer encoder blocks. The ViT-B/16 model was trained directly on the ImageNet 1000-class dataset for 186,666 steps with evaluation every 1,866 steps. The baseline optimizer we used was NAdamW. This optimizer combines the Nesterov-accelerated adaptive moment estimation (NAdam) \cite{dozat2016incorporating} with the decoupled weight decay approach from AdamW \cite{loshchilov2017decoupled}. We used $\beta_1= 0.9414$, $\beta_2=0.9768$, $\epsilon = 1 \times 10^{-8}$, weight decay of 0.02, and 20\% label smoothing. The 5 random seeds took roughly 12 hrs of training on TPUv5. The RK4 optimizer was trained using learning rate 0.16 with cosine warmup and batch size 1024. The 5 random seeds took roughly 26 hrs of training on TPUv5.

\clearpage
\newpage

\section{Additional Experiments}
\label{appendix:additional_experiments}

\subsection{Wall-clock time Adam v.s. RK4 on CIFAR-10}
\label{appendix:figure_time_vs_step_cifar}

In Fig. \ref{figure:time_vs_step_cifar}, we trained a ResNet-18 \cite{he2016resnet,zagoruyko2016wide} on CIFAR-10 \cite{cifar10} for 1000 steps measuring the wall-clock time at every step for 5 random seeds. Adam's learning rate was tuned to 0.001 while the decay parameters were set to Optax defaults \cite{optax2020github}, and RK4 learning rate was tuned to 0.004. 
{\bf Left plot:} The batch-size was set to 1. The 5 random seeds took
roughly 2h of training on a single Google Cloud TPU V3 for each workloads.
{\bf Right plot:} The batch-size was set to 8192. The 5 random seeds took
roughly 23h of training on the \texttt{COLAB-TPU} configuration for each workloads.

\begin{figure}[h] 
\centering 
\includegraphics[width=0.49\textwidth]{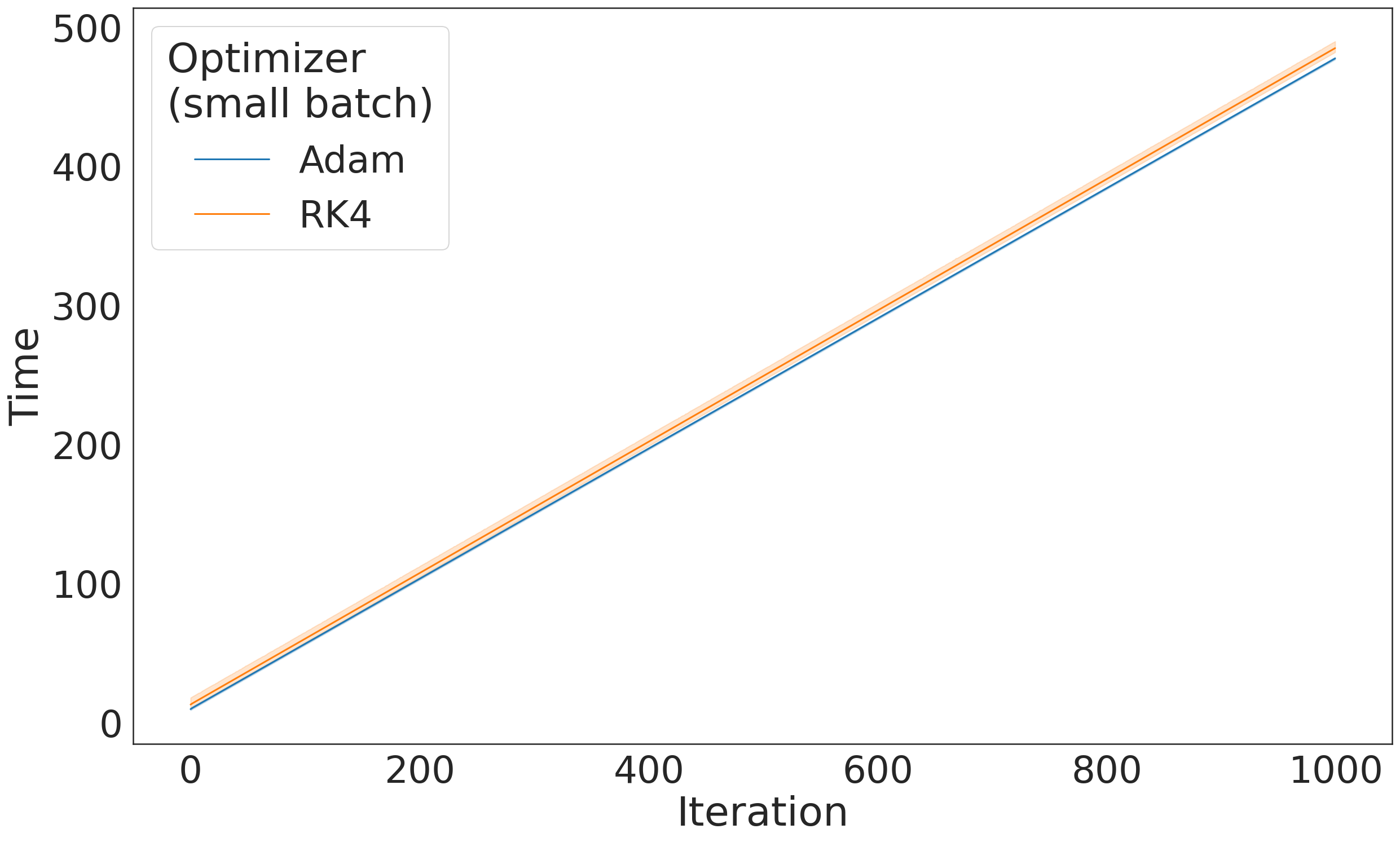}
\includegraphics[width=0.49\textwidth]{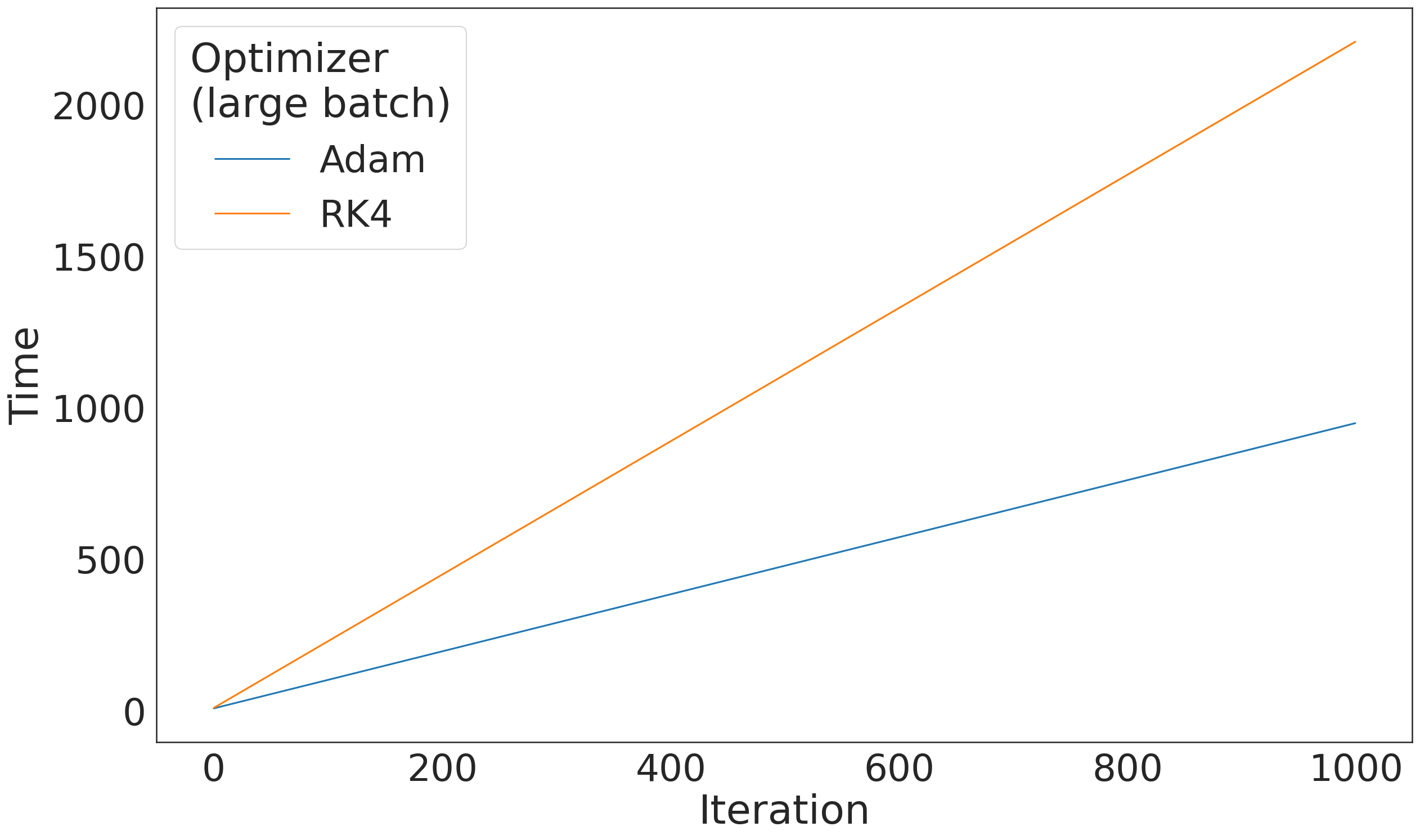}

\caption{
The wall-clock time for Adam and RK4 on CIFAR-10 is essentially the same for small batches but more than doubles for RK4 for large batches.
}
\label{figure:time_vs_step_cifar}
\end{figure}

\subsection{Preconditioners on CIFAR-10}
\label{section:cifar10_rk4_modified_adagrad}

\begin{figure}[h] 
\centering 
    \begin{minipage}{0.48\textwidth}
        \centering
        \includegraphics[width=1.0\textwidth]{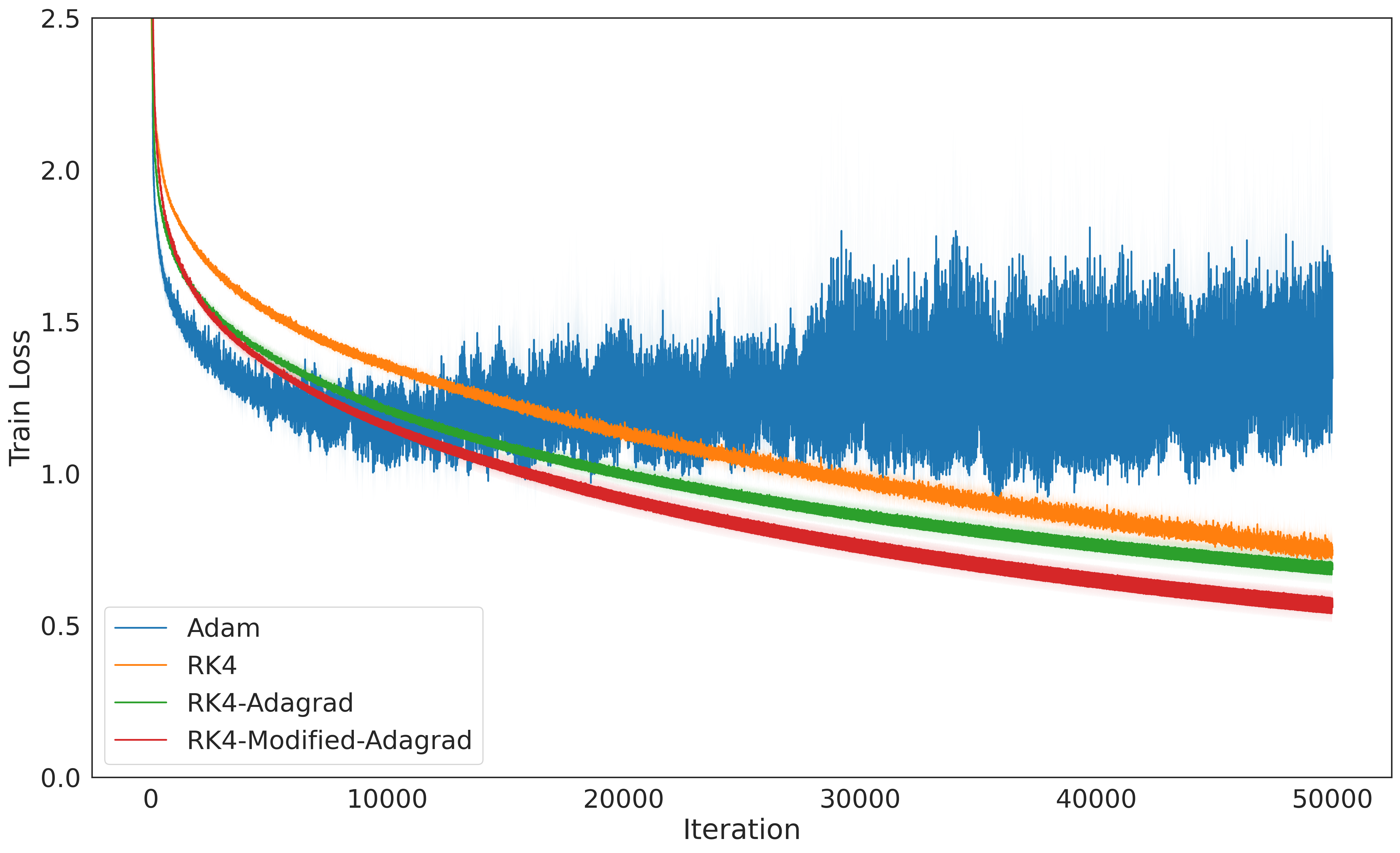}
    \end{minipage}
    \hfill 
    \begin{minipage}{0.48\textwidth}
        \centering
        \includegraphics[width=1.0\textwidth]{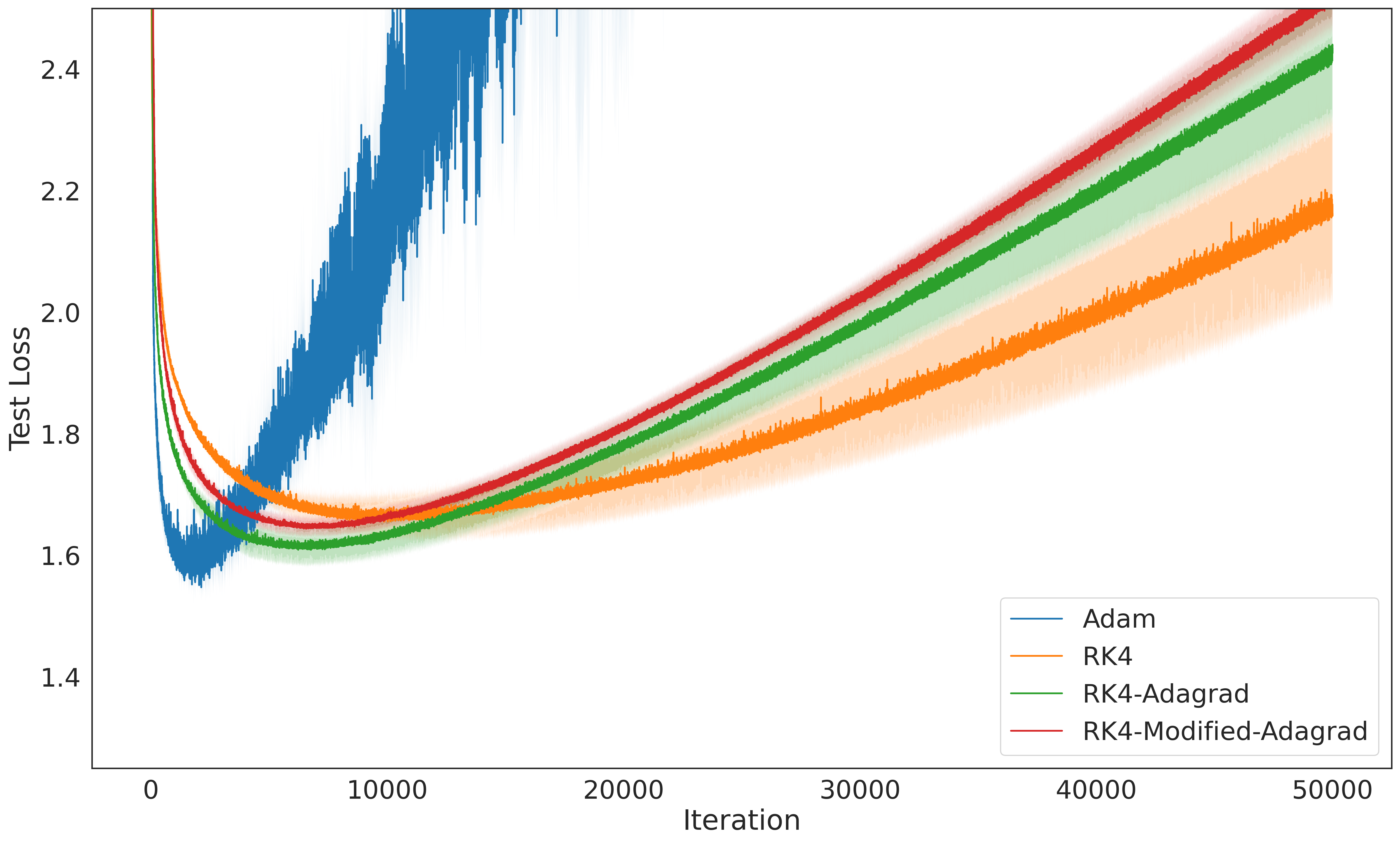}
    \end{minipage}
    \\
    \begin{minipage}{0.48\textwidth}
        \centering
        \includegraphics[width=1.0\textwidth]{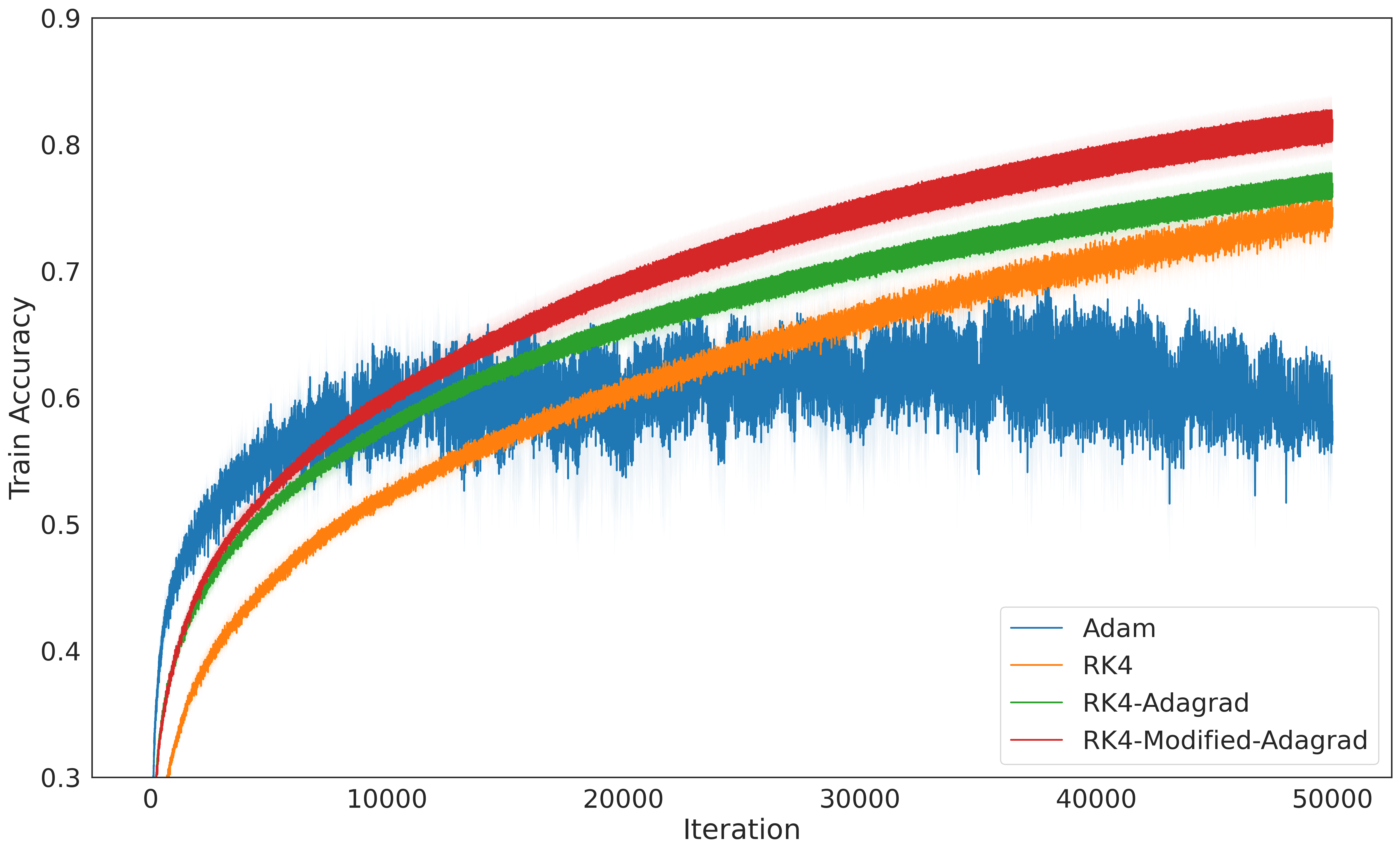}
    \end{minipage}
    \hfill 
    \begin{minipage}{0.48\textwidth}
        \centering
        \includegraphics[width=1.0\textwidth]{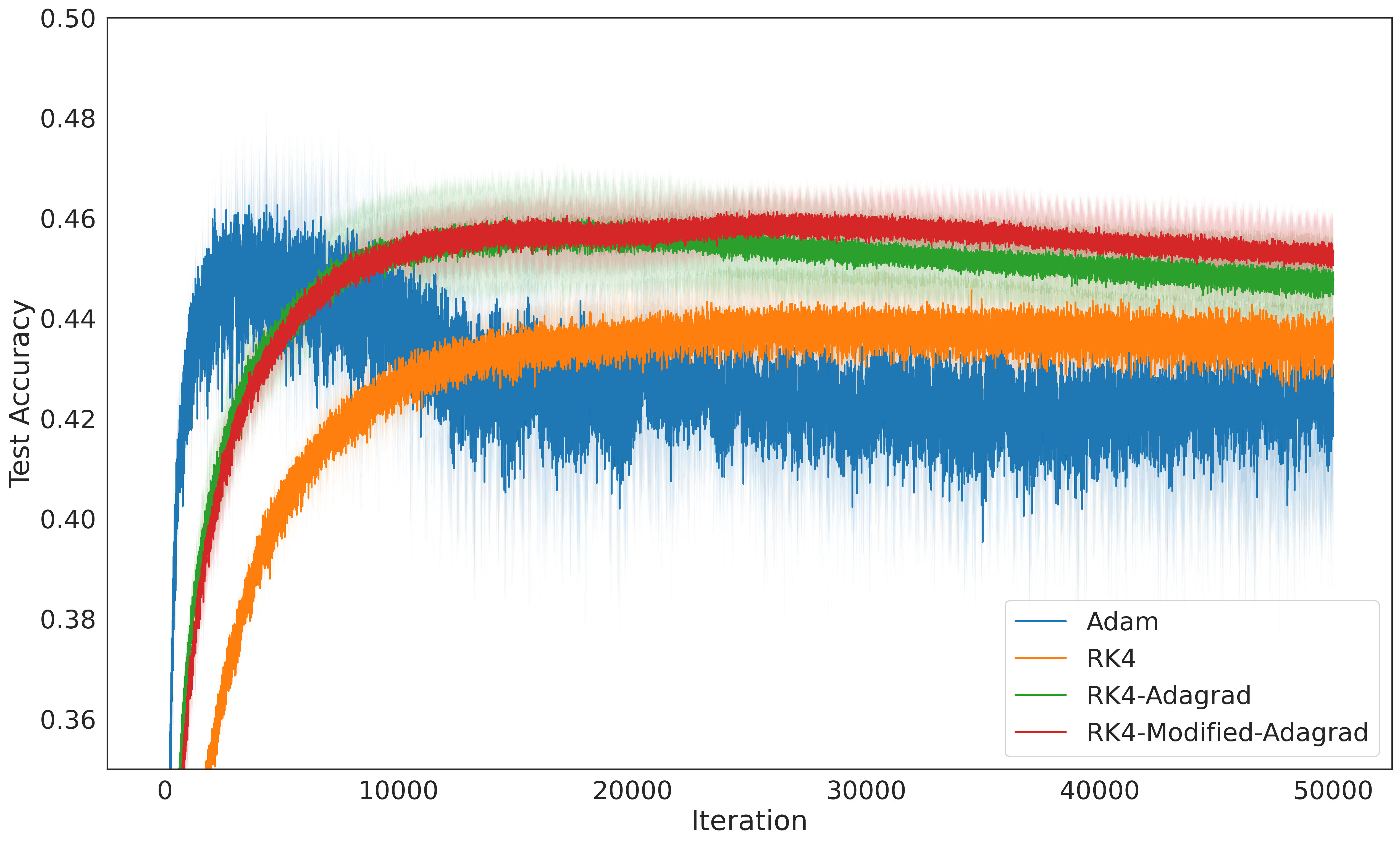}
    \end{minipage}
\caption{
The figure displays side-by-side training curves (left) versus test curves (right), illustrating that RK4 with preconditioning competes effectively with Adam on CIFAR-10 when using a 3-layer MLP (500 neurons per layer) trained with a batch size of 1024 for 50,000 steps. For these experiments, the tuned learning rates were: Adam (0.002), RK4 (0.001), RK4 + ADGR (0.008) and RK4 + Modified-ADGR (0.004). The random seeds took roughly 6h of training on the \texttt{COLAB-TPU} configuration.}
\label{figure:conditioners_cifar}
\end{figure}

\clearpage
\newpage

\subsection{Momentum and adaptive learning rate}
\label{section:cifar10_rk4_momentum_adaptive_lr}

\begin{figure}[h] 
\centering 
    \begin{minipage}{0.49\textwidth}
        \centering
        \includegraphics[width=1.0\textwidth]{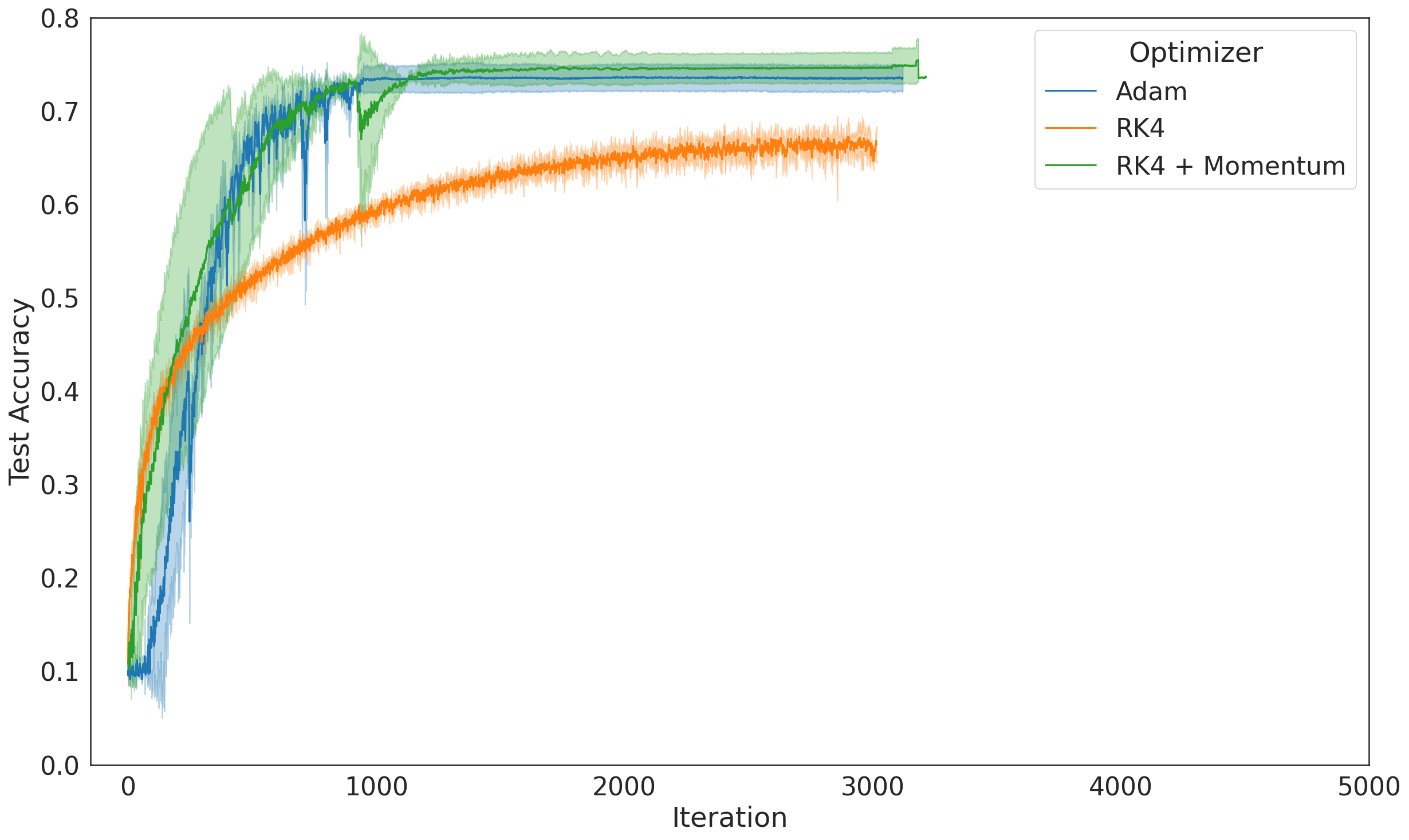}
    \end{minipage}
    \hfill 
    \begin{minipage}{0.49\textwidth}
        \centering
        \includegraphics[width=1.0\textwidth]{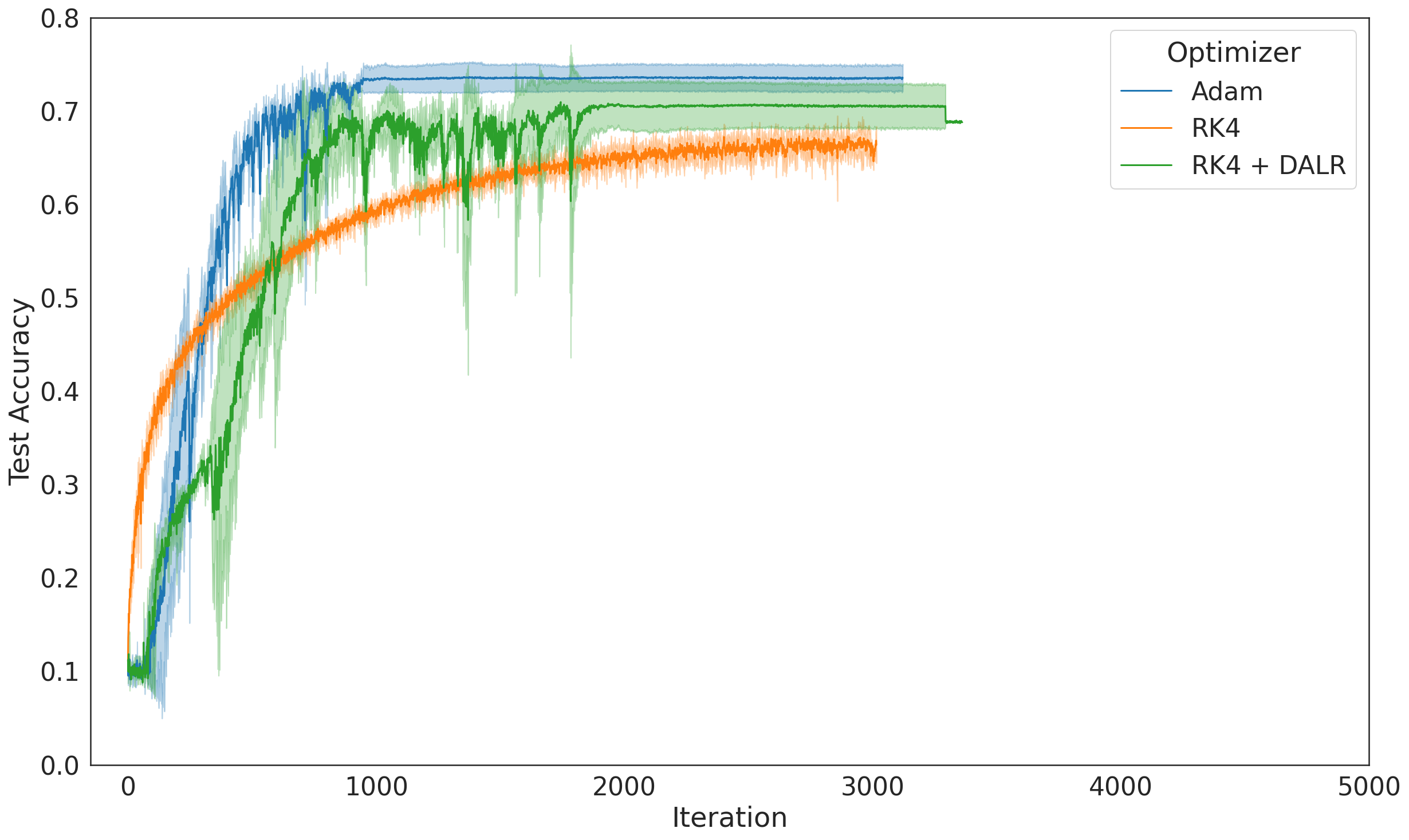}
    \end{minipage}
\caption{
The figure displays side-by-side training curves for momentum (left) and adaptive learning rate (DALR) (right) comparing it with Adam on CIFAR-10 when using a ResNet-18 \cite{resnets} trained with a batch size of 8192 for 3,000 steps. For these experiments, the tuned learning rates were: Adam (0.001), RK4 (0.004), RK4 + Momentum (0.004, $\beta=0.99$) and RK4 + DALR. No regularization, augmentation, or schedule was used in either case. The 5 random seeds for each experiment took roughly 3h 20 minutes of training on the \texttt{LARGE-TPU} configuration for each workload.
}
\label{figure:cifar_momentum_dalr}
\end{figure}

\end{document}